\documentclass{article}

\usepackage[square,sort,comma,numbers]{natbib}
\usepackage[final]{arxiv}
\usepackage{algorithm2e}

\usepackage[utf8]{inputenc} 
\usepackage{hyperref}       
\usepackage{url}            
\usepackage{booktabs}       
\usepackage{amsfonts}       
\usepackage{amsthm}
\usepackage{mathtools}
\usepackage{nicefrac}       
\usepackage{microtype}      
\usepackage{lipsum}
\usepackage{graphicx}

\usepackage{verbatim}
\usepackage{color,graphicx} 
\usepackage{amsfonts} 
\usepackage{amsmath}
\usepackage{amssymb}
\usepackage{enumitem}

\usepackage{overpic}
\usepackage{subcaption}

\newtheorem{thm}{Theorem}
\newtheorem{lem}{Lemma}
\newtheorem{prop}{Proposition}
\newtheorem{cor}{Corollary}
\theoremstyle{definition}
\newtheorem{defn}{Definition}
\newtheorem{ex}{Example}
\newtheorem{rmk}{Remark}
\newtheorem{ass}{Assumption}

\ifdefined\remappendix

\else

\fi

\newcommand{\tr}{\mathrm{tr}}
\newcommand{\sym}{\mathrm{sym}}

\newcommand{\dd}{\mathrm{d}}
\newcommand{\conv}{\mathrm{conv}}

\newcommand{\RR}{\mathbb{R}}      

\newcommand{\vecc}{\boldsymbol}
\newcommand{\diag}{\mathrm{diag}}

\title{Noisy Recurrent Neural Networks}
\author{
 Soon Hoe Lim \\
  Nordita, KTH Royal Institute of Technology \\
  and Stockholm University\\
  \texttt{soon.hoe.lim@su.se} \\
  \And 
  N. Benjamin Erichson
  \hspace{0.8cm} \\
  School of Engineering
  \hspace{0.8cm} \\
  University of Pittsburgh \hspace{0.8cm} \\
  \texttt{erichson@pitt.edu} \hspace{0.8cm} 
  \\
  \And
  Liam Hodgkinson  \\
  ICSI and Department of Statistics, \\
  UC Berkeley \\
  \texttt{liam.hodgkinson@berkeley.edu}
  \And
  Michael W. Mahoney  \\
  ICSI and Department of Statistics, \\
  UC Berkeley \\
  \texttt{mmahoney@stat.berkeley.edu}
}

\usepackage[verbose=true,letterpaper]{geometry}
\AtBeginDocument{
	\newgeometry{
		textheight=8.8in,
		textwidth=6.1in,
		top=1in,
		headheight=14pt,
		headsep=25pt,
		footskip=30pt
	}
}

\begin{document}

\maketitle

\begin{abstract}
We provide a general framework for studying recurrent neural networks (RNNs) trained by injecting noise into hidden states.
Specifically, we consider RNNs that can be viewed as discretizations of stochastic differential equations driven by input data.
This framework allows us to study the implicit regularization effect of general noise injection schemes by deriving an approximate explicit regularizer in the small noise regime. 
We find that, under reasonable assumptions, this implicit regularization promotes flatter minima; it biases towards models with more stable dynamics; and, in classification tasks, it favors models with larger classification margin.
Sufficient conditions for global stability are obtained, highlighting the phenomenon of stochastic stabilization, where noise injection can improve stability during training.
Our theory is supported by empirical results which demonstrate that the RNNs have improved robustness with respect to various input perturbations.
\end{abstract}

\section{Introduction}
Viewing recurrent neural networks (RNNs) as discretizations of ordinary differential equations (ODEs) driven by input data has recently gained attention~\cite{chang2019antisymmetricrnn,Kag2020RNNs,erichson2020lipschitz,rusch2021coupled}. 
The ``formulate in continuous time, and then discretize'' approach \cite{ma2020towards} motivates novel architecture designs before experimentation, and it provides a useful interpretation as a dynamical system. 
This, in turn, has led to gains in reliability and robustness to data perturbations.

Recent efforts have shown how adding noise can also improve stability during training, and consequently improve robustness \cite{liu2019neural}.
In this work, we consider discretizations of the corresponding stochastic differential equations (SDEs) obtained from ODE formulations of RNNs through the addition of a diffusion (noise) term.
We refer to these as \emph{Noisy RNNs} (NRNNs). 
By dropping the noisy elements at inference time, NRNNs become a stochastic learning strategy which, as we shall prove, has a number of important benefits.
In particular, stochastic learning strategies (including dropout) are often used as natural regularizers, favoring solutions in regions of the loss landscape with desirable properties (often improved generalization and/or robustness). 
This mechanism is commonly referred to as \emph{implicit regularization} \cite{MO11-implementing,Mah12,smith2021origin}, differing from \emph{explicit regularization} where the loss is explicitly modified. 
For neural network models, implicit regularization towards wider minima is conjectured to be a prominent ingredient in the success of stochastic optimization \cite{zhang2016understanding, keskar2016large}. Indeed, implicit regularization has been linked to increases in classification margins \cite{poggio2017theory}, which can lead to improved generalization performance \cite{sokolic2017generalization}. 
A common approach to identify and study implicit regularization is to approximate the implicit regularization by an appropriate explicit regularizer \cite{ali2020implicit, camuto2020explicit, gong2020maxup}. 
Doing so, we will see that NRNNs favor wide minima (like SGD);  more stable dynamics; and classifiers with a large classification margin, keeping generalization error small. 

SDEs have also seen recent appearances in \emph{neural SDEs} \cite{tzen2019neural,hodgkinson2020stochastic}, stochastic generalizations of \emph{neural ODEs} \cite{chen2018neural} which can be seen as an analogue of NRNNs for non-sequential data, with a similar relationship to NRNNs as feedforward neural networks do to RNNs. They have been shown to be robust in practice \cite{liu2019neural}. Analogously, we shall show that the NRNN framework leads to more reliable and robust RNN classifiers, whose promise is demonstrated by experiments on benchmark data sets. \\

\textbf{Contributions.} 
For the class of NRNNs (formulated first as a continuous-time model, which is then discretized):

\begin{itemize}[leftmargin=*,topsep=0pt,parsep=0pt,partopsep=0pt]
\item we identify the form of the implicit regularization for NRNNs  through a corresponding (data-dependent) explicit regularizer in the small noise regime (see Theorem \ref{thm_exp_reg_discrete});
\item we focus on its effect in classification tasks,   providing bounds for the classification margin  for the deterministic RNN classifiers (see Theorem \ref{thm_gen_discrete}); in particular, Theorem \ref{thm_gen_discrete} reveals that {\it stable RNN dynamics can lead to large  classification margin};
\item we show that noise injection can also lead to improved stability (see Theorem \ref{prop:Stability}) via a Lyapunov stability analysis of continuous-time NRNNs;
\item we demonstrate via empirical experiments on benchmark data sets that NRNN classifiers are more robust to data perturbations when compared to other recurrent models, while retaining state-of-the-art performance for clean data. Research code is provided here:  \url{https://github.com/erichson/NoisyRNN}.
\\
\end{itemize}

{\bf Notation.} 
We use $\|v\| := \|v\|_2$ to denote the Euclidean norm of the vector $v$, and $\|A\|_2$ and $\|A\|_F$ to denote the spectral norm and Frobenius norm of the matrix $A$, respectively. 
The $i$th element of a vector $v$ is denoted by $v^i$ or $[v]^i$, and the $(i,j)$-entry of a matrix $A$ by $A^{ij}$ or $[A]^{ij}$. For a vector $v = (v^1, \dots, v^d)$, diag($v$) denotes the diagonalization of $v$ with $\diag(v)^{ii} = v^i$. 
$I$ denotes the identity matrix (with dimension clear from context), while superscript $T$ denotes transposition. 
For a matrix $M$, $M^\sym = (M + M^T)/2$ denotes its symmetric part, $\lambda_{\min}(M)$ and $\lambda_{\max}(M)$ denote its minimum and maximum eigenvalue respectively,  $\sigma_{\max}(M)$ denotes its maximum singular value, and $Tr(M)$ denotes its trace. 
For a function 
$f : \RR^n \to \RR^m$ such that each of its first-order partial derivatives (with respect to $x$) exist, $\frac{\partial f}{\partial x} \in \RR^{m \times n}$ is the Jacobian matrix of $f$. For a scalar-valued function $g: \RR^n \to \RR$,
$\nabla_h g$ is the gradient of $g$ with respect to the variable $h \in \RR^n$ and $H_h g$ is the Hessian of $g$ with respect to $h$.

\section{Related Work}
\label{sec:Related}

\textbf{Dynamical Systems and Machine Learning.} 
There are various interesting connections between machine learning and dynamical systems. Formulating machine learning in the framework of continuous-time dynamical systems was recently popularized by \cite{weinan2017proposal}. Subsequent efforts focus on constructing learning models by approximating continuous-time dynamical systems \cite{chen2018neural, kidger2020neural,queiruga2020continuous} and studying them using tools from  numerical analysis \cite{lu2018beyond,yang2019dynamical,zhang2019towards,zhang2019stability}. On the other hand, dynamical systems theory provides useful theoretical tools for analyzing neural networks (NNs), including RNNs \cite{vogt2020lyapunov,engelken2020lyapunov,lim2020understanding,chang2019antisymmetricrnn,erichson2020lipschitz}, and useful principles for designing NNs \cite{haber2017stable,sun2018stochastic}. Other examples of dynamical systems inspired models include the learning of invariant quantities via their Hamiltonian or Lagrangian representations~\cite{lutter2019deep,NEURIPS2019_26cd8eca,chen2019symplectic,zhong2019symplectic,toth2019hamiltonian}. Another class of models is inspired by Koopman theory, yielding models where the evolution operator is linear~\cite{takeishi2017learning,morton2019deep,erichson2019physics,pan2020physics,li2019learning,balakrishnan2020deep,azencot2020forecasting,dogra2020optimizing}.

\textbf{Stochastic Training and Regularization Strategies.} 
Regularization techniques such as noise injection and dropout can help to prevent overfitting in neural networks. Following the classical work \cite{bishop1995training} that studies regularizing effects of noise injection on data, several work studies the effects of noise injection into different parts of networks for various architectures \cite{huang2016deep,noh2017regularizing,liu2019neural,sun2018stochastic,arora2020dropout,jim1996analysis,zaremba2014recurrent, wei2020implicit}. In particular, recently \cite{camuto2020explicit} studies the regularizing effect of isotropic Gaussian noise injection into the layers of feedforward networks. For RNNs, \cite{dieng2018noisin} shows that noise additions on the hidden states outperform Bernoulli dropout in terms of performance and bias, whereas \cite{fraccaro2016sequential} introduces a variant of stochastic RNNs for generative modeling of sequential data.  Some specific formulations of RNNs as SDEs were also considered in Chapter 10 of \cite{mao2007stochastic} and \cite{Das98}. Implicit regularization has also been studied more generally~\cite{MO11-implementing,Mah12,GM14_ICML,DLM19_Exact_TR,smith2021origin}.

\section{Noisy Recurrent Neural Networks} 
\label{sec:NRNN}

We formulate continuous-time recurrent neural networks (CT-RNNs) at full generality as a system of input-driven ODEs: for a terminal time $T > 0$ and an input signal $x = (x_t)_{t \in [0,T]} \in C([0,T]; \mathbb{R}^{d_x})$, the output $y_t   \in  \mathbb{R}^{d_y}$, for $t \in [0,T]$, is a linear map of hidden states $h_t \in \mathbb{R}^{d_h}$~satisfying
\begin{equation}
\label{eq:ContRNN}
\dd h_t = f(h_t,x_t) \dd t,\qquad y_t = V h_t,
\end{equation}
where $V \in \mathbb{R}^{d_y \times d_h}$, and $f:\mathbb{R}^{d_h} \times \mathbb{R}^{d_x} \to \mathbb{R}^{d_h}$ is typically Lipschitz continuous, guaranteeing existence and uniqueness of solutions to (\ref{eq:ContRNN}). 

A natural stochastic variant of CT-RNNs arises by replacing the ODE in (\ref{eq:ContRNN}) by an It\^{o} SDE, that is,
\begin{equation}
\label{NLRNN}
\dd h_t = f(h_t, x_t) \dd t + \sigma(h_t,x_t) \dd B_t, \qquad y_t = V h_t,
\end{equation}
where $\sigma:\mathbb{R}^{d_h} \times \mathbb{R}^{d_x} \to \mathbb{R}^{d_h \times r}$ and $(B_t)_{t \geq 0}$ is an $r$-dimensional Brownian motion. The functions $f,\sigma$ are referred to as the \emph{drift} and \emph{diffusion} coefficients, respectively. Intuitively, (\ref{NLRNN}) amounts to a noisy perturbation of the corresponding deterministic CT-RNN (\ref{eq:ContRNN}).  At full generality, we refer to the system (\ref{NLRNN}) as a \emph{continuous-time Noisy RNN} (CT-NRNN). To guarantee the existence of a unique solution to (\ref{NLRNN}), in the sequel, we assume that $\{f(\cdot,x_t)\}_{t \in [0,T]}$ and $\{\sigma(\cdot, x_t)\}_{t \in [0,T]}$ are uniformly Lipschitz continuous, and $t \mapsto f(h,x_t)$, $t \mapsto \sigma(h, x_t)$ are bounded in $t \in [0,T]$ for each fixed $h \in \mathbb{R}^{d_h}$. 
For further details, see Section B in Supplementary Material ({\bf SM}). 

While much of our theoretical analysis will focus on this general formulation of CT-NRNNs, our empirical and stability analyses focus on the choice of drift function
\begin{equation}
\label{eq:LipschitzRNN}
f(h, x) = A h + a(W h + U x + b),
\end{equation}
where $a:\mathbb{R}\to\mathbb{R}$ is a Lipschitz continuous scalar activation function extended to act on vectors pointwise, $A,W \in \mathbb{R}^{d_h \times d_h}$, $U \in \mathbb{R}^{d_h \times d_x}$ and $b \in \mathbb{R}^{d_h}$. Typical examples of activation functions include $a(x) = \tanh(x)$. The matrices $A,W,U,V,b$ are all assumed to be trainable parameters. This particular choice of drift dates back to the early Cohen-Grossberg formulation of CT-RNNs, and was recently reconsidered in \cite{erichson2020lipschitz}.

\subsection{Noise Injections as Stochastic Learning Strategies}\label{subsect_noiseinjection}

While precise choices of drift functions $f$ are the subject of existing deterministic RNN theory, good choices of the diffusion coefficient $\sigma$ are less clear. Here, we shall consider a parametric class of diffusion coefficients given by:
\begin{equation}\label{experiment_NRNN}
    \sigma(h,x) \equiv \epsilon (\sigma_1 I + \sigma_2 \diag(f(h,x))), 
\end{equation}
where the noise level $\epsilon > 0$ is small, and $\sigma_1 \geq 0$ and $\sigma_2 \geq 0$ are tunable parameters describing the relative strength of additive noise and a multiplicative noise respectively. 

While the stochastic component is an important part of the model, one can  set $\epsilon \equiv 0$ at inference time. 
In doing so, noise injections in NRNNs may be viewed as a learning strategy.
A similar stance is considered in \cite{liu2019neural} for treating neural SDEs.
From this point of view, we may relate noise injections generally to regularization mechanisms considered in previous works. For example, additive noise injection was studied in the context of feedforward NNs in \cite{camuto2020explicit}, in which case a Gaussian noise is injected to the activation function at each layer of the NN. Furthermore, multiplicative noise injections includes stochastic depth and dropout strategies as special cases \cite{lu2018beyond,liu2019neural}. By taking a Gaussian approximation to Bernoulli noise and taking a continuous-time limit, NNs with stochastic dropout can be weakly approximated by an SDE with appropriate multiplicative noise, see \cite{lu2018beyond}. All of these works highlight various advantages of noise injection for training NNs.

\subsection{Numerical Discretizations}
\label{subsect_discretizations}

As in the deterministic case, exact simulation of  the SDE in  (\ref{NLRNN}) is infeasible in practice, and so one must specify a numerical integration scheme. We will focus on the explicit Euler-Maruyama (E-M) integrators \cite{kloeden2013numerical}, which are the stochastic analogues of Euler-type integration schemes for ODEs.

Let $0 \coloneqq t_0 < t_1 < \cdots < t_M \coloneqq T$ be a partition of the interval $[0,T]$. Denote $\delta_m := t_{m+1} - t_m$ for each $m=0,1,\dots,M-1$, and $\delta := (\delta_m)$. The  E-M scheme provides a family (parametrized by $\delta$) of approximations to the solution of the SDE in \eqref{NLRNN}:
\begin{equation}\label{e-m}
    h^{\delta}_{m+1} = h^{\delta}_{m} + f(h^{\delta}_{m},\hat{x}_{m}) \delta_m + \sigma(h^{\delta}_{m},\hat{x}_{m}) \sqrt{\delta_m} \xi_m,
\end{equation}
for $m=0,1,\dots,M-1$, where $(\hat{x}_{m})_{m=0,\dots,M-1}$ is a given sequential data, the $\xi_m \sim \mathcal{N}(0,I)$ are independent $r$-dimensional standard normal random vectors, and $h^{\delta}_0 = h_{0}$. As $\Delta \coloneqq \max_m \delta_m~\to~0$, the family of approximations $(h^{\delta}_{m})$ converges strongly to the It\^o process $(h_{t})$ satisfying (\ref{NLRNN}) (at rate $\mathcal{O}(\sqrt{\Delta})$ when the step sizes are uniform; see Theorem 10.2.2 in \cite{kloeden2013numerical}). See Section C in {\bf SM} for details on the general case.

\section{Implicit Regularization}
\label{sec:ImpReg}

To highlight the advantages of NRNNs over their deterministic counterpart, we show that, under reasonable assumptions, NRNNs exhibit a natural form of \emph{implicit regularization}.
By this, we mean regularization imposed implicitly by the stochastic learning strategy, without explicitly modifying the loss, but that, e.g., may promote flatter minima. 
Our goal is achieved by deriving an appropriate explicit regularizer through a perturbation analysis in the small noise regime. This becomes useful when considering NRNNs as a learning strategy, since we can precisely determine the effect of the noise injection as a regularization mechanism. 

The study for discrete-time NRNNs is of practical interest and is our focus here. Nevertheless, analogous results for continuous-time NRNNs are also valuable for exploring other discretization schemes. For this reason, we also study the continuous-time case in Section E in {\bf SM}. Our analysis covers general NRNNs, not necessarily those with the drift term \eqref{eq:LipschitzRNN} and diffusion term \eqref{experiment_NRNN}, that satisfy the following assumption, which is typically reasonable in practice. We remark that a ReLU activation will violate the assumption. However, RNNs with ReLU activation are less widely used in practice. Without careful initialization \cite{le2015simple,talathi2015improving}, they typically suffer more from exploding gradient problems compared to those with bounded activation functions such as $\tanh$.

\begin{ass}
\label{ass:Smooth}
The drift $f$ and diffusion coefficient $\sigma$ of the SDE in (\ref{NLRNN}) satisfy the following:
\begin{enumerate}[label=(\roman*),leftmargin=*]
    \item for all $t \in [0,T]$ and $x \in \mathbb{R}^{d_x}$, $h \mapsto f(h, x)$ and $h \mapsto \sigma^{ij}(h,x)$ have Lipschitz continuous partial derivatives in each coordinate up to order three (inclusive);
    \item for any $h \in \mathbb{R}^{d_h}$, $t \mapsto f(h, x_t)$ and $t \mapsto \sigma(h, x_t)$ are bounded and Borel measurable on $[0,T]$. 
\end{enumerate}
\end{ass}

We consider a rescaling of the noise $\sigma \mapsto \epsilon \sigma$ in \eqref{NLRNN}, where $\epsilon > 0$ is assumed to be a small parameter, in line with our noise injection strategies in Subsection \ref{subsect_noiseinjection}. 

In the sequel, we let $\bar{h}^{\delta}_m$ denote the hidden states of the corresponding deterministic RNN model, satisfying
\begin{equation}
\label{eq:DeterRNNDiscrete}
\bar{h}_{m+1}^\delta = \bar{h}_m^\delta + \delta_m f(\bar{h}_m^\delta, \hat{x}_m), \quad m = 0,1,\dots,M-1,
\end{equation}
with $\bar{h}^{\delta}_0 = h_0$.  Let $\Delta := \max_{m \in \{0,\dots,M-1\}} \delta_m$, and denote the state-to-state Jacobians by
\begin{equation}
    \hat{J}_m = I+\delta_m \frac{\partial f}{\partial h}(\bar{h}^{\delta}_{m},\hat{x}_{m}).
\end{equation}
For $m,k = 0,\dots,M-1$, also let
\begin{equation}
    \hat{\Phi}_{m,k} = \hat{J}_{m} \hat{J}_{m-1} \cdots \hat{J}_k,
\end{equation}
where the empty product is assumed to be the identity. Note that the $\hat{\Phi}_{m,k}$  are  products of the state-to-state Jacobian matrices, important for analyzing signal propagation in  RNNs \cite{chen2018dynamical}. For the sake of brevity, we denote $f_m = f(\bar{h}^{\delta}_m, \hat{x}_m)$ and $\sigma_m = \sigma(\bar{h}^{\delta}_m,\hat{x}_m)$ for $m = 0,1,\dots,M$.

The following result, which is our first main result, relates the loss function, averaged over realizations of the injected noise, used for training NRNN to that for training deterministic RNN in the small noise regime.

\begin{thm}[Implicit regularization induced by noise injection]
\label{thm_exp_reg_discrete}
Under Assumption \ref{ass:Smooth}, 
\begin{align} \label{exp_reg}
    \mathbb{E} \ell(h^{\delta}_M) &= \ell(\bar{h}^{\delta}_M) + \frac{\epsilon^2}2 [\hat{Q}(\bar{h}^{\delta}) +\hat{R}(\bar{h}^{\delta})]  + \mathcal{O}(\epsilon^3),
\end{align}
as $\epsilon \to 0$, where the terms $\hat{Q}$ and $\hat{R}$  are given by
\begin{align} \label{P_term_discrete}
    \hat{Q}(\bar{h}^{\delta}) &= \nabla l(\bar{h}^{\delta}_M)^T  \sum_{k=1}^M \delta_{k-1} \hat{\Phi}_{M-1,k}  \sum_{m=1}^{M-1} \delta_{m-1} \vecc{v}_{m}, \\
    \hat{R}(\bar{h}^{\delta}) &= \sum_{m=1}^{M} \delta_{m-1} \tr( \sigma_{m-1}^T \hat{\Phi}^T_{M-1,m} H_{\bar{h}^{\delta}}l \ \hat{\Phi}_{M-1,m} \sigma_{m-1}),
    \label{R_term_discrete}
\end{align}
with $\vecc{v}_m$ a vector with the $p$th component:
\begin{equation}
    [v_m]^p = \tr(\sigma_{m-1}^T \hat{\Phi}_{M-2,m}^T H_{\bar{h}^\delta} [f_M]^{p} 
    \hat{\Phi}_{M-2,m}  \sigma_{m-1}),
\end{equation}
for $p=1,\dots, d_h$. Moreover, 
$|\hat{Q}(\bar{h}^{\delta})| \leq C_Q \Delta^2, \ \ |\hat{R}(\bar{h}^{\delta})| \leq C_R \Delta, $
for $C_Q, C_R >0$ independent of $\Delta$.
\end{thm}

If the loss is convex, then $\hat{R}$ is non-negative, but $\hat{Q}$ needs not be. However, $\hat{Q}$ can be made negligible relative to $\hat{R}$ provided that $\Delta$ is taken sufficiently small.
This also ensures that the E-M approximations are accurate. 

To summarize, Theorem \ref{thm_exp_reg_discrete} implies that the injection of noise into the hidden states of deterministic RNN is, on average, approximately equivalent to a regularized objective functional. Moreover, the explicit regularizer is solely determined by the discrete-time flow generated by the Jacobians $\frac{\partial f_{m}}{\partial \bar{h}}(\bar{h}^{\delta}_m)$, the diffusion coefficients $\sigma_{n}$, and the Hessian of the loss function, all evaluated along the dynamics of the deterministic RNN.
We can therefore expect that the use of NRNNs as a regularization mechanism should reduce the state-to-state Jacobians and Hessian of the loss function according to the noise level $\epsilon$. 
Indeed, NRNNs exhibit a smoother Hessian landscape than that of the deterministic counterpart (see Figure 3 in {\bf SM}). 

The Hessian of the loss function commonly appears in implicit regularization analyses, and suggests a preference towards wider minima in the loss landscape. Commonly considered a positive attribute \cite{keskar2016large}, this, in turn, suggests a degree of robustness in the loss to perturbations in the hidden states \cite{yao2018hessian}. 
More interesting, however, is the appearance of the Jacobians, which is indicative of a preference towards slower, more stable dynamics. 
Both of these attributes suggest NRNNs could exhibit a strong tendency towards models which are less sensitive to input perturbations. Overall, we can see that the use of NRNNs as a regularization mechanism reduces the state-to-state Jacobians
and Hessian of the loss function according to the noise level.

\section{Implications in Classification Tasks}
\label{sec:classif}

Our focus now turns to an investigation of the benefits of  NRNNs over their deterministic counterparts for classification tasks. From Theorem \ref{thm_exp_reg_discrete}, it is clear that adding noise to deterministic RNN implicitly regularizes the state-to-state Jacobians. 
Here, we show that doing so also enhances an implicit tendency towards  classifiers with large classification margin.
Our analysis here covers general deterministic RNNs, 
although we also apply our results to obtain explicit expressions for Lipschitz RNNs.

Let $\mathcal{S}_N$ denote a set of training samples $s_n \coloneqq (\vecc{x}_n, y_{n})$ for $n=1,\dots,N$, where each input sequence $\vecc{x}_n = (x_{n,0},x_{n,1},\dots,x_{n,M-1}) \in \mathcal{X} \subset \mathbb{R}^{d_x M}$ has a corresponding class label $y_n \in \mathcal{Y} = \{1,\dots,d_y\}$. Following the statistical learning framework, these samples are assumed to be independently drawn from an underlying probability distribution $\mu$ on the sample space $\mathcal{S} = \mathcal{X} \times \mathcal{Y}$. An RNN-based classifier $g^\delta(\vecc{x})$ is constructed in the usual way by taking 

\begin{equation} \label{det_RNN_disc}
    g^\delta(\vecc{x}) = \mathrm{argmax}_{i=1,\dots,d_y} p^i(V \bar{h}^\delta_M[\vecc{x}]),
\end{equation}

where $p^i(x) = e^{x^i} / \sum_j e^{x^j}$ is the softmax function. Letting $\ell$ denoting the cross-entropy loss, such a classifier is trained from $\mathcal{S}_N$ by minimizing the empirical risk (training error),
$\mathcal{R}_N(g^\delta) \coloneqq \frac{1}{N} \sum_{n=1}^N \ell(g^\delta(\vecc{x}_n), y_n)$,
as a proxy for the true (population) risk (test error), $\mathcal{R}(g^\delta) = \mathbb{E}_{(\vecc{x},y)\sim\mu}\ell(g^\delta(\vecc{x}),y)$, with $(\vecc{x},y) \in \mathcal{S}$.  The measure used to quantify the prediction quality is the generalization error (or estimation error), which is the difference between the empirical risk of the classifier on the training set and the true risk: $\mathrm{GE}(g^\delta) := |\mathcal{R}(g^\delta) - \mathcal{R}_N(g^\delta)|$.

The classifier is a  function of the output of the deterministic RNN, which is an Euler discretization of the ODE \eqref{eq:ContRNN} with step sizes $\delta = (\delta_m)$.  In particular, for the Lipschitz RNN, 
\begin{equation} \label{disc_phi}
    \hat{\Phi}_{m,k} =  \hat{J}_m \hat{J}_{m-1} \cdots \hat{J}_k,  
\end{equation}
where  $\hat{J}_l =  I + \delta_l (A + D_l W)$, with  $D_l^{ij} =  a'([W \bar{h}^{\delta}_l + U \hat{x}_l +  b]^i) e_{ij}$.

In the following, we let $\conv(\mathcal{X})$ denote the convex hull of $\mathcal{X}$.  
We let $\hat{\vecc{x}}_{0:m} := (\hat{x}_0, \dots, \hat{x}_m)$ so that $\hat{\vecc{x}} = \hat{\vecc{x}}_{0:M-1}$, and use the notation $f[\vecc{x}]$ to indicate the dependence of the function $f$ on the vector $\vecc{x}$. Our result will depend on two characterizations of a training sample $s_i = (\vecc{x}_i, y_i)$.

\begin{defn}[Classification Margin] \label{defn_class_margin}
The classification margin of a training sample $s_i = (\vecc{x}_i, y_i)$ measured by the Euclidean metric $d$ is defined as the radius of the largest $d$-metric ball in $\mathcal{X}$ centered at $\vecc{x}_i$ that is contained in the decision region associated with the class label $y_i$, i.e., it is: 
$\gamma^d(s_i) = \sup\{ a: d(\vecc{x}_i, \vecc{x}) \leq a \Rightarrow g^\delta(\vecc{x}) = y_i \ \ \forall \vecc{x}\}.$
\end{defn}

Intuitively, a larger classification margin allows a classifier to associate a larger region centered on a point $\vecc{x}_i$ in the input space to the same class. 
This makes the classifier less sensitive to input perturbations, and a perturbation of $\vecc{x}_i$ is still likely to fall within this region, keeping the classifier prediction. 
In this sense, the classifier becomes more robust. In our case, the networks are trained by a loss (cross-entropy) that promotes separation of different classes in the network output. This, in turn, maximizes a certain notion of score of each training sample.

\begin{defn}[Score] \label{defn_score}
For a  training sample $s_i = (\vecc{x}_i, y_i)$, we define its score as $o(s_i) = \min_{j \neq y_i } \sqrt{2} (e_{y_i} - e_j)^T  S^\delta[\vecc{x}_i] \geq 0,$
where $e_i \in \RR^{d_y}$ is the Kronecker delta vector with  $e_i^i = 1$ and $e_i^j = 0$ for $i\neq j$, $S^\delta[\vecc{x}_i] :=  p(V \bar{h}^\delta_M[\vecc{x}_i ])$ with $\bar{h}^\delta_M[\vecc{x}_i]$ denoting the hidden state of the RNN, driven by the input sequence $\vecc{x}_i$, at terminal index $M$. 
\end{defn}

Recall that the classifier $g^\delta(\vecc{x}) = \arg \max_{i \in 1,\dots, d_y} [S^\delta]^i[\vecc{x}]$,
and the decision boundary between class $i$ and class $j$ in the feature  space is given by the hyperplane $\{ z = S^\delta : z^i = z^j\}$.  A positive score implies that at the network output, classes are separated by  a margin that corresponds to the score. However, a large score may not imply a large classification margin.

Following the approach of \cite{sokolic2017robust, xu2012robustness}, we  obtain the  second main result, providing bounds for classification margin  for the deterministic RNN classifiers $g^\delta$. We also provide a generalization bound in terms of the classification margin under additional assumptions (see Theorem 11 in {\bf SM}).

\begin{thm}
\label{thm_gen_discrete}
Suppose that Assumption \ref{ass:Smooth} holds. 
Assume that the score $o(s_i) > 0 $ and
\begin{equation}
\label{eq:classmarginbound}
    \gamma(s_i) :=  \frac{o(s_i)}{ C \sum_{m=0}^{M-1} \delta_m \sup_{\hat{\vecc{x}} \in \conv(\mathcal{X})}
    \|\hat{\Phi}_{M,m+1}[\hat{\vecc{x}}]\|_2} > 0,
\end{equation}
where $C = \|V\|_2 \left( \max_{m=0,1,\dots,M-1}  \left\| \frac{\partial f(\bar{h}^\delta_m, \hat{x}_m)}{\partial \hat{x}_m}   \right\|_2 \right) > 0$
is independent of $s_i$ (in particular, $C = \|V\|_2 [\max_{m=0,\dots,M-1}  \|D_m U\|_2 ]$ for Lipschitz RNNs),
the $\hat{\Phi}_{m,k}$ are defined in \eqref{disc_phi} and the $\delta_m$ are the step sizes.
Then, the classification margin for the training sample $s_i$:
\begin{equation} \label{ub_classmargin}
\gamma^d(s_i)  \geq \gamma(s_i). 
\end{equation}
\end{thm}

Now,  recalling from Section \ref{sec:ImpReg}, up to $\mathcal{O}(\epsilon^2)$ and under the assumption that $\hat{Q}$ vanishes, the loss minimized by the NRNN classifer is, on average, $\ell(\bar{h}^{\delta}_M) + \epsilon^2 \hat{R}(\bar{h}^{\delta})$, as $\epsilon \to 0$, with regularizer
\begin{equation} \label{regularizer_disc}
    \hat{R}(\bar{h}^{\delta}) = \frac{1}{2} \sum_{m=1}^{M} \delta_{m-1} \|\hat{M}_{M-1} \hat{\Phi}_{M-1,m} \sigma_{m-1}\|_F^2,
\end{equation}
where $\hat{M}_M^T \hat{M}_M \coloneqq H_{\bar{h}^{\delta}_{M}} l$ is  the Cholesky decomposition of the Hessian matrix of the convex cross-entropy loss.  
The appearance of the state-to-state Jacobians in $\Phi_{m,k}$ in both the regularizer (\ref{regularizer_disc}) and the lower bound (\ref{eq:classmarginbound}) suggests that  noise injection implicitly aids generalization performance. 
More precisely, in the small noise regime and on average, NRNNs promote classifiers with large classification margin, an attribute linked to both improved robustness and generalization \cite{xu2012robustness}. In this sense, training with NRNN classifiers is a stochastic strategy to improve generalization over deterministic RNN classifiers, particularly in learning tasks where the given data is corrupted (c.f. the caveats pointed out in \cite{stutz2019disentangling}). 
 
Theorem \ref{thm_gen_discrete} implies that the lower bound for the classification margin is determined by the spectrum of the $\hat{\Phi}_{M-1,m}$. To make the lower bound large, keeping $\delta_m$ and $M$ fixed,  the spectral norm of the $\hat{\Phi}_{M-1,m}$ should be made small. Doing so improves stability of the RNN, but may also lead to vanishing gradients, hindering capacity of the model to learn. 
To maximize the lower bound while avoiding the vanishing gradient problem, one should tune the numerical step sizes $\delta_m$ and noise level $\epsilon$ in NRNN appropriately. RNN architectures for the drift which help to ensure moderate Jacobians (e.g. $\|\hat{\Phi}_{M-1,m}\|_2 \approx 1$ for all $m$  \cite{chen2018dynamical}) also remain valuable in this respect.

\section{Stability and Noise-Induced Stabilization}
\label{sect_stability}
 
Here we obtain sufficient conditions to guarantee stochastic stability of CT-NRNNs. This will also provide another lens to highlight the potential of NRNNs for improved robustness.
A dynamical system is considered \emph{stable} if trajectories which are close to each other initially remain close at subsequent times. As observed in \cite{pascanu2013difficulty,miller2018stable,chang2019antisymmetricrnn}, stability plays an essential role in the study of RNNs to avoid the \emph{exploding gradient problem}, a property of unstable systems where the gradient increases in magnitude with the depth. While gradient clipping during training can somewhat alleviate this issue, better performance and robustness is achieved by enforcing stability in the model itself. 

Our stability analysis will focus on establishing \emph{almost sure exponential stability} (for other notions of stability, see {\bf SM}) for  CT-NRNNs with the drift function (\ref{eq:LipschitzRNN}). To preface the definition, consider initializing the SDE at two different random variables $h_0$ and $h_0' := h_0 + \epsilon_0$, where $\epsilon_0 \in \RR^{d_h}$ is a constant non-random perturbation 
with $\|\epsilon_0\| \leq \delta$. The resulting hidden states, $h_t$ and $h_t'$, are set to satisfy (\ref{NLRNN}) with the same Brownian motion $B_t$, starting from their initial values $h_0$ and $h_0'$, respectively. 
The evolution of $\epsilon_t = h_t' - h_t$ satisfies 
\begin{equation}
    \dd\epsilon_t = A\epsilon_t \dd t + \Delta a_t(\epsilon_t) \dd t  + \Delta \sigma_t(\epsilon_t) \dd B_t, \label{eq_perturb}
\end{equation}
where $\Delta a_t(\epsilon_t) = a(Wh'_t + Ux_t + b) - a(Wh_t + Ux_t + b)$ and $\Delta\sigma_t(\epsilon_t) = \sigma(h_t+\epsilon_t,x_t) - \sigma(h_t,x_t)$.
Since $\Delta a_t(0) = 0$, $\Delta\sigma_t(0) = 0$ for all $t \in [0,T]$, $\epsilon_t = 0$ admits a trivial \emph{equilibrium} for \eqref{eq_perturb}.  
Our objective is to analyze the stability of the solution $\epsilon_t = 0$, that is, to see how the final state $\epsilon_T$ (and hence the output of the RNN) changes for an arbitrarily small initial perturbation $\epsilon_0 \neq 0$. 
To this end, we consider an extension of the Lyapunov exponent to SDEs at the level of sample path \cite{mao2007stochastic}.

\begin{defn}[Almost sure global exponential stability]
The sample (or pathwise) Lyapunov exponent of the trivial solution of \eqref{eq_perturb} is
$\Lambda = \limsup_{t \to \infty} t^{-1} \log \|\epsilon_t\|$. 
The trivial solution $\epsilon_t = 0$ is \emph{almost surely globally exponentially stable} if $\Lambda$ is almost surely negative for all $\epsilon_0 \in \RR^{d_h}$. 
\end{defn}

For the sample Lyapunov exponent $\Lambda(\omega)$, there is a constant $C>0$ and a random variable $0 \leq \tau(\omega) < \infty$ such that for all $t > \tau(\omega)$, $\|\epsilon_t\| = \|h_t'-h_t\| \leq  C e^{\Lambda t}$ almost surely. Therefore, almost sure exponential stability implies that almost all sample paths of (\ref{eq_perturb}) will tend to the equilibrium solution $\epsilon = 0$ exponentially fast. With this definition in tow, we obtain the following stability result.

\begin{thm} 
\label{prop:Stability}
Assume that $a$ is monotone non-decreasing, and
$\sigma_1 \|\epsilon\| \leq  \|\Delta \sigma_t(\epsilon)\|_F  \leq \sigma_2 \|\epsilon\|$ for all nonzero $\epsilon \in \mathbb{R}^{d_h}$, $t \in [0,T]$.
Then for any $\epsilon_0 \in \mathbb{R}^{d_h}$, with probability one,
\begin{equation} \label{ineq_cor}
    \phi + \lambda_{\min}(A^{\sym}) \leq \Lambda \leq \psi + L_a\sigma_{\max}(W) + \lambda_{\max}(A^{\sym}),
\end{equation} 
with $\phi = -\sigma_2^2 + \frac{\sigma_1^2}{2}$ and $\psi = -\sigma_1^2+\frac{\sigma_2^2}{2}$,
where $L_a$ is the Lipschitz constant of $a$. 
\end{thm}

In the special case without noise ($\sigma_1 = \sigma_2 = 0$), we recover case (a) of Theorem 1 in \cite{erichson2020lipschitz}: when $A^\sym$ is negative definite and $\sigma_{\min}(A^\sym) > L_a \sigma_{\max}(W)$, Theorem \ref{prop:Stability} implies that (\ref{NLRNN}) is exponentially stable. 
Most strikingly, and similar to \cite{liu2019neural}, Theorem \ref{prop:Stability} implies that even if the deterministic CT-RNN is not exponentially stable, it can be stabilized through a stochastic perturbation.
Consequently, injecting noise appropriately can improve training performance.

\section{Empirical Results}
\label{sec:experiments}

The evaluation of robustness of neural networks (RNNs in particular) is an often neglected yet crucial aspect. In this section, we investigate the robustness of NRNNs and compare their performance to other recently introduced state-of-the-art models on both clean and corrupted data. 
We refer to Section G in {\bf SM} for further details of our experiments.

Here, we study the sensitivity of different RNN models with respect to a sequence of perturbed inputs during inference time. We consider different types of perturbations: (a) white noise; (b) multiplicative white noise; (c) salt and pepper; and (d) adversarial perturbations. 
To be more concrete, let $x$ be a sequence. The perturbations in consideration are as follows.

\begin{itemize}[leftmargin=*]
	\item Additive \emph{white noise perturbations} are constructed as $\tilde{x} = x + \Delta x$, where the additive noise is drawn from a Gaussian distribution $\Delta x \sim \mathcal{N}(0,\sigma)$. This perturbation strategy emulates measurement errors that can result from data acquisition with poor sensors (where $\sigma$ can be used to vary the strength of these errors). 
	 Multiplicative \emph{white noise perturbations} are constructed as $\tilde{x} = x \cdot \Delta x$, where the additive noise is drawn from a Gaussian distribution $\Delta x \sim \mathcal{N}(1,\sigma_M)$.
	
	\item \emph{Salt and pepper perturbations} emulate defective pixels that result from converting analog signals to digital signals. The noise model takes the form
	$\mathbb{P}(\tilde{X}=X)=1 - \alpha$, and $\mathbb{P}(\tilde{X}=\max) = \mathbb{P}(\tilde{X}=\min)=\alpha / 2,$
	where $\tilde{X}(i,j)$ denotes the corrupted image and $\min$ and $\max$ denote to the minimum and maximum pixel values. The parameter $\alpha$ controls the proportion of defective pixels.
	\item \emph{Adversarial perturbations} are ``worst-case'' non-random perturbations 
	maximizing the loss $\ell(g^\delta(X+\Delta X),y)$ 
	subject to the constraint that the norm of the perturbation $\|\Delta X\| \leq r$. We consider the fast gradient sign method for constructing these perturbations \cite{szegedy2013intriguing}.
\end{itemize}

We consider in addition to the NRNN three other RNNs derived from continuous-time models, including the Lipschitz RNN~\cite{erichson2020lipschitz} (the deterministic counterpart to our NRNN), the coupled oscillatory RNN (coRNN)~\cite{rusch2021coupled} and the antisymmetric RNN~\cite{chang2019antisymmetricrnn}. We also consider the exponential RNN~\cite{lezcano2019cheap}, a discrete-time model that uses orthogonal recurrent weights.
We train each model with the prescribed tuning parameters for the ordered (see Sec.~\ref{sec:orderedMNIST}) and permuted (see \textbf{SM}) MNIST task.
For the Electrocardiogram (ECG) classification task we performed a non-exhaustive hyper-tuning parameter search. 
For comparison, we train all models with hidden-to-hidden weight matrices of dimension $d_h=128$. 
We average the classification performance over ten different seed values. 

\subsection{Ordered Pixel-by-Pixel MNIST Classification}\label{sec:orderedMNIST}

First, we consider the ordered pixel-by-pixel MNIST classification task \cite{le2015simple}.
This task sequentially presents $784$ pixels to the model and uses the final hidden state to predict the class membership probability of the input image. 
In the \textbf{SM} we present additional results for the situation when instead of an ordered sequence a fixed random permutation of the input sequence is presented to the model.

\begin{table*}[!b]
	\caption{Robustness w.r.t. white noise ($\sigma$) and S\&P ($\alpha$) perturbations on the ordered MNIST task.}
	\label{tab:mnist-table}
	\centering
	\scalebox{0.8}{
		\begin{tabular}{l c c c c | c c c c c c}
			\toprule
			Name                  &  clean & $\sigma=0.1$ & $\sigma=0.2$ & $\sigma=0.3$ & $\alpha=0.03$ & $\alpha=0.05$ & $\alpha=0.1$\\
			\midrule 
			
			Antisymmetric RNN~\cite{chang2019antisymmetricrnn}  & 97.5\% & 45.7\% & 22.3\% & 17.0\%  & 77.1\% & 63.9\% & 42.6\%  \\
			
			CoRNN~\cite{rusch2021coupled}  & 99.1\% & 96.6\% & 61.9\% & 32.1\%  & 95.6\% & 88.1\% & 58.9\%  \\
			
			Exponential RNN~\cite{lezcano2019cheap} & 96.7\% & 86.7\% & 58.1\% & 33.3\%  & 83.6\% & 70.7\% & 43.4\%  \\
			
			Lipschitz RNN~\cite{erichson2020lipschitz}   & \textbf{99.2}\% & 98.4\% & 78.9\% & 47.1\%  & 97.6\% & 93.4\% & 73.5\%  \\
			
			NRNN (mult./add. noise: 0.02/0.02)   & 99.1\% & \textbf{98.9}\% & 88.4\% & 62.9\%  & 98.3\% & 95.6\% & 78.7\%  \\		
			
			NRNN (mult./add. noise: 0.02/0.05)   & 99.1\% & \textbf{98.9}\% & \textbf{92.2}\% & \textbf{73.5}\%  & \textbf{98.5}\% & \textbf{97.1}\% & \textbf{85.5}\%  \\
			
			\bottomrule
	\end{tabular}}
\end{table*}

\begin{table*}[!b]
	\caption{Robustness w.r.t. adversarial perturbations on the ordered pixel-by-pixel MNIST task.}
	\label{tab:mnist-table-fsgm}
	\centering
	\scalebox{0.8}{
		\begin{tabular}{l c c c c c c}
			\toprule
			Name                  &  $r=0.01$ & $r=0.05$ & $r=0.1$ & $r=0.15$\\
			\midrule

			Antisymmetric RNN~\cite{chang2019antisymmetricrnn}  & 79.4\% & 24.7\% & 11.4\% & 10.2\%  \\
			
			CoRNN~\cite{rusch2021coupled}  & 97.5\% & 85.5\% & 55.9\% & 35.1\%  \\
			
			Exponential RNN~\cite{lezcano2019cheap} & 94.5\% & 59.3\% & 19.7\% & 14.3\%  \\
			
			Lipschitz RNN~\cite{erichson2020lipschitz}   & 98.1\% & 85.7\% & 58.9\% & 37.1\%  \\
			
			NRNN (mult./add. noise: 0.02/0.02)   & \textbf{98.8}\% & 94.3\% & 79.6\% & 58.3\%  \\		
			
			NRNN (mult./add. noise: 0.02/0.05)   & \textbf{98.8}\% & \textbf{95.5}\% & \textbf{86.8}\% & \textbf{70.6}\%  \\
			
			\bottomrule
	\end{tabular}}
\end{table*}

\begin{figure*}[!b]
	\centering
	\begin{subfigure}[t]{0.49\textwidth}
		\centering
		\begin{overpic}[width=1\textwidth]{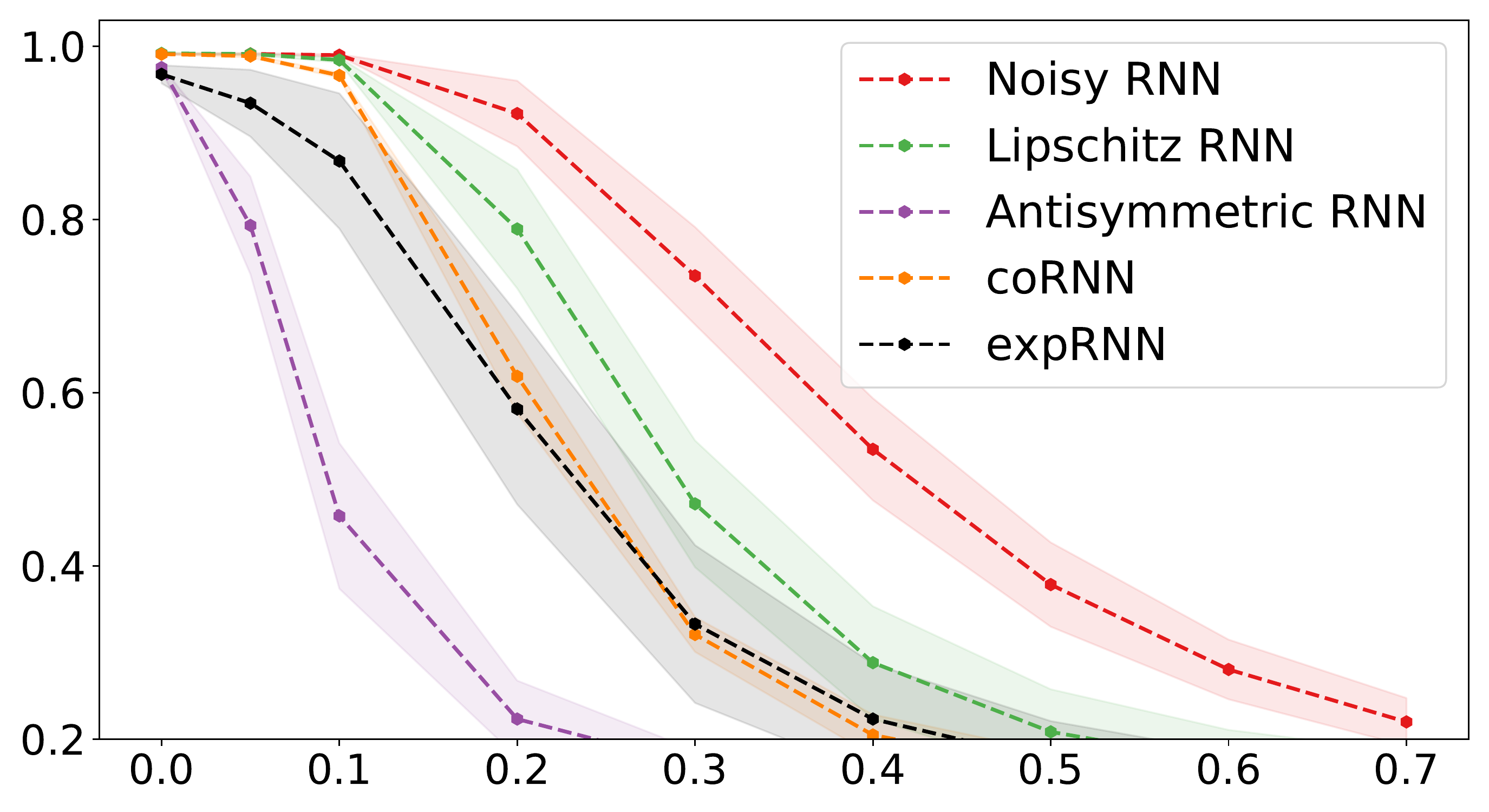}
			\put(-6,15){\rotatebox{90}{\footnotesize test accuracy}}			
			\put(42,-3){\footnotesize {amount of noise}}  	
		\end{overpic}\vspace{+0.2cm}		
		\caption{White noise perturbations.}
	\end{subfigure}
	~
	\begin{subfigure}[t]{0.49\textwidth}
		\centering
		\begin{overpic}[width=1\textwidth]{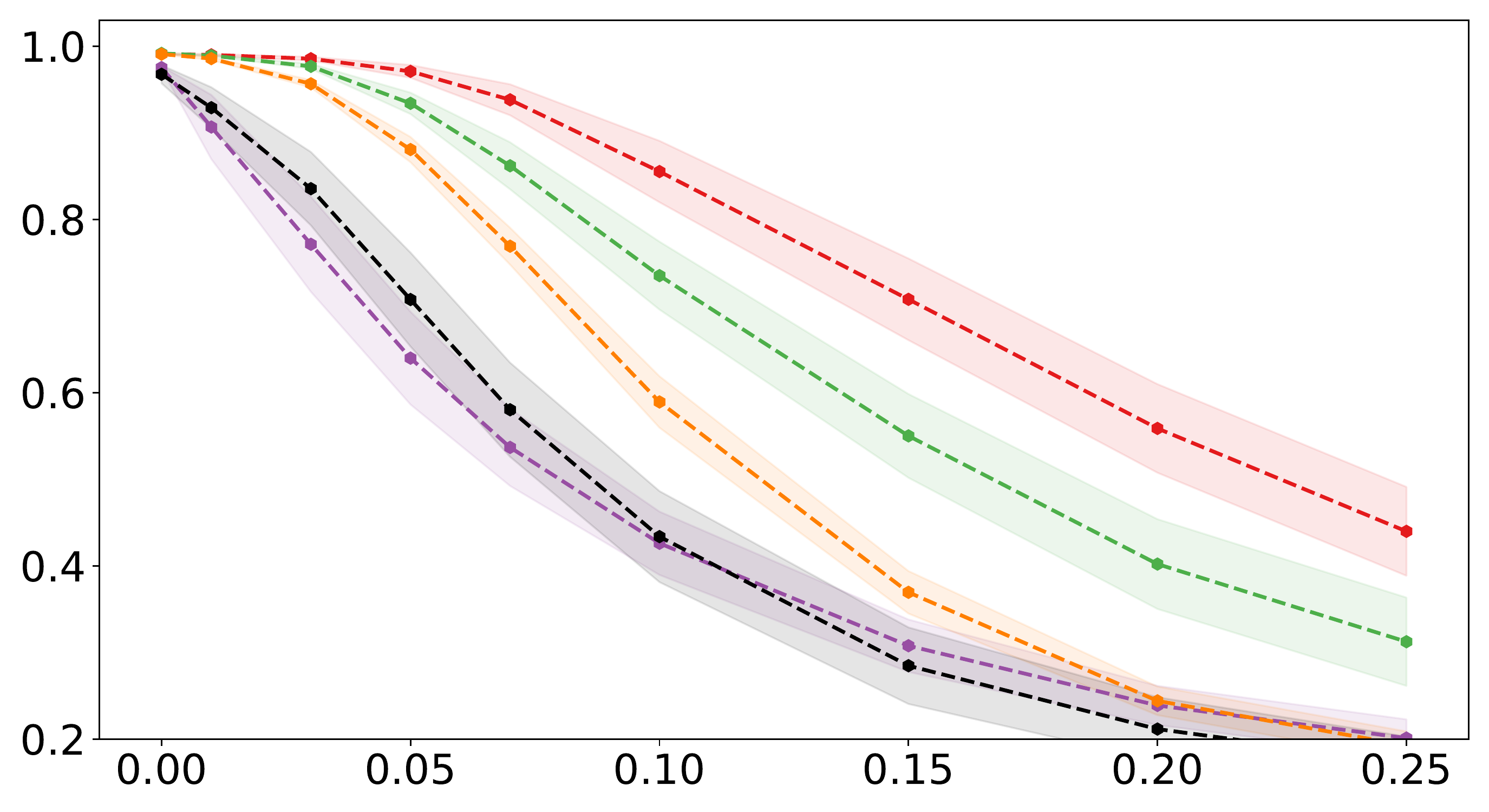} 
			\put(42,-3){\footnotesize {amount of noise}}  		
		\end{overpic}\vspace{+0.2cm}			
		\caption{Salt and pepper perturbations.}
	\end{subfigure}	
	\caption{Test accuracy for the ordered MNIST task as function of the strength of input perturbations.}
	\label{fig:smnist}
\end{figure*}

Table~\ref{tab:mnist-table} shows the average test accuracy (evaluated for models that are trained with 10 different seed values) for the ordered task. Here we present results for white noise and salt and pepper (S\&P) perturbations.
While the Lipschitz RNN performs best on clean input sequences, the NRNNs show an improved resilience to input perturbations. Here, we consider two different configuration for the NRNN. In both cases, we set the multiplicative noise level to $0.02$, whereas we consider the additive noise levels $0.02$ and $0.05$. We chose these configurations as they appear to provide a good trade-off between accuracy and robustness. Note, that the predictive accuracy on clean inputs starts to drop when the noise level becomes too large.

Table~\ref{tab:mnist-table-fsgm} shows the average test accuracy for the ordered MNIST task for adversarial perturbations. Again, the NRNNs show a superior resilience even to large perturbations, whereas the Antisymmetric and Exponential RNN appear to be sensitive even to small perturbations.

Figure~\ref{fig:smnist} summarizes the performance of different models with respect to white noise and salt and pepper perturbations. The colored bands indicate $\pm 1$ standard deviation around the average performance.  In all cases, the NRNN appears to be less sensitive to input perturbations as compared to the other models, while maintaining state-of-the-art performance for clean inputs.

\subsection{Electrocardiogram (ECG) Classification}

Next, we consider the Electrocardiogram (ECG) classification task that aims to discriminate between normal and abnormal heart beats of a patient that has severe congestive heart failure~\cite{goldberger2000physiobank}. We use $500$ sequences of length $140$ for training, $500$ sequences for validation, and $4000$ sequences for testing.

Table~\ref{tab:ecg-table} shows the average test accuracy (evaluated for models that are trained with 10 different seed values) for this task. We present results for additive white noise and multiplicative white noise perturbations.
Here, the NRNN, trained with multiplicative noise level set to $0.03$ and additive noise levels set to $0.06$, performs best both on clean as well as on perturbed input sequences. 

Figure~\ref{fig:ecg} summarizes the performance of different models with respect to additive and multiplicative white noise perturbations. Again, the NRNN appears to be less sensitive to input perturbations as compared to the other models, while achieving state-of-the-art performance for clean inputs.

\begin{table*}[!t]
	\caption{Robustness w.r.t. white ($\sigma$) and multiplicative ($\sigma_M$) noise perturbations on the ECG task.}
	\label{tab:ecg-table}
	\centering
	\scalebox{0.8}{
		\begin{tabular}{l c c c c | c c c c c c}
			\toprule
			Name                  &  clean & $\sigma=0.4$ & $\sigma=0.8$ & $\sigma=1.2$ & $\sigma_M=0.4$ & $\sigma_M=0.8$ & $\sigma_M=1.2$\\
			\midrule 
			
			Antisymmetric RNN~\cite{chang2019antisymmetricrnn}  & 97.1\% & 96.6\% & 91.6\% & 77.0\%  & 96.6\% & 94.6\% & 91.2\%  \\
			
			CoRNN~\cite{rusch2021coupled}  & 97.5\% & 96.8\% & 92.9\% & 87.2\%  & 93.9\% & 85.4\% & 78.4\%  \\
			
			Exponential RNN~\cite{lezcano2019cheap} & 97.4\% & 95.6\% & 86.4\% & 76.7\%  & 95.7\% & 89.4\% & 81.3\%  \\
			
			Lipschitz RNN~\cite{erichson2020lipschitz} & \textbf{97.7}\% & 97.4\% & 95.1\% & 88.9\%  & 97.6\% & 97.0\% & 95.6\%  \\
			
			NRNN (mult./add. noise: 0.03/0.06) & \textbf{97.7}\% & \textbf{97.5}\% & \textbf{96.3}\% & \textbf{92.6}\%  & \textbf{97.7}\% & \textbf{97.3}\% & \textbf{96.5}\%  \\

			\bottomrule
	\end{tabular}}
\end{table*}

\begin{figure*}[!t]
	\centering
	\begin{subfigure}[t]{0.49\textwidth}
		\centering
		\begin{overpic}[width=1\textwidth]{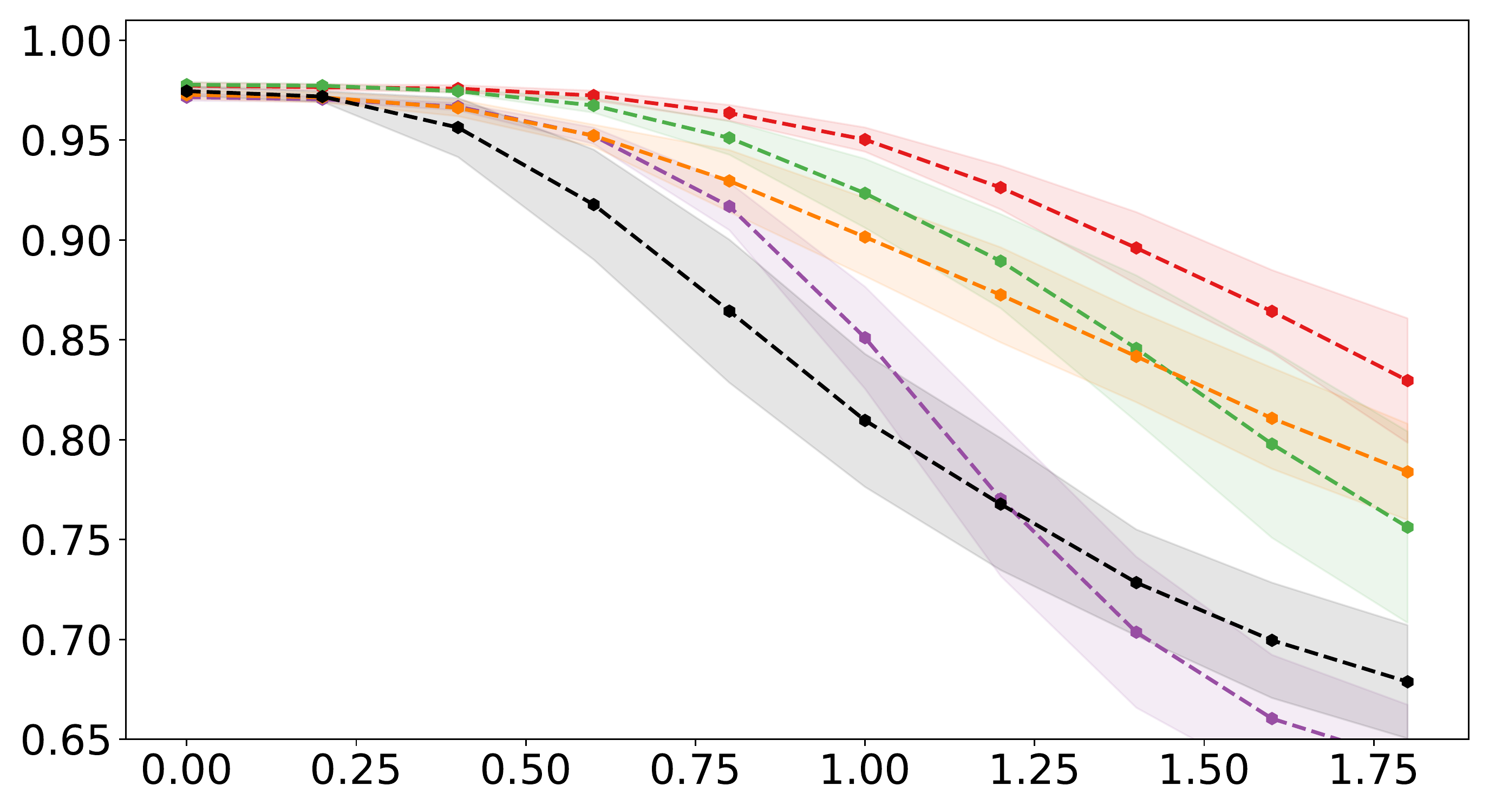}
			\put(-6,15){\rotatebox{90}{\footnotesize test accuracy}}			
			\put(42,-3){\footnotesize {amount of noise}}  	
		\end{overpic}\vspace{+0.2cm}		
		\caption{Additive white noise perturbations.}
	\end{subfigure}
	~
	\begin{subfigure}[t]{0.49\textwidth}
		\centering
		\begin{overpic}[width=1\textwidth]{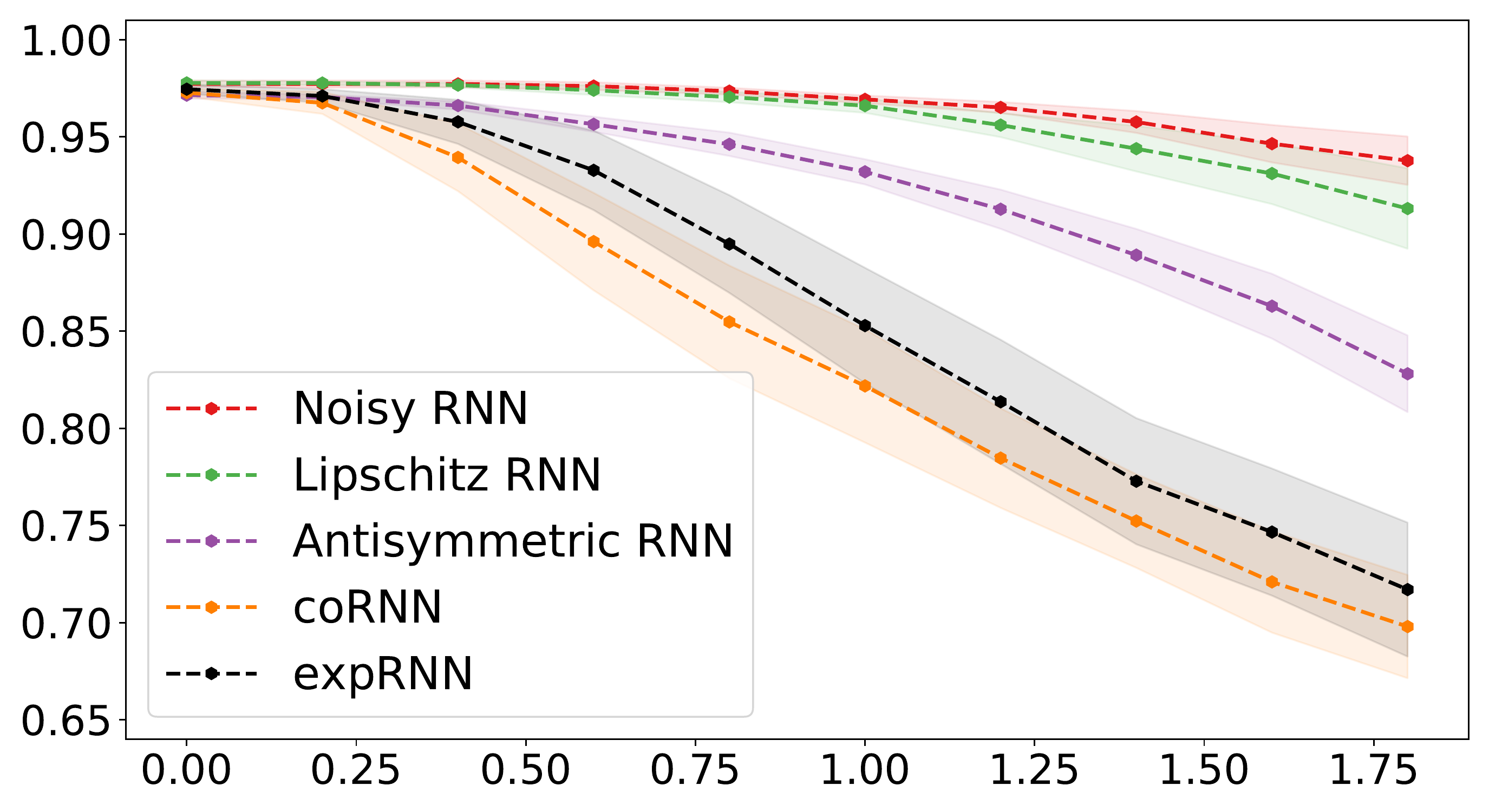} 
			\put(42,-3){\footnotesize {amount of noise}}  		
		\end{overpic}\vspace{+0.2cm}			
		\caption{Multiplicative white noise perturbations.}
	\end{subfigure}	
	\caption{Test accuracy for the ECG task as function of the strength of input perturbations.}
	\label{fig:ecg}
\end{figure*}

\section{Conclusion}
\label{sec:conclusion}

In this paper we provide a thorough theoretical analysis of RNNs trained by injecting noise into the hidden states. Within the framework of SDEs, we study the regularizing effects of general noise injection schemes. 
The experimental results are in agreement with our theory and its implications, finding that Noisy RNNs achieve superior robustness to input perturbations, while maintaining state-of-the-art generalization performance.
We believe our framework can be used to guide the principled design of a class of reliable and robust RNN classifiers.
As our work is mainly theoretical, it does not present any foreseeable societal consequence.

Our work opens up a range of interesting future directions. In particular, for deterministic RNNs, it was shown that the models learn optimally near the edge of stability \cite{chen2018dynamical}. One could extend these analyses to NRNNs with the ultimate goal of improving their performance. On the other hand, as discussed in Section \ref{sec:classif}, although the noise is shown here to implicitly stabilize RNNs, it could negatively impact capacity for long-term memory \cite{miller2018stable,zhao2020rnn}. 
Providing analyses to account for this and the implicit bias due to the stochastic optimization procedure \cite{smith2021origin,emami2021implicit} is the subject of future work.
%

\section*{Acknowledgements}
We are grateful for the generous support from Amazon AWS.
S. H. Lim would like to acknowledge  Nordita Fellowship 2018-2021 for providing support of this work.
N. B. Erichson, L. Hodgkinson, and M. W. Mahoney would like to acknowledge the IARPA (contract W911NF20C0035), ARO, NSF, and and ONR via its BRC on RandNLA for providing partial support of this work.
Our conclusions do not necessarily reflect the position or the policy of our sponsors, and no official endorsement should be inferred.


\bibliographystyle{plain}
\bibliography{ref_RNN}

\newpage

\begin{center}
    {\Large \bf  Supplementary Material (SM)}
\end{center}

\appendix

\section{Notation and Background}
\label{app_sect_notation_background}

We begin by introducing some notations that will be used in this {\bf SM}.

\begin{itemize}
\item $\| \cdot \|_F$ denotes Frobenius norm, $\|\cdot\|_p$ denote $p$-norm ($p>0$) of a vector/matrix (in particular,  $\|v\| := \|v\|_2$ denotes Euclidean norm of the vector $v$ and $\|A\|_2$ denotes the spectral norm of the matrix $A$). 
\item The $i$th element of a vector $v$ is denoted as $v^i$ or $[v]^i$ and the $(i,j)$-entry of a matrix $A$ is denoted as $A^{ij}$ or $[A]^{ij}$. 
\item $I$ denotes identity matrix (the dimension should be clear from the context).
\item $\tr$ denotes trace, the superscript $T$ denotes transposition, and $\RR^+ := (0,\infty)$. 
\item For a function 
$f : \RR^n \to \RR^m$ such that each of its first-order partial derivatives (with respect to $x$) exist on $\RR^n$, $\frac{\partial f}{\partial x} \in \RR^{m \times n}$ denotes the Jacobian matrix of $f$. 
\item For a scalar-valued function $g: \RR^n \to \RR$,
$\nabla_h g$ denotes gradient of $g$ with respect to the variable $h \in \RR^n$ and $H_h g$ denotes Hessian of $g$ with respect to $h$.
\item The notation a.s. means $\mathbb{P}$-almost surely and $\mathbb{E}$ is expectation with respect to $\mathbb{P}$, where $\mathbb{P}$ is an underlying probability measure. 
\item For a matrix $M$,  $M^\sym = (M + M^T)/2$ denote its symmetric part, $\lambda_{\min}(M)$ and $\lambda_{\max}(M)$ denote its minimum and maximum eigenvalue  respectively, and $\sigma_{\min}(M)$ and $\sigma_{\max}(M)$ denote its minimum and maximum singular value respectively. 
\item For a vector $v = (v^1, \dots, v^d)$, diag($v$) denotes the diagonal matrix with the $i$th diagonal entry equal $v^i$.
\item $\boldsymbol{1}$ denotes a vector with all entries equal to one. 
\item $e_{ij}$ denotes the Kronecker delta.
\item $\mathcal{C}(I;J)$ denotes the space of continuous $J$-valued functions defined on $I$.
\item $\mathcal{C}^{2,1}(\mathcal{D} \times I; J)$ denotes  the space of all $J$-valued functions $V(x, t)$ defined on $\mathcal{D} \times I$ which are continuously twice differentiable in $x \in \mathcal{D}$ and
once differentiable in $t \in I$.
\end{itemize}
 
Next, we recall the RNN models considered in the main paper. 

\noindent {\bf Continuous-Time NRNNs.}  For a terminal time $T > 0$ and an input signal $x = (x_t)_{t \in [0,T]} \in \mathcal{C}([0,T]; \mathbb{R}^{d_x})$, the output $y_t   \in  \mathbb{R}^{d_y}$, for $t \in [0,T]$, is a linear map of the hidden states $h_t \in \mathbb{R}^{d_h}$ satisfying the It\^{o} stochastic differential equation (SDE):
\begin{equation}
\label{app_NLRNN}
\dd h_t = f(h_t, x_t) \dd t + \sigma(h_t,x_t) \dd B_t, \qquad y_t = V h_t,
\end{equation}
where $V \in \mathbb{R}^{d_y \times d_h}$,  $f:\mathbb{R}^{d_h} \times \mathbb{R}^{d_x} \to \mathbb{R}^{d_h}$, $\sigma:\mathbb{R}^{d_h} \times \mathbb{R}^{d_x} \to \mathbb{R}^{d_h \times r}$ and $(B_t)_{t \geq 0}$ is an $r$-dimensional Wiener process. 

In particular, as an example and for empirical experiments, we focus on the choice of drift function:
\begin{equation}
\label{eq:app_LipschitzRNN}
f(h, x) = A h + a(W h + U x + b),
\end{equation}
where $a:\mathbb{R}\to\mathbb{R}$ is a Lipschitz continuous scalar activation function (such as $\tanh$) extended to act on vectors pointwise, $A,W \in \mathbb{R}^{d_h \times d_h}$, $U \in \mathbb{R}^{d_h \times d_x}$ and $b \in \mathbb{R}^{d_h}$, and the choice of diffusion coefficient: 
\begin{equation}\label{app_experiment_NRNN}
    \sigma(h,x) = \epsilon (\sigma_1 I + \sigma_2 \diag(f(h,x))), 
\end{equation}
where the noise level $\epsilon > 0$ is small, and $\sigma_1 \geq 0$ and $\sigma_2 \geq 0$ are tunable parameters describing the relative strength of additive noise and a multiplicative noise respectively. \\

We consider the following NRNN models by discretizing the SDE \eqref{app_NLRNN}, as discussed in detail in the main paper.   

\noindent {\bf Discrete-Time NRNNs.} Let $0 \coloneqq t_0 < t_1 < \cdots < t_M \coloneqq T$ be a partition of the interval $[0,T]$. Denote $\delta_m := t_{m+1} - t_m$ for each $m=0,1,\dots,M-1$, and $\delta := (\delta_m)$. The  Euler-Mayurama (E-M) scheme provides a family (parametrized by $\delta$) of approximations to the solution of the SDE in \eqref{app_NLRNN}:
\begin{equation}\label{app_e-m}
    h^{\delta}_{m+1} = h^{\delta}_{m} + f(h^{\delta}_{m},\hat{x}_{m}) \delta_m + \sigma(h^{\delta}_{m},\hat{x}_{m}) \sqrt{\delta_m} \xi_m,
\end{equation}
for $m=0,1,\dots,M-1$, where $(\hat{x}_{m})_{m=0,\dots,M-1}$ is a given sequential data, the $\xi_m \sim \mathcal{N}(0,I)$ are independent $r$-dimensional standard normal random vectors, and $h^{\delta}_0 = h_{0}$. Eq. \eqref{app_e-m} describes the update equation of our NRNN models, an example of which is when $f$ and $\sigma$ are taken to be \eqref{eq:app_LipschitzRNN} and \eqref{app_experiment_NRNN} respectively (see also the experiments in the main paper). In the special case when $\epsilon := 0$ in this example, we recover the Lipschitz RNN of \cite{erichson2020lipschitz}.  

It is worth mentioning that while higher-order integrators are also possible to consider, the presence of It\^{o} white noise poses a significant challenge over the standard ODE case.  
Generally speaking, implementations of  higher-order schemes require additional computational effort which may outweigh the benefit of using them. For instance, in implicit E-M schemes the zero of a nonlinear equation has to be determined in each time step \cite{hutzenthaler2015numerical}. In Milstein and stochastic Runge-Kutta schemes, there is an extra computational cost in simulating the L\'evy area \cite{malham2010introduction}. Similar challenges arise for other multistep schemes and higher-order schemes.\\ 

\noindent {\bf Organizational Details.} This {\bf SM} is organized as follows.

\begin{itemize}
    \item In Section \ref{app_sec:Conditions}, we provide results that guarantee existence and uniqueness of solutions to the SDE defining our continuous-time NRNNs. 
    \item In Section \ref{app_sect_discretization}, we provide results that guarantee stability and convergence of our discrete-time NRNNs. These results are in fact  very general and may be of independent interest.
    \item In Section \ref{app_sect_impl_reg}, we provide results  on implicit regularization due to noise injection in both continuous-time and discrete-time NRNNs, in particular the proof of Theorem 1 in the main paper.
    \item In Section \ref{app_sect_class_margins}, we provide some background and results to study classification margin and generalization bound of the corresponding discrete-time deterministic RNNs, in particular the proof of Theorem 2 in the main paper. 
    \item  In Section \ref{app_sec:StabilityAppendix}, we discuss stability  of continuous-time NRNNs and the noise-induced stabilization phenomenon, and provide conditions that guarantee almost sure exponential stability of the NRNNs, in particular the proof of Theorem 3 in the main paper. 
    \item In Section \ref{app_sect_experiment}, we provide details on the empirical results in the main paper and additional results.
\end{itemize}

\section{Existence and Uniqueness of Solutions}
\label{app_sec:Conditions}
 
Essential to any discussion concerning SDEs is the existence and uniqueness of solutions --- in this case, we are interested in strong solutions \cite{karatzas1998brownian}.

In the following, we fix a complete filtered probability space $(\Omega, \mathcal{F}, (\mathcal{F}_t)_{t \geq 0}, \mathbb{P})$ which satisfies the usual conditions \cite{karatzas1998brownian} and on which there is defined an  $r$-dimensional Wiener process $(B_t)_{t \geq 0}$. We also fix a $T>0$ and denote $f(h_t, t) := f(h_t, x_t)$, $\sigma(h_t,t) := \sigma(h_t,x_t)$ to emphasize the explicit dependence of the functions on time $t$ through the input  $x_t$.  

We start with the following assumptions on the SDE \eqref{app_NLRNN}.
\begin{ass} \label{ass_gen}
(a) (Global Lipschitz condition) The  coefficients $f$ and $\sigma$ are $L$-Lipschitz, i.e., there exists a constant $L > 0$ such that 
\begin{equation} \label{ass_Lip}
\|f(h,t)-f(h',t)\| + \|\sigma(h,t) - \sigma(h',t)\|_F \leq L \|h-h'\|
\end{equation}
for all $h,h' \in \RR^{d_h}$ and $t \in [0,T]$.\\
(b) (Linear growth condition)
$f$ and $\sigma$ satisfy the following linear growth condition, i.e., there exists a constant $K > 0$  such that
\begin{equation}
    \|f(h,t)\|^2 + \|\sigma(h,t)\|_F^2  \leq K (1+\|h\|^2)
\end{equation}
for all $h \in \RR^{d_h}$ and $t \in [0,T]$.
\end{ass}

Under Assumption \ref{ass_gen}, it is a standard result from stochastic analysis   that the SDE \eqref{app_NLRNN} has a unique  solution (which is a continuous and adapted process $(h_t)_{t \in [0,T]}$ satisfying the integral equation $h_t = h_0 + \int_0^t f(h_s,s) ds + \int_0^t \sigma(h_s,s) dB_s$)
for every initial value $h_0 \in \RR^{d_h}$, for  $t \in [0,T]$ (see, for instance, Theorem 3.1 in Section 2.3 of \cite{mao2007stochastic}). The uniqueness is in the sense that for any other solution $h'_t$ satisfying the SDE, 
\begin{equation}
    \mathbb{P}[h_t = h_t' \  \text{for all } t \in [0,T]] = 1.
\end{equation}

For our purpose, the following conditions suffice to satisfy Assumption \ref{ass_gen}.

\begin{ass} \label{ass}
The function $a: \RR \to \RR$ is an activation function (i.e., a non-constant and Lipschitz continuous function), and $\sigma$ is  $L_\sigma$-Lipschitz for some $L_\sigma > 0$.
\end{ass}

\begin{lem}
Consider the SDE \eqref{app_NLRNN} defining our CT-NRNN. Then, under Assumption \ref{ass}, Assumption \ref{ass_gen}  is satisfied.
\end{lem}
\begin{proof}
Note that  $f(h,t) = Ah + a(Wh + Ux_t + b)$, where $a$ is an activation function. For any $t \in [0,T]$,
\begin{align}
    \|f(h,t)-f(h',t)\| &\leq \|A(h-h')\| + \|a(Wh+Ux_t+b) - a(Wh' + Ux_t + b)\| \\
    &\leq \|A\| \|h-h'\| + L_a \|W(h-h')\| \\
    &\leq (\|A\| + L_a\|W\|) \|h-h'\|,  
\end{align}
for all $h, h' \in \RR^{d_h}$, where $L_a > 0$ is the Lipschitz constant of the (non-constant) activation function $a$. Therefore, the condition (a) in Assumption \ref{ass_gen} is satisfied since by our assumption $\sigma$ is $L_{\sigma}$-Lipschitz for some constant $L_\sigma > 0$. In this case one can take $L = \max(\|A\| + L_a\|W\|, L_\sigma)$ in Eq. \eqref{ass_Lip}. 

Since $f$ and $\sigma$ are $L$-Lipschitz, they  satisfy the linear growth condition (b) in Assumption \ref{ass_gen}. Indeed, if $f$ is $L$-Lipschitz, then for $t \in [0,T]$,
\begin{align}
\|f(h,t)\| &= \|f(h,t) - f(0,t) + f(0,t)\| \leq L\|h\| + \|f(0,t)\| \leq L\|h\|+C_f,
\end{align}
for some constant $C_f\in (0,\infty)$, where we have used the fact that $f(0,t) = a(Ux_t+b)$ is bounded for $t \in [0,T]$ (since continuous functions on compact sets are bounded). So,
\begin{align}
\|f(h,t)\|^2 &\leq (L\|h\| + C_f)^2 \leq L^2\|h\|^2 + 2LC_f \|h\| + C_f^2.
\end{align}
For $\|h\| \geq 1$, we have:
\begin{equation}
    \|f(h,t)\|^2 \leq (L^2 + 2L C_f)\|h\|^2 + C_f^2 \leq (L^2 + 2L C_f + C_f^2) (1+\|h\|^2).
\end{equation}
For $\|h\| < 1$, we have:
\begin{equation}
    \|f(h,t)\|^2 \leq (L^2 + 2L C_f)\|h\| + C_f^2 \leq  (L^2 + 2L C_f) + C_f^2 \leq (L^2 + 2L C_f + C_f^2)(1+\|h\|^2).
\end{equation}
Choosing $K= L^2 + 2LC_f + C_f^2$ gives us the linear growth condition for $f$.

Similarly, one can show that $\sigma$ satisfies the linear growth condition. The proof is done.
\end{proof}

Throughout the paper, we work with SDEs satisfying Assumption \ref{ass}. The following additional assumption on the SDEs will be needed and  invoked.

\begin{ass}\label{ass_add}
For $t \in [0,T]$,
the partial derivatives of the coefficients $f^i(h,t)$, $\sigma^{ij}(h,t)$ with respect to $h$ up to order three (inclusive) exist. Moreover, the coefficients $f^i(h,t)$, $\sigma^{ij}(h,t)$ and all these partial derivatives  are: 
\begin{itemize}
\item[(i)] bounded and Borel measurable in $t$, for  fixed $h \in \RR^{d_h}$; 
\item[(ii)] Lipschitz continuous in $h$, for  fixed $t \in [0,T]$.
\end{itemize}
\end{ass}

In particular, Assumption \ref{ass_add} implies that these partial derivatives (with respect to $h$) of $f$ and $\sigma$ satisfy (a)-(b) in Assumption \ref{ass_gen}. Assumption \ref{ass_add} holds for SDEs with commonly used activation functions such as hyperbolic tangent. We remark that Assumption \ref{ass}-\ref{ass_add} may be weakened in various directions (for instance, to  locally Lipschitz coefficients)  but for the purpose of this paper we need not go beyond these assumptions.

\section{Stability and Convergence of the Euler-Maruyama Schemes} 
\label{app_sect_discretization}

We provide stability and strong convergence results for the explicit Euler-Mayurama (E-M) approximations of the SDE \eqref{app_NLRNN}, which is  time-inhomogeneous due to the dependence of the drift and possibly diffusion coefficient on a time-varying input, here. Intuitively, strong convergence results ensure that the approximated path follows the continuous path accurately, in contrast to weak convergence results which can only guarantee this at the level of probability distribution. The latest version of strong convergence results for time-homogeneous SDEs can be found in \cite{fang2019adaptive,fang2020adaptive}. The results for our time-inhomogeneous SDEs can be obtained by adapting the proof in \cite{fang2019adaptive} without much difficulty. Since we  cannot find them in the literature, we provide them in this section. 

First, we recall the discretization scheme.
Let $0 := t_0 < t_1 < \cdots < t_M :=T$ and $t_{m+1} = t_m + \delta_m$, for $m=0,1,\dots,M-1$ and some time step $\delta_m  >  0$.
Note that we work at full generality here since  the step sizes $\delta_m$ are not necessarily uniform and may even depend on the numerical solution, i.e., $\delta_m = \delta(h^{\delta}_{m})$ (see Example \ref{ex_adaptive_scheme}). The general results will be of independent interest, in particular for further explorations in designing other variants of NRNNs. 

For $m=0,1,\dots,M-1$, consider
\begin{equation} 
    h^{\delta}_{m+1} = h^{\delta}_{m} +     f(h^{\delta}_{m}, \hat{x}_{m})  \delta_m  +  \sigma(h^{\delta}_{m},\hat{x}_{m}) \Delta B_m,
\end{equation}
where $\Delta B_m := B_{t_{m+1}}-B_{t_{m}}$, $(\hat{x}_{m})_{m=0,1,\dots,M-1}$ is a given input sequential data, and $h^{\delta}_{0} = h_0$.

Let $\underline{t} = \max\{t_m: t_m \leq t\}$, $m_t = \max\{m: t_m \leq t\}$ for the nearest time point before time $t$, and its index. Denote the piecewise constant interpolant process $\bar{h}_t = h^{\delta}_{\underline{t}}$. It is convenient to use continuous-time approximations, so  we consider the continuous interpolant that satisfies:
\begin{equation} \label{interpolant}
   h^{\delta}_t = h^{\delta}_{\underline{t}} + f(h^{\delta}_{\underline{t}},x_{\underline{t}}) (t-\underline{t}) + \sigma(h^{\delta}_{\underline{t}},x_{\underline{t}}) (B_t - B_{\underline{t}}),
\end{equation}
so that $h^{\delta}_t$ is the solution of the SDE:
\begin{equation}
    dh^{\delta}_t = f(h^{\delta}_{\underline{t}},x_{\underline{t}}) dt + \sigma(h^{\delta}_{\underline{t}},x_{\underline{t}}) dB_t = f(\bar{h}_t, x_t) dt + \sigma(\bar{h}_t, x_t) dB_t.
\end{equation}

We make the following assumptions about the time step. 

\begin{ass} \label{ass_timestep}
The (possibly adaptive) time step function $\delta: \RR^{d_h} \to \RR^+$ is continuous and strictly positive, and there exist constants $\alpha,\beta > 0$ such that for all $h \in \RR^{d_h}$, $\delta$ satisfies
\begin{equation}
    \langle h, f(h,t) \rangle + \frac{1}{2} \delta(h) \|f(h,t)\|^2 \leq \alpha \|h\|^2 + \beta
\end{equation}
for every $t \in [0,T]$.
\end{ass}

Note that if another time step function $\delta^\epsilon(h)$ is smaller than $\delta(h)$, then $\delta^\epsilon(h)$ also satisfies Assumption \ref{ass_timestep}.

A simple adaptation of the proof of Theorem 2.1.1 in \cite{fang2019adaptive} to our case of time-inhomogeneous SDE gives the following result. 

\begin{prop}[Finite-time stability] Under Assumption \ref{ass_gen} and Assumption \ref{ass_timestep}, $T$ is a.s. attainable (i.e., for $\omega \in \Omega$, $\mathbb{P}[\exists N(\omega) < \infty$ s.t. $t_{N(\omega)} \geq T] = 1$) and for all $p > 0$ there exists a constant $C > 0$ (depending on only $p$ and $T$) such that 
\begin{equation}
    \mathbb{E}\left[\sup_{t \in [0,T]} \|h^{\delta}_t\|^p \right] \leq C. 
\end{equation}
\end{prop}

This is the discrete-time analogue of the result that $\mathbb{E}\left[\sup_{t \in [0,T]} \|h_t\|^p \right] < \infty$ for all $p>0$, which can be proven by simply adapting the proof of Lemma 2.1.1. in \cite{fang2019adaptive}.

In the case where the time step is adaptive,  we take the following lower bound on the time step to bound the expected number of time steps (how quickly $\delta(h) \to 0$ as $\|h\| \to 0$). 

\begin{ass} \label{ass_lb_timestep}
There exist constants $a,b,q > 0$ such that the adaptive time step function satisfies:
\begin{equation}
    \delta(h) \geq \frac{1}{a\|h\|^q  + b}. 
\end{equation}
\end{ass}

Next, we provide strong convergence result for the numerical approximation with the time step $\delta$. When the time step $\delta$ is adaptive, one needs to rescale the time step function by a small scalar-valued magnitude $\epsilon > 0$ and then consider the limit as $\epsilon \to 0$. Following \cite{fang2019adaptive}, we make the following assumption. 

\begin{ass} \label{ass_rescaled}
The rescaled time step function $\delta^\epsilon$ satisfies
\begin{equation}
\epsilon \min(T,\delta(h)) \leq \delta^\epsilon(h) \leq \min(\epsilon T, \delta(h)),
\end{equation}
where $\delta$ satisfies Assumption \ref{ass_timestep}-\ref{ass_lb_timestep}.
\end{ass}

Under this additional assumption, we have the following convergence result, which can be proven by adapting the proof of Theorem 2.1.2 in \cite{fang2019adaptive} to our time-inhomogeneous SDE case. The proof is based on the argument used for the uniform time step analysis (see Theorem 2.2 in \cite{higham2002strong}), taking into account the adaptive nature of the time step appropriately.

\begin{thm}[Strong convergence] 
Let the SDE \eqref{app_NLRNN} satisfy Assumption \ref{ass_gen} and the time step function satisfy Assumption \ref{ass_rescaled}. Then, for all $p>0$, 
\begin{equation}
    \lim_{\epsilon \to 0} \mathbb{E} \left[ \sup_{t \in [0,T]} \|h^{\delta^\epsilon}_t - h_t \|^p \right] = 0,
\end{equation}
where $h^{\delta}_t$ is the continuous interpolant satisfying \eqref{interpolant} and $h_t$ satisfies the SDE \eqref{app_NLRNN}.
\end{thm}

In particular, the non-adaptive time stepping scheme  satisfies the above assumptions. Therefore, stability and strong convergence of the schemes are guaranteed by the above results.

Under stronger assumptions on the drift $f$ we can obtain the order of strong convergence for the numerical schemes; see Theorem 2.1.3 in \cite{fang2019adaptive} for the case of time-homogeneous SDEs. This result can be adapted to our case to obtain order-$\frac{1}{2}$ strong convergence, which is also obtained in the special case when the step sizes are uniform (see Theorem 10.2.2 in \cite{kloeden2013numerical}). 

\begin{thm}[Strong convergence rate] Assume that $f$ satisfies the following one-sided Lipschitz condition, i.e., there exists a constant $\alpha > 0$ such that for all $h, h' \in \RR^{d_h}$, 
\begin{equation}
    \langle h-h', f(h,t)-f(h',t) \rangle \leq \alpha \|h-h'\|^2
\end{equation}
for all $t \in [0,T]$, and the following locally polynomial growth Lipschitz condition, i.e., there exists $\gamma, \mu, q > 0 $ such that for all $h, h' \in \RR^{d_h}$, 
\begin{equation}
    \|f(h,t)-f(h',t)\| \leq (\gamma  (\|h\|^q + \|h'\|^q) + \mu ) \|h-h'\|,
\end{equation}
for all $t \in [0,T]$. Moreover, assume that $\sigma$ is globally Lipschitz and the time step function satisfies Assumption \ref{ass_rescaled}. Then, for all $p > 0$, there exists a constant $C>0$ such that 
\begin{equation}
    \mathbb{E}\left[ \sup_{t \in [0,T]} \|h_t^{\delta^\epsilon}-h_t\|^p \right] \leq C \epsilon^{p/2}.
\end{equation}
\end{thm}

Lastly, it is worth mentioning the following adaptive scheme, which may be a useful option when designing NRNNs.

\begin{ex}[Adaptive E-M] \label{ex_adaptive_scheme}
Under the same setup as the classical E-M setting, we may also introduce an adaptive step size scheme through a sequence of random vectors $d_m$. In this case,
\begin{equation}
h^{\delta}_{m+1} = h^{\delta}_{m} + d_m \odot f(h^{\delta}_{m}, \hat{x}_{m}) \Delta t_m + \sigma(h^{\delta}_{m}, \hat{x}_m)(\Delta t_m)^{1/2} \xi_m,\end{equation}
where $\odot$ denotes the pointwise (Hadamard) product, and each $d_n$ may be dependent on $h^{\delta}_{i}$ for $i \leq m$. Provided that $d_m \to (1,\dots,1)$ uniformly almost surely as $\Delta t~\to~0$, one could also obtain the same convergence as the classical E-M case. The adaptive setting allows for potentially better approximations by shrinking step sizes in places where the solution changes rapidly. An intuitive explanation for the instability of the standard E-M approximation of SDEs is that there is always a very small probability of a large Brownian increment which causes the approximation to produce a solution with undesirable growth. Using an adaptive time step eliminates this problem. Moreover, this scheme includes, in appropriate sense, the stochastic depth in \cite{lu2018beyond} (see page 9 there) and the dropout in \cite{liu2019neural} as special cases upon choosing an appropriate $d_m$ and $\sigma$. 

In particular, one can consider the  following drift-tamed E-M scheme, where all components, $d_m^i$, of the elements of the sequence are generated as a function of $h^{\delta}_{m}$, i.e.,
    \begin{equation}
        d_m = \frac{1}{\max{\{1, c_1  \|h^{\delta}_{m}\|+ c_2}\}} \boldsymbol{1},
    \end{equation}
for some $c_1, c_2 >0$. In this way, the drift term is ``tamed'' by a solution-dependent multiplicative factor no larger than one, which prevents the hidden state in the next time step from becoming too large.  This adaptive scheme is related to the one introduced in \cite{hutzenthaler2012strong} to provide an explicit numerical method that would display strong convergence in circumstances where the standard E-M method does not. Under certain conditions strong convergence of this scheme can be proven (even for SDEs with superlinearly growing drift coefficients). Other adaptive schemes include the increment-tamed scheme of \cite{hutzenthaler2015numerical} and many others.
\end{ex}

\section{Implicit Regularization in NRNNs}
\label{app_sect_impl_reg}
As discussed in the main paper, although the learning is carried out in discrete time, it is worth studying the continuous-time setting. The results for the continuous-time case may provide alternative perspectives and, more importantly, will be useful as a reference for exploring other discretization schemes for the CT-NRNNs. In  Subsection \ref{app_subset_impl_reg_cont_time}, we study implicit regularization for the continuous-time NRNNs. In Subsection \ref{app_subset_impl_reg_disc_time}, we study implicit regularization for discrete-time NRNNs and comment on the difference between the continuous-time and discrete-time case. 

We remark that the approach presented here is standard in showing implicit regularization. The essence of the approach is to view NRNN as a training scheme for the deterministic RNN.  Also, note that had one attempted to conduct an analysis based on NRNN directly, the resulting bound would be stochastic due to the presence of the diffusion term and it is not clear  how this bound helps explaining implicit regularization.

\subsection{Continuous-Time Setting}
\label{app_subset_impl_reg_cont_time}

\noindent {\bf Main Result and Discussions.} For the sake of brevity, we denote $f_t(\cdot) := f(\cdot, x_t)$ and $\sigma_t(\cdot) := \sigma(\cdot, x_t)$ for $t \in [0,T]$ in the following. 

To begin, consider the process $(\bar{h}_t)_{t \in [0,T]}$ satisfying the following initial value problem (IVP):
\begin{equation}
    d\bar{h}_t = f_t(\bar{h}_t) dt, \ \ \bar{h}_0 = h_0. 
\end{equation}
Let $\Psi$ denote the unique fundamental matrix satisfying the following properties: for $0\leq s \leq u \leq t \leq T$,
\begin{itemize}
    \item[(a)] \begin{align} \label{ivp}
        \frac{\partial\Psi(t,s)}{\partial t}=\frac{\partial f_{t}}{\partial \bar{h}}(\bar{h}_{t})\Psi(t,s);  \hspace{0.2cm} \frac{\partial\Psi(t,s)}{\partial s}=-\Psi(t,s)\frac{\partial f_{t}}{\partial \bar{h}}(\bar{h}_{t});
    \end{align}
    \item[(b)] $\Psi(t,s)=\Psi(t,u)\Psi(u,s)$;
    \item[(c)] $\Psi(t,s)=\Psi^{-1}(s,t)$;
    \item[(d)] $\Psi(s,s)=I$.
\end{itemize}

Also, let $\Sigma(t,s) := \Psi(t,s)\sigma_s(\bar{h}_s)$ for $0 \leq s \leq t \leq T$.

The following result links the expected loss function used for training CT-NRNNs to that for training deterministic CT-RNNs when the noise amplitude is small.

\begin{thm}[Explicit regularization induced by noise injection for CT-NRNNs]
\label{prop_explreg_CT}
Under Assumption A in the main paper, 
\begin{equation} \label{exp_reg}
    \mathbb{E} \ell(h_T) = \ell(\bar{h}_T) + \frac{\epsilon^2}2 [Q(\bar{h}) + R(\bar{h})] + \mathcal{O}(\epsilon^3),
\end{equation}
as $\epsilon \to 0$, where $Q$ and $R$ are given by
\begin{align} 
    \label{Q_term}
    Q(\bar{h}) &= (\nabla l(\bar{h}_T))^T \int_0^T \dd s \  \Psi(T,s) \int_0^s \dd u \  v(u) +  (\nabla l(\bar{h}_T))^T \int_0^T \dd s \  w(s),  \\
    R(\bar{h}) &=  \int_0^T  \dd s \   \tr(\Sigma(T,s) \Sigma(T,s)^\top \nabla^2 \ell(\bar{h}_T)), \label{R_term}
\end{align}
with $v(u)$ a vector with the $p$th component ($p=1,2,\dots,d_h)$:
\begin{equation}
    v^p(u) = tr(\Sigma(s,u) \Sigma^T(s,u) \nabla^2 [f_s]^p(\bar{h}_s)),
\end{equation}
and $w(s)$ a vector with the $q$th component ($q=1,2,\dots,d_h$):
\begin{equation}
    w^q(s) = \sum_{k=1}^r \sum_{j,l=1}^{d_h} \Psi_{k,l}^{qj}(T,s) \partial_l \sigma_s^{jk}(\bar{h}_s) \sigma_s^{lk}(\bar{h}_s).
\end{equation}
\end{thm}

Therefore, to study the difference between the CT-NRNNs and their deterministic version, it remains to investigate the role of $Q$ and $R$ in Theorem \ref{prop_explreg_CT}. 
If the Hessian is positive semi-definite, then $R(\bar{h})$ is also positive semi-definite and thus a viable regularizer. On the other hand, $Q(\bar{h})$ need not be non-negative.  However, by assuming that $\nabla^2 f$ and $\nabla \sigma^{ij}$ are small (that is, $f$ is approximately linear and $\sigma$ relatively independent of $\bar{h}$), then $Q$ can be perceived negligible and we may focus predominantly on $R$.  An argument of this kind was used in \cite{camuto2020explicit} in the context of Gauss-Newton Hessian approximations. In particular, $Q=0$ for linear NRNNs with additive noise. Therefore, Theorem \ref{prop_explreg_CT} essentially tells us that injecting noise to deterministic RNN is approximately equivalent to considering a  regularized objective functional. Moreover, the explicit regularizer is solely determined by the flow generated by the Jacobian $\frac{\partial f_{t}}{\partial \bar{h}}(\bar{h}_{t})$, the diffusion coefficient $\sigma_t$ and the Hessian of the loss function, all evaluated along the dynamics of the deterministic RNN.  

Under these assumptions, ignoring higher-order terms and bounding the Frobenius inner product in (\ref{R_term}), we can interpret training with CT-NRNN as an approximation of the following optimal control problem \cite{weinan2017proposal} with the running cost $C(t) \coloneqq \frac{1}{2} \tr( \sigma_t(\bar{h}_t)^T \Psi(T,t)^T \nabla^2 \ell(\bar{h}_t) \Psi(T,t) \sigma_t(\bar{h}_t))$: 
\begin{align}
    &\min \mathbb{E}_{(\vecc{x},y) \sim \mu} \left[ \ell(\bar{h}_T) +  \epsilon^2 \int_0^T C(t) dt \right] \\ 
    &\ \text{ s.t. } \dd \bar{h}_t = f_t(\bar{h}_t) \dd t, \ t \in [0,T], \ \bar{h}_0 = h_0,
\end{align}
where $(\vecc{x} := (x_t)_{t \in [0,T]}, y)$ denotes a training example drawn from the distribution $\mu$ and the minimization is with respect to the parameters (controls) in the corresponding deterministic RNN. On the other hand, we can interpret training with the deterministic RNN as the above optimal control problem with zero running cost or regularization.
Note that if the Hessian matrix is symmetric positive semi-definite, then $C(t)$ is a quadratic form with the associated metric tensor $M_t^T M_t \coloneqq \nabla^2 \ell(\bar{h}_t)$ and   
\begin{equation} \label{cost_bound}
C(t) = \frac{1}{2} \langle \Psi(T,t) \sigma_t, \Psi(T,t) \sigma_t \rangle_{M_t} = \frac{1}{2} \| M_t \Psi(T,t) \sigma_t \|_F^2   \leq  \frac{1}{2} \|\sigma_t\|_{F}^2 \|M_t\|_F^2  \|\Psi(T,t) \|_F^2.
\end{equation} 

Overall, we can see that the use of NRNNs as a regularization mechanism reduces the fundamental matrices $\Psi(T,s)$ according to the magnitude of the elements of $\sigma_t$. \\

\noindent {\bf Proof of Theorem \ref{prop_explreg_CT}.} Next, we prove Theorem \ref{prop_explreg_CT}. We will need some auxiliary results before doing so.

For a small perturbation
parameter $\epsilon>0$, the hidden states now satisfy the SDE
\[
\dd h_{t}=f_{t}(h_{t})\dd t+\epsilon\sigma_{t}(h_{t})\dd B_{t},
\]
where we have used the shorthand $f_{t}(\cdot)=f(\cdot,x_{t})$ and $\sigma_t(\cdot) = \sigma(\cdot, x_t)$. To
investigate the effect of the perturbation, consider the following hierarchy of differential equations:
\begin{align}
\dd h_{t}^{(0)} & =f_{t}(h_{t}^{(0)})\dd t, \label{hie1_ct}\\
\dd h_{t}^{(1)} & =\frac{\partial f_{t}}{\partial h}(h_{t}^{(0)})h_{t}^{(1)}\dd t+\sigma_{t}(h_{t}^{(0)})\dd B_{t},  \label{hie2_ct} \\
\dd h_{t}^{(2)} & =\frac{\partial f_{t}}{\partial h}(h_{t}^{(0)})h_{t}^{(2)}\dd t+\Phi_{t}^{(1)}(h_{t}^{(0)},h_{t}^{(1)})\dd t+\Phi_{t}^{(2)}(h_{t}^{(0)},h_{t}^{(1)})\dd B_{t}, \label{hie3_ct}
\end{align}
with $h_{0}^{(0)}=h_{0}$, $h_{0}^{(1)}=0$, and $h_{0}^{(2)}=0$,
and where
\begin{align}
\Phi_{t}^{(1)}(h_{0},h_{1}) & =\frac{1}{2}\sum_{i,j}\frac{\partial^{2}f_{t}}{\partial h^{i}\partial h^{j}}(h_{0})h_{1}^{i}h_{1}^{j}\\
\Phi_{t}^{(2)}(h_{0},h_{1}) & =\sum_{i}\frac{\partial\sigma_{t}}{\partial h^{i}}(h_{0})h_{1}^{i}.
\end{align}
In the sequel, we will suppose Assumption A in the main paper (which is equivalent to Assumption \ref{ass_gen} and Assumption \ref{ass_add}) holds. Under this assumption,
each of these initial value problems have a unique solution for $t\in[0,T]$.
The processes $h_{t}^{(0)}$,
$h_{t}^{(1)}$ and $h_{t}^{(2)}$ denote the zeroth-, first-, and
second-order terms in an expansion of $h_{t}$ about $\epsilon=0$.
This can be easily seen using Kunita's theory of stochastic flows.
In particular, by Theorem 3.1 in \cite{kunita1984stochastic}, letting $h_{\epsilon,t}^{(1)}=\frac{\partial h_{t}}{\partial\epsilon}$,
we find that
\[
\dd h_{\epsilon,t}^{(1)}=\frac{\partial f_{t}}{\partial h}(h_{t})h_{\epsilon,t}^{(1)}\dd t+\left(\sigma_{t}(h_{t})+\epsilon\Phi_{t}^{(2)}(h_{t},h_{\epsilon,t}^{(1)})\right)\dd B_{t},
\]
and so we find that $h_{0,t}^{(1)}=h_{t}^{(1)}$. Similarly, $h_{\epsilon,t}^{(2)}=\frac{\partial^{2}h_{t}}{\partial\epsilon^{2}}=\frac{\partial h_{\epsilon,t}^{(1)}}{\partial\epsilon}$
can be shown to satisfy
\begin{align*}
\dd h_{\epsilon,t}^{(2)} & =\frac{\partial f_{t}}{\partial h}(h_{t})h_{\epsilon,t}^{(2)}\dd t+2\Phi_{t}^{(1)}(h_{t},h_{\epsilon,t}^{(1)})\dd t  +\left(2\Phi_{t}^{(2)}(h_{t},h_{\epsilon,t}^{(1)})+\epsilon\sum_{k}[h_{\epsilon,t}^{(1)}]^{k}\frac{\partial}{\partial h^{k}}\Phi_{t}^{(2)}(h_{t},h_{\epsilon,t}^{(1)})\right)\dd B_{t}.
\end{align*}
This equation is obtained by applying Theorem 3.1 in \cite{kunita1984stochastic} to find
the first derivative of the system $(h_{t},h_{\epsilon,t}^{(1)})$
with respect to $\epsilon$ and projecting to the second coordinate.
Taking $\epsilon=0$, we find that $h_{0,t}^{(2)}=2h_{t}^{(2)}$.
Therefore, informally, a pathwise second-order Taylor expansion about
$\epsilon=0$ reveals that $h_{t}=h_{t}^{(0)}+\epsilon h_{t}^{(1)}+\epsilon^{2}h_{t}^{(2)}+\mathcal{O}(\epsilon^{3})$.
To formalize this statement, we will later bound the third-order error
term in Lemma \ref{lem_expansion}. 

While the equation for $h_{t}^{(0)}$  is
not explicitly solvable, both $h_{t}^{(1)}$ and $h_{t}^{(2)}$ are.
In particular, for $t\in[0,T]$ (see Eq. (4.28) in  \cite{sarkka2019applied}):
\begin{align}
h_{t}^{(1)} & =\int_{0}^{t}\Psi(t,s)\sigma_{s}(h_{s}^{(0)})\dd B_{s} = \int_0^t \Sigma(t,s) \dd B_s, \label{eq:Ht1sol}\\
h_{t}^{(2)} & =\int_{0}^{t}\Psi(t,s)\Phi_s^{(1)}(h_s^{(0)},h_s^{(1)})\label{eq:Ht2sol}\dd s+\int_{0}^{t}\Psi(t,s)\Phi_s^{(2)}(h_s^{(0)},h_s^{(1)})\dd B_s.
\end{align}

The key result needed to prove Theorem \ref{prop_explreg_CT} is contained  in the following theorem. In the sequel, big $\mathcal{O}$ notation
is to be understood in the almost sure sense. 

\begin{thm}\label{thm_sdeexpansion}
For a scalar-valued loss function $\ell\in\mathcal{C}^{2}(\mathbb{R}^{d_{h}})$,
for $t\in[0,T]$,
\begin{align*}
\ell(h_{t}) & =\ell(h_{t}^{(0)})+\epsilon\nabla\ell(h_{t}^{(0)})\cdot h_{t}^{(1)} + \epsilon^{2}\left(\nabla\ell(h_{t}^{(0)})\cdot h_{t}^{(2)}+\frac{1}{2}(h_{t}^{(1)})^{\top}\nabla^{2}\ell(h_{t}^{(0)})(h_{t}^{(1)})\right)+\mathcal{O}(\epsilon^{3}),
\end{align*}
as $\epsilon\to0$.
\end{thm}

We now prove Theorem \ref{thm_sdeexpansion}.  The proof relies on two lemmas. The first bounds the solutions $h^{(i)}$ over $[0,T]$.

\begin{lem}\label{lem_bound}
For any $p>0$, $\sup_{s\in[0,T]}\|h_{s}^{(0)}\|^{p}<\infty$ and
$\mathbb{E}\sup_{s\in[0,T]}\|h_{s}^{(i)}\|^{p}<\infty$ for $i=1,2$.
\end{lem}

\begin{proof}

For $s\in[0,T]$,  $h_{s}^{(0)}=h_{0}+\int_{0}^{s}f_{u}(h_{u}^{(0)})\dd u$,
so recalling $(x+y)^{2}\leq2x^{2}+2y^{2}$ and $\|f_{t}(h)\|^{2}\leq K(1+\|h\|^{2})$,
\begin{align*}
\|h_{s}^{(0)}\|^{2} & \leq2\|h_{0}\|^{2}+2\int_{0}^{s}\|f_{u}(h_{u}^{(0)})\|^{2}\dd u\\
 & \leq2\left(\|h_{0}\|^{2}+K^{2}s+K^{2}\int_{0}^{s}\|h_{u}^{(0)}\|^{2}\dd u\right).
\end{align*}
Therefore, by Gronwall's inequality,
\begin{align*}
\|h_{s}^{(0)}\|^{2} & \leq2(\|h_{0}\|^{2}+K^{2}s)e^{2K^{2}s}\\
 & \leq2(\|h_{0}\|^{2}+K^{2}T)e^{2K^{2}T}<+\infty,
\end{align*}
and so $\sup_{s\in[0,T]}\|h_{s}^{(0)}\|<\infty$. Similarly, for $s\in[0,T]$,
\[
h_{s}^{(1)}=\int_{0}^{s}\frac{\partial f_{u}}{\partial h}(h_{u}^{(0)})h_{u}^{(1)}\dd u+\int_{0}^{s}\sigma_{u}(h_{u}^{(0)})\dd B_{u}.
\]
Therefore, for $p\geq2$ (since $(x+y)^{p}\leq2^{p-1}(x^{p}+y^{p})$ by
Jensen's inequality):
\begin{align*}
\|h_{s}^{(1)}\|^{p} & \leq2^{p-1}\int_{0}^{s}\left\lVert \frac{\partial f_{u}}{\partial h}(h_{u}^{(0)})\right\rVert ^{p}\| h_{u}^{(1)}\| ^{p}\dd u +2^{p-1}\left\lVert \int_{0}^{s}\sigma_{u}(h_{u}^{(0)})\dd B_{u}\right\rVert ^{p}.
\end{align*}
Because the It\^o integral is a continuous martingale, the Burkholder-Davis-Gundy
inequality (see Theorem 3.28 in \cite{karatzas1998brownian}) implies that for positive constants
$C_{p}$ depending only on $p$ (but not necessarily the same in each
appearance),
\begin{align*}
\mathbb{E}\sup_{s\in[0,T]}\|h_{s}^{(1)}\|^{p} & \leq C_{p}\int_{0}^{s}\mathbb{E}\sup_{s\in[0,u]}\|h_{s}^{(1)}\|^{p}\dd u+C_{p}\left(\int_{0}^{t}\|\sigma_{u}(h_{u}^{(0)})\|^{2}\dd u\right)^{p/2}
\end{align*}
An application of Gronwall's inequality yields
\[
\mathbb{E}\sup_{s\in[0,T]}\|h_{s}^{(1)}\|^{p}\leq C_{p}\left(\int_{0}^{T}\|\sigma_{u}(h_{u}^{(0)})\|^{2}\dd u\right)^{p/2}e^{C_{p}T}.
\]
Therefore, $\mathbb{E}\sup_{s\in[0,T]}\|h_{s}^{(1)}\|^{p}<\infty$
for all $p\geq2$. The $p\in(0,2)$ case follows from H\"{o}lder's inequality.
Repeating this same approach for $h_{s}^{(2)}$ completes the proof.
\end{proof}

The second of our two critical lemmas provides a pathwise expansion
of $h_{t}$ about $\epsilon$ in the vein of \cite{blagoveshchenskii1961some}. Doing so characterizes the response of the NRNN hidden states to small noise perturbations at the sample path level. It can be seen as a strengthening of Theorem 2.2 in  \cite{freidlin1998random} for our time-inhomogeneous SDEs.

\begin{lem}\label{lem_expansion}
For a fixed $\epsilon_{0}>0$, and any $0<\epsilon\leq\epsilon_{0}$,
with probability one,
\[
h_{t}=h_{t}^{(0)}+\epsilon h_{t}^{(1)}+\epsilon^{2}h_{t}^{(2)}+\epsilon^{3}R_{3}^{\epsilon}(t),
\]
where for any $p > 0$, \begin{equation}\label{sense2}\sup_{\epsilon\in(0,\epsilon_{0})}\mathbb{E}\sup_{t\in[0,T]}\|R_{3}^{\epsilon}(t)\|^{p}<\infty.\end{equation}\
\end{lem}

\begin{proof}
It suffices to show that $\sup_{\epsilon\in(0,\epsilon_{0})}\mathbb{E}\sup_{t\in[0,T]}\|R_{3}^{\epsilon}(t)\|^{p}< \infty$
for $p\geq2$ --- the $p\in(0,2)$ case follows from H\"{o}lder's inequality.
In the sequel, we shall let $K$ denote a finite number (not necessarily
the same in each appearance) depending only on $f,\sigma,T,\epsilon_{0}$,
and $p$, and therefore independent of $t,\epsilon$. 

For $\epsilon>0$, let $h_{t}^{\epsilon}=h_{t}^{(0)}+\epsilon h_{t}^{(1)}+\epsilon^{2}h_{t}^{(2)}$
and $R_{3}(t)=\epsilon^{-3}(h_{t}-h_{t}^{\epsilon})$, where $h_{t},h_{t}^{(1)},h_{t}^{(2)}$
are coupled together through the same Brownian motion. Then
\begin{align*}
\epsilon^{3}R_{3}(t) & =\int_{0}^{t}\bigg(f_{s}(h_{s})-f_{s}(h_{s}^{(0)})-\epsilon\frac{\partial f_{s}}{\partial h}(h_{s}^{(0)})h_{s}^{(1)} -\epsilon^{2}\frac{\partial f_{s}}{\partial h}(h_{s}^{(0)})h_{s}^{(2)}-\epsilon^{2}\Phi_{s}^{(1)}(h_{s}^{(0)},h_{s}^{(1)})\bigg)\dd s\\
 &\quad  \ \  +\epsilon\int_{0}^{t}\sigma_{s}(h_{s})-\sigma_{s}(h_{s}^{(0)})-\epsilon\Phi_{s}^{(2)}(h_{s}^{(0)},h_{s}^{(1)})\dd B_{s}.
\end{align*}
To simplify, we decompose $\epsilon^{3}R_{3}(t)$ into the sum of
four random variables $\theta_{i}(t)$, $i=1,\dots,4$, given by
\begin{align*}
\theta_{1}(t) & =\int_{0}^{t}\left[f_{s}(h_{s})-f_{s}(h_{s}^{\epsilon})\right]\dd s\\
\theta_{2}(t) & =\int_{0}^{t}\bigg[f_{s}(h_{s}^{\epsilon})-f_{s}(h_{s}^{(0)})-\epsilon\frac{\partial f_{s}}{\partial h}(h_{s}^{(0)})h_{s}^{(1)}-\epsilon^{2}\frac{\partial f_{s}}{\partial h}(h_{s}^{(0)})h_{s}^{(2)}-\epsilon^{2}\Phi_{s}^{(1)}(h_{s}^{(0)},h_{s}^{(1)})\bigg]\dd s\\
\theta_{3}(t) & =\epsilon\int_{0}^{t}\left[\sigma_{s}(h_{s})-\sigma_{s}(h_{s}^{(0)}+\epsilon h_{s}^{(1)})\right]\dd B_{s}\\
\theta_{4}(t) & =\epsilon\int_{0}^{t}\left[\sigma_{s}(h_{s}^{(0)}+\epsilon h_{s}^{(1)})-\sigma_{s}(h_{s}^{(0)})-\epsilon\Phi_{s}^{(2)}(h_{s}^{(0)},h_{s}^{(1)})\right]\dd B_{s}.
\end{align*}
Beginning with the more straightforward terms $\theta_{1}(t)$, $\theta_{3}(t)$,
by Lipschitz continuity of $f$,
\[
\left\lVert f_{s}(h_{s})-f_{s}(h_{s}^{\epsilon})\right\rVert \leq L_{f}\epsilon^{3}\left\lVert R_{3}(s)\right\rVert ,
\]
and so
\[
\mathbb{E}\sup_{s\in[0,t]}\left\lVert \theta_{1}(s)\right\rVert ^{p}\leq K\epsilon^{3p}\int_{0}^{t}\mathbb{E}\sup_{s\in[0,u]}\left\lVert R_{3}(s)\right\rVert \dd u.
\]
In the same way, $\| \sigma_{s}(h_{s})-\sigma_{s}(h_{s}^{(0)}+\epsilon h_{s}^{(1)})\| \leq L_{\sigma}\epsilon^{2}\| h_{s}^{(2)}+\epsilon R_{3}(s)\| $.
Recall that $(\int_{0}^{t}g(s)\dd s)^{p}\leq t^{p-1}\int_{0}^{t}g(s)^{p}\dd s$
by Jensen's inequality. Now, $\theta_{3}(s)$ is a continuous martingale,
and hence, the Burkholder-Davis-Gundy inequality (see Theorem 3.28 in \cite{karatzas1998brownian}) implies that for
some constant $C_{p}>0$ depending only on $p$,
\begin{align*}
\mathbb{E}\sup_{s\in[0,t]}\|\theta_{3}(s)\|^{p} & \leq\epsilon^{p}C_{p}\left(\int_{0}^{t}\mathbb{E}\left\lVert \sigma_{s}(h_{s})-\sigma_{s}(h_{s}^{(0)}+\epsilon h_{s}^{(1)})\right\rVert ^{2}\dd s\right)^{p/2}\\
 & \leq C_{p}L_{\sigma}^{p}\epsilon^{3p}\left(\int_{0}^{t}\mathbb{E}\| h_{s}^{(2)}+\epsilon R_{3}(s)\| ^{2}\dd s\right)^{p/2}\\
 & \leq C_{p}L_{\sigma}^{p}\epsilon^{3p}T^{p/2-1}\int_{0}^{t}\mathbb{E}\| h_{s}^{(2)}+\epsilon R_{3}(s)\| ^{p}\dd s\\
 & \leq C_{p}L_{\sigma}^{p}\epsilon^{3p}2^{p-1}T^{p/2-1}\left(\int_{0}^{t}\mathbb{E}\| h_{s}^{(2)}\|^{p}\dd s+\epsilon^{p}\int_{0}^{t}\mathbb{E}\| R_{3}(s)\| ^{p}\dd s\right).
\end{align*}
From Lemma \ref{lem_bound}, it follows that
\[
\mathbb{E}\sup_{s\in[0,t]}\|\theta_{3}(s)\|^{p}\leq K\epsilon^{3p}\left(1+\epsilon^{p}\int_{0}^{t}\mathbb{E}\sup_{s\in[0,u]}\| R_{3}(s)\| ^{p}\dd u\right).
\]
Treating the $\theta_{2}$ term next, for each $s\in[0,t]$, by Taylor's
theorem, there exists some $\epsilon_{s}\in(0,\epsilon)$ such that
\begin{align*}
&f_{s}(h_{s}^{\epsilon})-f_{s}(h_{0})  -\epsilon\frac{\partial f_{s}}{\partial h}(h_{0})h_{1} =\epsilon^{2}\frac{\partial f_{s}}{\partial h}(h_{s}^{\epsilon_{s}})h_{2}+\epsilon^{2}\Phi_{s}^{(1)}(h_{s}^{\epsilon_{s}},h_{1}).
\end{align*}
Therefore, by Lipschitz continuity of the derivatives of $f$, 
\begin{align*}
\theta_{2}(t) & =\epsilon^{2}\int_{0}^{t}\bigg(\frac{\partial f_{s}}{\partial h}(h_{s}^{\epsilon_{s}})h_{2}+\Phi_{s}^{(1)}(h_{s}^{\epsilon_{s}},h_{1})-\frac{\partial f_{s}}{\partial h}(h_{0})h_{2}-\Phi_{s}^{(1)}(h_{0},h_{1})\bigg)\dd s\\
 & \leq K\epsilon^{2}\int_{0}^{t}\|h_{s}^{\epsilon_{s}}-h_{0}\|\dd s\\
 & \leq K\epsilon^{3}\int_{0}^{t}\|h_{s}^{(1)}\|+\epsilon\|h_{s}^{(2)}\|\dd s.
\end{align*}
From Lemma \ref{lem_bound}, it follows that
\[
\mathbb{E}\sup_{s\in[0,T]}\|\theta_{2}(s)\|^{p}\leq K\epsilon^{3p}.
\]
Similarly, by Taylor's theorem, there exists $\epsilon_{s}\in(0,\epsilon)$
such that
\[
\sigma_{s}(h_{s}^{(0)}+\epsilon h_{s}^{(1)})-\sigma_{s}(h_{s}^{(0)})=\epsilon\Phi_{s}^{(2)}(h_{s}^{(0)}+\epsilon_{s}h_{s}^{(1)},h_{s}^{(1)}),
\]
and so for $p\geq2$, by the Burkholder-Davis-Gundy inequality and
Lipschitz continuity of the derivatives of $\sigma$,
\begin{align*}
\mathbb{E}\sup_{s\in[0,t]}\|\theta_{4}(s)\|^{p} & \leq C_{p}\epsilon^{2p}\left(\int_{0}^{t}\mathbb{E}\left\lVert \Phi_{s}^{(2)}(h_{s}^{(0)}+\epsilon_{s}h_{s}^{(1)},h_{s}^{(1)})-\Phi_{s}^{(2)}(h_{s}^{(0)},h_{s}^{(1)})\right\rVert ^{2}\dd s\right)^{p/2}\\
 & \leq KC_{p}\epsilon^{3p}\left(\int_{0}^{t}\mathbb{E}\|h_{s}^{(1)}\|^{2}\dd s\right)^{p/2}\\
 & \leq KC_{p}T^{p/2}\epsilon^{3p}\mathbb{E}\sup_{s\in[0,T]}\|h_{s}^{(1)}\|^{p}\leq K\epsilon^{3p}.
\end{align*}
Combining estimates for $\theta_{1},\theta_{2},\theta_{3},\theta_{4}$,
\begin{align*}
\mathbb{E}\sup_{s\in[0,t]}\|R_{3}(s)\|^{p} & =4^{p-1}\epsilon^{-3p}\bigg(\mathbb{E}\sup_{s\in[0,t]}\|\theta_{1}(s)\|^{p}+\mathbb{E}\sup_{s\in[0,t]}\|\theta_{2}(s)\|^{p} + \mathbb{E}\sup_{s\in[0,t]}\|\theta_{3}(s)\|^{p}+\mathbb{E}\sup_{s\in[0,T]}\|\theta_{4}(s)\|^{p}\bigg)\\
 & \leq K\left(1+\int_{0}^{t}\mathbb{E}\sup_{s\in[0,u]}\left\lVert R_{3}(s)\right\rVert ^{p}\dd u\right),
\end{align*}
and so by Gronwall's inequality, $\mathbb{E}\sup_{s\in[0,t]}\|R_{3}(s)\|^{p}\leq Ke^{Kt}$.
Since $K$ is independent of $t\leq T$ and $\epsilon\leq\epsilon_{0}$,
it follows that
\[
\sup_{\epsilon\in(0,\epsilon_{0})}\mathbb{E}\sup_{s\in[0,t]}\|R_{3}(s)\|^{p}\leq Ke^{Kt}<+\infty,
\]
and the result follows. 
\end{proof}

We remark that perturbative techniques such as the one used to obtain Theorem \ref{lem_expansion} are standard in the theory of stochastic flows.

Theorem \ref{thm_sdeexpansion} now follows in a straightforward fashion from Lemma \ref{lem_expansion} by taking a second-order Taylor expansion of $\ell(h_t^{(0)}+\epsilon h_t^{(1)} + \epsilon^2 h_t^{(2)} + \mathcal{O}(\epsilon^3))$ about $\epsilon = 0$. 

We are now in a position to prove Theorem \ref{prop_explreg_CT} using Theorem \ref{thm_sdeexpansion}.

\begin{proof}[Proof of Theorem \ref{prop_explreg_CT}]
From Theorem \ref{thm_sdeexpansion}, we have, upon taking expectation:
\begin{equation}
    \mathbb{E} \ell(h_t) = \ell(h_t^{(0)}) + \epsilon (\nabla_{h^{(0)}} \ell)^T \mathbb{E} h^{(1)}_t +   \epsilon^2 \left( (\nabla_{h^{(0)}} \ell)^T \mathbb{E} h_t^{(2)} +   \frac{1}{2} \mathbb{E} (h_t^{(1)})^T  (H_{h^{(0)}} \ell) h_t^{(1)} \right) + \mathcal{O}(\epsilon^3),
\end{equation}
for $t \in [0,T]$, as $\epsilon \to 0$, where $H_{h^{(0)}}$ denotes Hessian operator and the $h_t^{(i)}$ satisfy Eq. \eqref{hie1_ct}-\eqref{hie3_ct}.

Since $\nabla f_t$ and its derivative are bounded and are thus Lipschitz continuous, by  Picard's theorem the IVP has a unique solution. Moreover, it follows from our assumptions that the solution to the IVP is  square-integrable (i.e., $\int_0^t \|\Psi(t,s)\|_F^2 ds < \infty$ for any $t \in [0,T]$). Therefore, the solution $h_t^{(1)}$ to Eq. \eqref{hie2_ct} can be uniquely represented as the following It\^o integral:
\begin{equation} \label{eq_aux}
    h_t^{(1)} = \int_0^t \Psi(t,s) \sigma(h_s^{(0)},s) dB_s,
\end{equation}
where $\Psi(t,s)$ is the (deterministic) fundamental matrix solving the IVP \eqref{ivp}. We have $\mathbb{E} h^{(1)}_t = 0$ and
\begin{equation}
    \mathbb{E} \|h_t^{(1)}\|^2 = \int_0^t \|\Psi(t,s) \sigma(h_s^{(0)},s) \|_F^2 ds < \infty. 
\end{equation}

Similar argument together with Assumption \ref{ass_add} shows that the solution $h_t^{(2)}$ to Eq. \eqref{hie3_ct} admits the following unique integral representation, with the $i$th component:
\begin{align}
    h_t^{(2)i} &= \frac{1}{2} \int_0^t  \Psi^{ij}(t,s) [h_s^{(1)}]^l \frac{\partial^2 b^j}{\partial[h_s^{(0)}]^l \partial [h_s^{(0)}]^k} [h_s^{(1)}]^k ds +  \int_0^t \Psi^{ij}(t,s) \frac{\partial \sigma^{jk}}{\partial [h_s^{(0)}]^l} [h_s^{(1)}]^l   dB_s^k,
\end{align}
where the last integral above is a uniquely defined It\^o integral.

Plugging  Eq. \eqref{eq_aux} into the above expression and then taking expectation, we have:
\begin{align}
    \mathbb{E} h_t^{(2)i} &= \frac{1}{2} \mathbb{E} \int_0^t  ds \Psi^{ij}(t,s) \frac{\partial^2 b^j}{\partial[h_s^{(0)}]^l \partial [h_s^{(0)}]^k} \int_0^s dB^{l_2}_{u_1} \int_0^s dB_{u_2}^{k_2} \Psi^{l l_1}(s,u_1) \sigma^{l_1 l_2} \sigma^{k_1 k_2} \Psi^{k k_1}(s,u_2) \nonumber \\  
    &\quad \ \  + \frac{1}{2} \mathbb{E}  \int_0^t  dB_s^k \Psi^{ij}(t,s)  \frac{\partial \sigma^{jk}}{\partial [h_s^{(0)}]^l} \int_0^t dB_u^{l_2}   \Psi^{l l_1}(s,u) \sigma^{l_1 l_2}(h_u^{(0)},u), \label{eq_50}
\end{align}
where we have performed change of variable to arrive at the last double integral above. 

Using the semigroup property of  $\Psi$, we have $\Psi(t,s) = \Psi(t,0)\Psi^{-1}(s,0)$ for any $s \leq t$ (and so $\Psi^{ij}(t,s) = \Psi^{ij_1}(t,0) (\Psi^{-1})^{j_1 j}(s,0)$ and 
$\Psi^{ll_1}(s,u) = \Psi^{ll_2}(s,0) (\Psi^{-1})^{l_2 l_1}(u,0)$ etc.). Using this property in \eqref{eq_50} and then evaluating the resulting expression using properties of moments of stochastic integrals (applying Eq. (5.7) and Proposition 4.16 in \cite{gall2016brownian} -- note that It\^o isometry follows from Eq. (5.7) there), we obtain $(\nabla_{h^{(0)}} \ell)^T \mathbb{E} h_t^{(2)} = Q(h^{(0)})$, where $Q$ satisfies Eq. \eqref{Q_term}. 

Similarly, plugging Eq. \eqref{eq_aux} into $\mathbb{E} h_t^{(1)T}(H_{h^{(0)}}\ell)h_t^{(1)}$, and then proceeding as above and applying the cyclic property of trace, give $\frac{1}{2} \mathbb{E} (h_t^{(1)})^T  (H_{h^{(0)}} \ell) h_t^{(1)} = R(h^{(0)})$, where $R$ satisfies Eq. \eqref{R_term}. The proof is done.
\end{proof}

\subsection{Discrete-Time Setting: Proof of Theorem 1 in the Main Paper}
\label{app_subset_impl_reg_disc_time}

The goal in this subsection is to prove Theorem 1 in the main paper,  the discrete-time analogue of Theorem \ref{prop_explreg_CT}. We recall the theorem in the following.

\begin{thm}[Explicit regularization induced by noise injection for discrete-time NRNNs -- Theorem 1 in the main paper] 
\label{app_thm_exp_reg_discrete}
Under Assumption A in the main paper, 
\begin{align} \label{exp_reg}
    \mathbb{E} \ell(h^{\delta}_M) &= \ell(\bar{h}^{\delta}_M) + \frac{\epsilon^2}2 [\hat{Q}(\bar{h}^{\delta}) +\hat{R}(\bar{h}^{\delta})]  + \mathcal{O}(\epsilon^3),
\end{align}
as $\epsilon \to 0$, where the terms $\hat{Q}$ and $\hat{R}$  are given by
\begin{align} \label{P_term_discrete}
    \hat{Q}(\bar{h}^{\delta}) &= (\nabla l(\bar{h}^{\delta}_M))^T  \sum_{k=1}^M \delta_{k-1} \hat{\Phi}_{M-1,k}  \sum_{m=1}^{M-1} \delta_{m-1} \vecc{v}_{m}, \\
    \hat{R}(\bar{h}^{\delta}) &= \sum_{m=1}^{M} \delta_{m-1} \tr( \sigma_{m-1}^T \hat{\Phi}^T_{M-1,m} H_{\bar{h}^{\delta}}l \ \hat{\Phi}_{M-1,m} \sigma_{m-1}),
    \label{R_term_discrete}
\end{align}
with $\vecc{v}_m$ a vector with the $p$th component $(p=1,\dots, d_h)$:
\[
    [v_m]^p = \tr(\sigma_{m-1}^T \hat{\Phi}_{M-2,m}^T H_{\bar{h}^\delta} [f_M]^{p} 
    \hat{\Phi}_{M-2,m}  \sigma_{m-1}).
\]

Moreover,
\begin{equation}
|\hat{Q}(\bar{h}^{\delta})| \leq C_Q \Delta^2, \ \ |\hat{R}(\bar{h}^{\delta})| \leq C_R \Delta,    
\end{equation}
for $C_Q, C_R >0$ independent of $\Delta$.
\end{thm}

To prove Theorem \ref{app_thm_exp_reg_discrete}, the key idea is to first obtain a discretized version of the loss function in Theorem \ref{thm_sdeexpansion}  by either discretizing the results in Theorem \ref{thm_sdeexpansion} or by proving directly from the discretized equations \eqref{app_e-m}. It then remains to  compute the expectation of this loss as functional of the  discrete-time process. The first part is straightforward while the second part involves some tedious recursive computations.

Let $0 \coloneqq t_0 < t_1 < \cdots < t_M \coloneqq T$ be a partition of the interval $[0,T]$ and let $\delta_m = t_{m+1} - t_m$ for each $m=0,1,\dots,M-1$.
For small parameter $\epsilon > 0$, the   E-M scheme is given by:
\begin{equation}\label{e-m-app}
    h^{\delta}_{m+1} = h^{\delta}_{m} + f(h^{\delta}_{m},\hat{x}_{m}) \delta_m + \epsilon \sigma(h^{\delta}_{m},\hat{x}_{m}) \sqrt{\delta_m} \xi_m,
\end{equation}
where $(\hat{x}_{m})_{m=0,\dots,M-1}$ is a given sequential data, each $\xi_m \sim \mathcal{N}(0,I)$ is an independent $r$-dimensional standard normal random vector, and $h^{\delta}_{0} = h_{0}$.

Consider the following hierarchy of recursive equations. For the sake of notation cleanliness, we replace the superscript $\delta$ by hat when denoting the $\delta$-dependent approximating solutions in the following.   

For $m=0,1,\dots,M-1$:
\begin{align}
\hat{h}_{m+1}^{(0)} &= \hat{h}_{m}^{(0)} + \delta_m f(\hat{h}_{m}^{(0)}, \hat{x}_{m}), \label{hie1_disc} \ \ \hat{h}_{0}^{(0)} = h_0, \\
\hat{h}^{(1)}_{m+1} &= \hat{J}_m \hat{h}_{m}^{(1)} + \sqrt{\delta_m} \sigma(\hat{h}_{m}^{(0)},\hat{x}_{m}) \xi_m, \  \ \hat{h}_{0}^{(1)} = 0, \label{hie2_disc} \\ 
\hat{h}^{(2)}_{m+1} &= \hat{J}_m \hat{h}_{m}^{(2)} + \sqrt{\delta_m} + \delta_m \Psi_1(\hat{h}_{m}^{(0)}, \hat{h}_{m}^{(1)}) +  \delta_m \Psi_2(\hat{h}_{m}^{(0)},\hat{h}_{m}^{(1)}) \xi_m, \  \ \hat{h}_{0}^{(2)} = 0, \label{hie3_disc}
\end{align}
where the
\begin{align}
\hat{J}_m &= I + \delta_m f'(\hat{h}_{m}^{(0)},\hat{x}_{m}) 
\end{align}
are the state-to-state Jacobians and
\begin{align}
\Psi_{1}(h_{0},h_{1}) & =\frac{1}{2}\sum_{i,j}\frac{\partial^{2}f_{m}}{\partial h^{i}\partial h^{j}}(h_{0})h_{1}^{i}h_{1}^{j}, \label{jac_1} \\
\Psi_2(h_{0},h_{1}) & =\sum_{i}\frac{\partial\sigma_{m}}{\partial h^{i}}(h_{0})h_{1}^{i}. \label{jac_2}
\end{align}

Note that the above equations can also be obtained by E-M discretization of Eq. \eqref{hie1_ct}-\eqref{hie2_ct}. 

The following theorem is a discrete-time analogue of Theorem \ref{thm_sdeexpansion}. Recall that the big $\mathcal{O}$ notation
is to be understood in the almost sure sense. 

\begin{thm}\label{thm_sdeexpansion_discrete}
Under the same assumption as before, for a scalar-valued loss function $\ell\in\mathcal{C}^{2}(\mathbb{R}^{d_{h}})$,
for $m=0,1,\dots,M-1$, we have
\begin{align}
\ell(\hat{h}_{m+1}) & =\ell(\hat{h}_{m}^{(0)})+\epsilon\nabla\ell(\hat{h}_{m}^{(0)})\cdot \hat{h}_{m}^{(1)} + \epsilon^{2}\left(\nabla\ell(\hat{h}_{m}^{(0)})\cdot \hat{h}_{m}^{(2)}+\frac{1}{2}(\hat{h}_{m}^{(1)})^{\top}\nabla^{2}\ell(\hat{h}_{m}^{(0)})(\hat{h}_{m}^{(1)})\right)+\mathcal{O}(\epsilon^{3}), \label{res_loss_disc}
\end{align}
as $\epsilon \to 0$, where the $\hat{h}^{(i)}_{m}$, $i=0,1,2$, satisfy Eq.\eqref{hie1_disc}-\eqref{hie3_disc}.
\end{thm}
\begin{proof}
The proof is analogous to the one for continuous-time case, working with the discrete-time process \eqref{e-m-app}  instead of continuous-time process.
\end{proof}

We begin by recalling a remark from the main text.
\begin{rmk} \label{rmk_interesting}
Interestingly, Theorem \ref{app_thm_exp_reg_discrete} looks like discrete-time analogue of Theorem \ref{prop_explreg_CT} for CT-RNN, except that, unlike the term $Q$ there, the  term $\hat{Q}$ for the discrete-time case  has no explicit dependence on the {\it derivative (with respect to $h$) of the noise coefficient $\sigma$}. Therefore, a direct discretization of the result in Theorem \ref{prop_explreg_CT} would not give us the correct explicit regularizer for discrete-time NRNNs. This remark highlights the difference between learning in the practical discrete-time setting versus learning in the idealized continuous-time setting with NRNNs. This also means that we need to work out an independently crafted proof for the discrete-time case.
\end{rmk}

The proof of Theorem \ref{app_thm_exp_reg_discrete}  involves some  tedious, albeit technically straightforward,  computations. The key ingredients are the recursive relations \eqref{hie1_disc}-\eqref{hie3_disc} and the property of standard Gaussian random vectors that 
\begin{equation} \label{gaussian_property}
\mathbb{E} \xi_p^l \xi_q^j = e_{pq} e_{lj},
\end{equation}
where the $e_{pq}$ denote the Kronecker delta.

To organize our proof, we begin by introducing some notation and proving a lemma.

\noindent {\bf Notation.} For $m=1,\dots,M-1$, let us denote  $f'_{m} := f'(\hat{h}_{m}^{(0)},\hat{x}_{m})$, $\sigma_{m} := \sigma(\hat{h}_{m}^{(0)},\hat{x}_{m})$,
\begin{align}
      H_{lj} f_{m}^i  &:= \frac{\partial^2 f^i(\hat{h}_{m}^{(0)},\hat{x}_{m})}{\partial [\hat{h}_{m}^{(0)}]^l \partial [\hat{h}_{m}^{(0)} ]^j }, \\ 
      D_{l} \sigma^{ij}_{m} &:=  \frac{\partial \sigma^{ij}(\hat{h}_{m}^{(0)},\hat{x}_{m})}{\partial [\hat{h}_{m}^{(0)}]^l }, 
\end{align}
and
\begin{equation}
 \hat{\Phi}_{m,k} := J_m J_{m-1} \cdots J_k, \ \ \hat{\Phi}_{k,k+1} = I,
\end{equation}
for $k=1,\dots,m$. For computational convenience, we are using Einstein's summation notation for repeated indices in the following.

\begin{lem} \label{lem_stat_hhat1}
For $m=0,1,\dots,M$, $\mathbb{E} \hat{h}_{m}^{(1)} = 0$ and 
\begin{equation}\label{variance_hhat1}
    \mathbb{E} [\hat{h}_{m}^{(1)}]^l [\hat{h}_{m}^{(1)}]^j = \delta_{m-1} \sigma^{l l_1}_{m-1} \sigma^{j l_1}_{m-1} + \sum_{k=1}^{m-1}\delta_{k-1}  \hat{\Phi}_{m-1,k}^{l l_2} \hat{\Phi}_{m-1,k}^{j j_2} \sigma^{l_2 l_3}_{k-1} \sigma^{j_2 l_3}_{k-1}.  
\end{equation}
\end{lem}

\begin{proof}

From Eq. \eqref{hie2_disc}, we have $\hat{h}_{0}^{(1)} = 0$, $\hat{h}_{1}^{(1)} = \sqrt{\delta_0} \sigma_{t_0} \xi_0$ and, upon iterating, for $m=1,\dots,M-1$,
\begin{equation} \label{hhat_one}
    \hat{h}_{m+1}^{(1)} = \sqrt{\delta_m} \sigma_{m} \xi_m + \sum_{k=1}^m \sqrt{\delta_{k-1}} \hat{\Phi}_{m,k}  \sigma_{k-1} \xi_{k-1}.
\end{equation}
The first equality in the lemma follows from taking expectation of Eq. \eqref{hhat_one}  and using the fact that the $\xi_k$ are (mean zero) standard Gaussian random variables. The second equality in the lemma follows from taking expectation of a product of components of the $\hat{h}_{m+1}^{(1)}$ in Eq. \eqref{hhat_one} and applying the property \eqref{gaussian_property}. 
\end{proof}

\begin{proof}[Proof of Theorem \ref{app_thm_exp_reg_discrete}]

Iterating Eq. \eqref{hie3_disc}, we obtain $\hat{h}_{0}^{(2)} = 0$, $\hat{h}_{1}^{(2)} = \delta_0 \Psi_1(h_0, 0) + \sqrt{\delta_0} \Psi_2(h_0,0) \xi_0$ and, for $m=1,\dots,M-1$,
\begin{align}
    \hat{h}_{m+1}^{(2)} &= \delta_m \Psi_1(\hat{h}_{m}^{(0)}, \hat{h}_{m}^{(1)}) + \sqrt{\delta_m} \Psi_2(\hat{h}_{m}^{(0)}, \hat{h}_{m}^{(1)}) + \sum_{k=1}^m \delta_{k-1}  \hat{\Phi}_{m,k}  \Psi_1(\hat{h}_{k-1}^{(0)}, \hat{h}_{k-1}^{(1)})\nonumber \\
    &\quad \ \ + \sum_{k=1}^m \sqrt{\delta_{k-1}} \hat{\Phi}_{m,k}  \Psi_1(\hat{h}_{k-1}^{(0)}, \hat{h}_{k-1}^{(1)}) \xi_{k-1}.
\end{align}

Substituting in the formulae \eqref{jac_1}-\eqref{jac_2} in the right hand side above and then using Eq. \eqref{hhat_one}:
\begin{align}
    [\hat{h}_{m+1}^{(2)}]^i &= \frac{\delta_m}{2} [\hat{h}_{m}^{(1)}]^l H_{lj} f^i_{m}  [\hat{h}_{m}^{(1)}]^j + \sum_{k=1}^m  \frac{\delta_{k-1}}{2} \hat{\Phi}_{m,k}^{ip}  [\hat{h}_{k-1}^{(1)}]^l H_{lj} f^p_{k-1 }  [\hat{h}_{k-1 }^{(1)}]^j \nonumber \\
    &\quad \ + \sqrt{\delta_m} D_l \sigma^{ij}_{m}  [\hat{h}_{m}^{(1)}]^l \xi_m^j + \sum_{k=1}^m \sqrt{\delta_{k-1}} \hat{\Phi}_{m,k}^{iq} D_l \sigma^{qr}_{k-1} [\hat{h}_{k-1 }^{(1)}]^l \xi_{k-1}^r \\
    &= \frac{\delta_m}{2} [\hat{h}_{m}^{(1)}]^l H_{lj} f^i_{m}  [\hat{h}_{m}^{(1)}]^j + \sum_{k=1}^m  \frac{\delta_{k-1}}{2} \hat{\Phi}_{m,k}^{ip}  [\hat{h}_{k-1}^{(1)}]^l H_{lj} f^p_{k-1}  [\hat{h}_{k-1}^{(1)}]^j \nonumber \\
    &\quad \ + \sqrt{\delta_m} D_l \sigma^{ij}_{m}  \xi_m^j \left( \sqrt{\delta_{m-1}} \sigma^{l l_1}_{m-1} \xi_{m-1}^{l_1} + \sum_{k=1}^{m-1} \sqrt{\delta_{k-1}} \hat{\Phi}_{m-1,k}^{l l_1} \sigma_{k-1}^{l_1 l_2} \xi_{k-1}^{l_2} \right) \nonumber \\
    &\quad \ + \sqrt{\delta_{1}} \hat{\Phi}_{m,2}^{iq} D_{l} \sigma^{qr}_{1 }  \xi_{1}^r (\sqrt{\delta_0} \sigma_{0}^{l l_1} \xi_0^{l_1})  \\  \nonumber
    &\quad \ + \sum_{k=3}^m \sqrt{\delta_{k-1}} \hat{\Phi}_{m,k}^{iq} D_{l} \sigma^{qr}_{k-1}  \xi_{k-1}^r \left( \sqrt{\delta_{k-2}} \sigma^{l p_1}_{k-2} \xi_{k-2}^{p_1} + \sum_{k'=1}^{k-2} \sqrt{\delta_{k'-1}} \hat{\Phi}_{k-2,k'}^{l p_1} \sigma_{k'-1}^{p_1 p_2} \xi_{k'-1}^{p_2} \right),
\end{align}
where we have made use of the fact that $\hat{h}_{0}^{(1)} = 0$ and $\hat{h}_{1}^{(1)} = \sqrt{\delta_0} \sigma_{0} \xi_0$ in the last two lines above to rewrite the summation (so that the summation over $k$ in the last line above starts at $k=3$).

Therefore,  using the above result, Lemma \ref{lem_stat_hhat1} and Eq. \eqref{gaussian_property}, we compute the expectation of $[\hat{h}_{m+1}^{(2)}]^i$:
\begin{align} \label{follow1}
    \mathbb{E} [\hat{h}_{m+1}^{(2)}]^i &= \frac{1}{2} \sum_{k=1}^{m+1} \delta_{k-1} \hat{\Phi}_{m,k}^{ip} H_{lj} f^p_{m} \sum_{k=1}^m \delta_{k-1} \hat{\Phi}_{m-1,k}^{l l_2} \sigma^{l_2 l_3}_{k-1} \sigma^{j_2 l_3}_{k-1} \hat{\Phi}_{m-1,k}^{j j_2}.
\end{align}

Moreover, using Lemma \ref{lem_stat_hhat1}, we obtain, for $m=1,2,\dots,M-1$,
\begin{equation} \label{follow2}
    \mathbb{E} [\hat{h}_{m+1}^{(1)} ]^l [H_{\hat{h}^{(0)}} l]^{lj} [\hat{h}_{m+1}^{(1)} ]^j = \sum_{k=1}^{m+1} \delta_{k-1} \sigma^{l_2 l_3}_{k-1} \hat{\Phi}_{m,k}^{l l_2}  [H_{\hat{h}^{(0)}} l]^{lj} \hat{\Phi}_{m,k}^{j j_2}  \sigma^{j_2 l_3}_{k-1}. 
\end{equation}

The first statement of the theorem then follows from Theorem \ref{thm_sdeexpansion_discrete} and  Eq. \eqref{follow1}-\eqref{follow2} (with $m:=M-1$):
\begin{align} 
    \hat{Q}(\bar{h}^{\delta}) &= \partial_i l(\bar{h}^{\delta}_M)  \sum_{k=1}^M \delta_{k-1} \hat{\Phi}_{M-1,k}^{ip}  \sum_{m=1}^{M-1} \delta_{m-1} \partial_{lj} [f_M]^{p} \hat{\Phi}_{M-2,m}^{l l_2} \sigma_{m-1}^{l_2 l_3} \sigma_{m-1}^{j_2 l_3}  \hat{\Phi}_{M-2,m}^{j j_2}, \\
    \hat{R}(\bar{h}^{\delta}) &= \sum_{m=1}^{M} \delta_{m-1} \sigma_{m-1}^{l_2 l_3} \hat{\Phi}_{M-1,m}^{l l_2} [H_{\bar{h}^{\delta}} l]^{lj} \hat{\Phi}_{M-1,m}^{j j_2}  \sigma_{m-1}^{j_2 l_3}.
\end{align}

The last statement of the theorem follows from taking straightforward bounds.    
\end{proof}

\begin{rmk}
We remark that the computed $\hat{h}_{m}^{(2)}$ (a key step in the above proof), like that for $h_{t}^{(2)}$ in the continuous-time case, has explicit dependence on the noise coefficient. It is only upon taking the expectation (see Eq. \eqref{follow1}) that the dependence on the noise coefficient vanishes (whereas $\mathbb{E}h_{t}^{(2)} \neq 0$ retains its dependence on the noise coefficient). This fully reconciles with Remark \ref{rmk_interesting}.
\end{rmk}

\begin{rmk}
Moreover, one can compute  the variance of $l(\hat{h}_M)$ to be $\epsilon^2 (\nabla l(\hat{h}_M^{(0)}))^T C \nabla l(\hat{h}_M^{(0)}) + \mathcal{O}(\epsilon^3)$, as $\epsilon \to 0$, where $C$ is a PSD matrix whose $(l,j)$-entry is given by Eq. \eqref{variance_hhat1} with $m:=M$. So we see that the spread of $l(\hat{h}_M)$ about its average is $\mathcal{O}(\epsilon^2)$ as $\epsilon \to 0$. 
\end{rmk}

\section{Bound on Classification Margin and a Generalization Bound for Deterministic RNNs: Proof of Theorem 2 in the Main Paper}
\label{app_sect_class_margins}

We recall the setting considered in the main paper before providing proof to the results presented there. 

Let $\mathcal{S}_N$ denote a set of training samples $s_n \coloneqq (\vecc{x}_n, y_{n})$ for $n=1,\dots,N$, where each input sequence $\vecc{x}_n = (x_{n,0},x_{n,1},\dots,x_{n,M-1}) \in \mathcal{X} \subset \mathbb{R}^{d_x M}$ has a corresponding class label $y_n \in \mathcal{Y} = \{1,\dots,d_y\}$. Following the statistical learning framework, these samples are assumed to be independently drawn from an underlying probability distribution $\mu$ on the sample space $\mathcal{S} = \mathcal{X} \times \mathcal{Y}$. An RNN-based classifier $g^\delta(\vecc{x})$ is constructed in the usual way by taking 

\begin{equation} \label{det_RNN_disc}
    g^\delta(\vecc{x}) = \mathrm{argmax}_{i=1,\dots,d_y} p^i(V \bar{h}^\delta_M[\vecc{x}]),
\end{equation}

where $p^i(x) = e^{x^i} / \sum_j e^{x^j}$ is the softmax function. Letting $\ell$ denoting the cross-entropy loss, such a classifier is trained from $\mathcal{S}_N$ by minimizing the empirical risk (training error) \[
\mathcal{R}_N(g^\delta) \coloneqq \frac{1}{N} \sum_{n=1}^N \ell(g^\delta(\vecc{x}_n), y_n)\]
as a proxy for the true (population) risk (testing error) $\mathcal{R}(g^\delta) = \mathbb{E}_{(\vecc{x},y)\sim\mu}\ell(g^\delta(\vecc{x}),y)$ with $(\vecc{x},y) \in \mathcal{S}$.

The measure used to quantify the prediction quality is the {\it generalization error} (or estimation error), which is the difference between the empirical risk of the classifier on the training set and the true risk:
\begin{equation}
    GE(g^\delta) := |\mathcal{R}(g^\delta) - \mathcal{R}_N(g^\delta)|. 
\end{equation}

The classifier is a  function of the output of the deterministic RNN, which is an Euler discretization of the ODE (1) in the main paper with step sizes $\delta = (\delta_m)$.  In particular, for the Lipschitz RNN, 
\begin{equation} \label{disc_phi}
    \hat{\Phi}_{m,k} =  \hat{J}_m \hat{J}_{m-1} \cdots \hat{J}_k,  
\end{equation}
where  $\hat{J}_l =  I + \delta_l (A + D_l W)$, with  $D_l^{ij} =  a'([W \bar{h}^{\delta}_l + U \hat{x}_l +  b]^i) e_{ij}$.

In the following, we let $\conv(\mathcal{X})$ denote the convex hull of $\mathcal{X}$.  
We denote $\hat{\vecc{x}}_{0:m} := (\hat{x}_0, \dots, \hat{x}_m)$ so that $\hat{\vecc{x}} = \hat{\vecc{x}}_{0:M-1}$, and use the notation $f[\vecc{x}]$ to indicate the dependence of the function $f$ on the vector $\vecc{x}$. Moreover, we will need the following two definitions to characterize a training sample $s_i = (\vecc{x}_i, y_i)$

Working in the above setting, we now recall and prove the  second main result in the main paper, providing bounds for classification margin  for the deterministic RNN classifiers $g^\delta$. 

\begin{thm}[Classification margin bound for the deterministic RNN -- Theorem 2 in the main paper] 
\label{app_thm_gen_discrete}
Suppose that Assumption A in the main paper holds. Assume that the $o(s_i) > 0 $ and
\begin{equation}
    \gamma(s_i) :=  \frac{o(s_i)}{ C \sum_{m=0}^{M-1} \delta_m \sup_{\hat{\vecc{x}} \in \conv(\mathcal{X})}
    \|\hat{\Phi}_{M,m+1}[\hat{\vecc{x}}] \|_2} > 0,
\end{equation}
where
 \[C = \|V\|_2 \left( \max_{m=0,1,\dots,M-1}  \left\| \frac{\partial f(\bar{h}^\delta_m, \hat{x}_m)}{\partial \hat{x}_m}   \right\|_2 \right) > 0\]
 is a constant (in particular, $C = \|V\|_2 \left(\max_{m=0,\dots,M-1}  \|D_m U\|_2 \right)$ for Lipschitz RNNs),
the $\hat{\Phi}_{m,k}$ are defined in \eqref{disc_phi} and the $\delta_m$ are the step sizes. 
Then, we have the following upper bound on the classification margin for the training sample $s_i$:
\begin{equation} \label{ub_classmargin}
\gamma^d(s_i)  \geq \gamma(s_i). 
\end{equation}
\end{thm}

Moreover, under additional assumptions one can obtain the following generalization bound, which follows from Theorem \ref{app_thm_gen_discrete}. 

\begin{thm}[A generalization bound for the deterministic RNN] 
\label{app_thm_gen_discrete_2} Under the same setting as Theorem \ref{app_thm_gen_discrete},
if we further assume that $\mathcal{X}$ is a (subset of)  $k$-dimensional manifold with $k \leq d_x M$, $\gamma := \min_{s_i \in \mathcal{S}_N} \gamma(s_i) > 0$, and $\ell(g^\delta(\vecc{x}),y) \leq L_g$ for all $s \in \mathcal{S}$, then for any $\delta' >0$, with probability at least $1-\delta'$,
\begin{align} \label{gen_bound}
    &GE(g^\delta) \leq L_g \left( \frac{1}{\gamma^{k/2}} \sqrt{ \frac{d_y C_M^k 2^{k+1} \log 2 }{N}}  + \sqrt{\frac{2 \log(1/\delta')}{N}} \right),
\end{align}
where $C_M > 0$ is a constant that measures complexity of  $\mathcal{X}$, $N$ is the number of training examples and $d_y$ is the number of label classes.
\end{thm}

\begin{rmk}
Generalization bounds involving classification margins (for RNNs in particular) are a separate topic with a significant presence in the literature. We emphasize that the generalization bound above is one of the many bounds that one can derive for RNNs. There exist much tighter bounds (for various variants of RNNs under various assumptions and settings) which may be equally applicable and lead to the same claimed conclusion, but are much more difficult to state (see, for instance, Theorem E.1 in \cite{wei2019data}). There are also other types of generalization bounds that are not obtained in terms of classification margin in the literature. Although they are interesting in their own, our focus here is on bounds that can be expressed in terms of classification margin. Therefore, meaningful comparisons between these generalization bounds are not straightforward.
\end{rmk}

In order to prove Theorem \ref{app_thm_gen_discrete} and Theorem \ref{app_thm_gen_discrete_2}, we place ourselves in the algorithmic robustness framework of \cite{xu2012robustness}. This framework provides bounds for the generalization error based on the robustness of a learning algorithm that learns a classifier $g$ by exploiting the structure of the training set $\mathcal{S}_N$. Robustness is, roughly speaking, the desirable property for a learning algorithm that if a testing sample is ``similar" to a training sample, then the testing error is close to the training error (i.e., the algorithm is insensitive to small perturbations in the training data). 

To ensure that our exposition is self-contained, we recall important definitions and results from \cite{xu2012robustness,sokolic2017generalization} to formalize the previous statement in the context of our deterministic RNNs in the following. 

\begin{defn}
Let $\mathcal{S}_N$ be a training set and $\mathcal{S}$ the sample space. A learning algorithm is $(K,\epsilon(\mathcal{S}_N))$-robust if $\mathcal{S}$ can be partitioned into $K$ disjoint sets denoted by $\mathcal{K}_k$, $k=1,\dots,K$:
\begin{align}
    &\mathcal{K}_k \subset \mathcal{S}, \ k=1,\dots,K, \\
    &\mathcal{S} = \cup_{k=1}^K \mathcal{K}_k, \ \text{ and } \ \mathcal{K}_k \cap \mathcal{K}_{k'} = \emptyset, \forall k \neq k',
\end{align}
such that for all $s_i \in \mathcal{S}_N$ and all $s \in \mathcal{S}$, 
\begin{equation}
    s_i = (\vecc{x}_i,y_i) \in \mathcal{K}_k \wedge s = (\vecc{x}, y) \in \mathcal{K}_k \implies |\ell(g(\vecc{x}_i),y_i) - \ell(g(\vecc{x}),y)| \leq \epsilon(\mathcal{S}_N).
\end{equation}
\end{defn}

The above definition says that a robust learning algorithm selects a classifier
$g$ for which the losses of any $s$ and $s_i$ in the same partition $\mathcal{K}_k$ are close. 

The following result from Theorem 1  in \cite{xu2012robustness} will be critical to the proof of Theorem \ref{app_thm_gen_discrete}. It provides a generalization bound for robust algorithms. 

\begin{thm}\label{thm_xu_gen}
If a learning algorithm is $(K,\epsilon(\mathcal{S}_N))$-robust and $\ell(g(\vecc{x}),y) \leq M$ for all $s = (\vecc{x},y) \in \mathcal{S}$, for some constant $M>0$, then for any $\delta > 0$, with probability at least $1-\delta$, 
\begin{equation}
     GE(g)  \leq \epsilon(\mathcal{S}_N) + M \sqrt{ \frac{2 K \log(2) + 2 \log(1/\delta) }{m} }.
\end{equation}
\end{thm}

Note that the above generalization bound is data-dependent, in contrast to bounds obtained via approaches based on complexity or stability arguments that give bounds in terms of data agnostic measures such as the Rademacher complexity or the VC dimension, which are found not sufficient for explaining the good generalization properties of deep neural networks. 

The number of partition $K$ in the above can be bounded in terms of the covering number of the sample space $\mathcal{S}$, which gives a way to measure the complexity of sets. We recall the definition of covering number in the following.

\begin{defn}[Covering] Let $\mathcal{A}$ be a set. We say that $\mathcal{A}$ is $\rho$-covered by a set $\mathcal{A}'$, with respect to the (pseudo-)metric $d$, if for all $a \in \mathcal{A}$, there exists $a' \in \mathcal{A}'$ with $d(a,a') \leq \rho$. We call the cardinality of the smallest $\mathcal{A}'$ that $\rho$-covers $\mathcal{A}$ covering number, denoted by $\mathcal{N}(\mathcal{S}; d, \rho)$.
\end{defn}

The covering number is the smallest number of (pseudo-)metric 
balls of radius $\rho$ needed to cover $\mathcal{S}$ and we denote it by $\mathcal{N}(\mathcal{S};d,\rho)$, where $d$ denotes the (pseudo-)metric.
The choice of metric $d$ determines how efficiently one may cover $\mathcal{X}$. For example, the Euclidean metric $d(x,x') = \|x-x'\|_2$ for $x,x' \in \mathcal{X}$. The covering number of many structured low-dimensional data models can be bounded in terms of their intrinsic properties.  Since in our case the space $\mathcal{S} = \mathcal{X} \times \mathcal{Y}$, we write $\mathcal{N}(\mathcal{S};d,\rho) \leq d_y \cdot \mathcal{N}(\mathcal{X}; d,\rho)$, where $d_y$ is the number of label classes. 
We take $d$ to be the Euclidean metric: $d(\vecc{x},\vecc{x}') = \| \vecc{x}-\vecc{x}'\|_2$ for $\vecc{x},\vecc{x}' \in \mathcal{X}$, unless stated otherwise.

\begin{lem}[Example 27.1 from \cite{shalev2014understanding}] \label{ex_covering}
Assume that $\mathcal{X} \subset \RR^m$ lies in a $k$-dimensional subspace of $\RR^m$. Let $c = \max_{x \in \mathcal{X}} \|x\|$ and take $d$ to be the Euclidean metric. Then $\mathcal{N}(\mathcal{X}; d,\rho) \leq (2c\sqrt{k}/\rho)^k$. 
\end{lem}

In other words, a subset, $\mathcal{X}$,  of a $k$-dimensional manifold has the covering number  $(C_M/\rho)^k$, where $C_M > 0$ is a constant.
We remark that other complexity measures such as  Rademacher complexity can be bounded based on the covering number (see \cite{shalev2014understanding} for details). 

The class of robust learning algorithms that is of interest to us is the large margin classifiers.  We define classification margin in the following.

\begin{defn}[Classification margin]
The classification margin of a training sample $s_i = (\vecc{x}_i, y_i)$ measured by a metric $d$ is defined as the radius of the largest $d$-metric ball in $\mathcal{X}$ centered at $\vecc{x}_i$ that is contained in the decision region associated with the class label $y_i$, i.e., it is: 
\begin{equation}
    \gamma^d(s_i) = \sup\{ a: d(\vecc{x}_i, \vecc{x}) \leq a \implies g(\vecc{x}) = y_i \ \ \forall \vecc{x}\}.
\end{equation}
\end{defn}

Intuitively, a larger classification margin allows a classifier to associate a larger region centered on a point $\vecc{x}_i$ in the input space to the same class. This makes the classifier less sensitive to input perturbations and  a noisy perturbation of $\vecc{x}_i$ is still likely to fall within this region, keeping the classifier prediction. In this sense, the classifier becomes more robust.

The following result follows from Example 9 in \cite{xu2012robustness}.

\begin{prop} \label{prop_classmargin_rob}
If there exists a $\gamma > 0$ such that $\gamma^d(s_i) > \gamma$ for all $s_i \in \mathcal{S}_N$, then the classifier $g$ is $(d_y \cdot \mathcal{N}(\mathcal{X}; d, \gamma/2), 0)$-robust.  
\end{prop}

In our case the networks are trained by a loss (cross-entropy) that promotes separation of different classes at the network output. The training aims at maximizing a certain notion of score of each training sample.

\begin{defn}[Score] \label{app_defn_score}
For a  training sample $s_i = (\vecc{x}_i, y_i)$, we define its score as
\begin{equation}
    o(s_i) = \min_{j \neq y_i } \sqrt{2} (e_{y_i} - e_j)^T  S^\delta[\vecc{x}_i] \geq 0,
\end{equation}
where $e_i \in \RR^{d_y}$ is the Kronecker delta vector with  $e_i^i = 1$ and $e_i^j = 0$ for $i \neq j$, $S^\delta[\vecc{x}_i] :=  p(V \bar{h}^\delta_M[\vecc{x}_i ])$ with $\bar{h}^\delta_M[\vecc{x}_i]$ denoting the hidden state of the RNN, driven by the input sequence $\vecc{x}_i$, at terminal index $M$. 
\end{defn}

The RNN classifier $g^\delta$ is defined as 
\begin{equation} \label{defn_classifier}
    g^\delta(\vecc{x}) = \arg \max_{i \in \{1, \dots, d_y\}} S^i[\vecc{x}],
\end{equation}
and the decision boundary between class $i$ and class $j$ in the output   space is given by the hyperplane $\{ z = p(V\bar{h}^\delta_M) : z^i = z^j\}$. A positive score implies that at the network output, classes are separated by  a margin that corresponds to the score. However, a large score may not imply a large classification margin -- recall that the classification margin is a function of the decision boundary in the input space, whereas the training algorithm aims at optimizing the decision boundary at the network output in the output space. 

We need the following lemma relating a pair of vectors in the input space and the output space. 

\begin{lem}\label{lem_Jac}
For any $\vecc{x}$, $\vecc{x}' \in \mathcal{X} \subset \RR^{d_x M}$, and a given RNN output functional $\mathcal{F}[\cdot]$, 
\begin{equation}
    \|\mathcal{F}[\vecc{x}] - \mathcal{F}[\vecc{x}']\|_2  \leq \sup_{\bar{\vecc{x}} \in \conv(\mathcal{X})} \|J[\bar{\vecc{x}}]\|_2 \cdot  \|\vecc{x} - \vecc{x}'\|_2, 
\end{equation}
where $\vecc{J}[\vecc{x}] = d\mathcal{F}[\vecc{x}]/d\vecc{x}$ is the input-output Jacobian of the RNN output functional.
\end{lem}
\begin{proof}
Let $t \in [0,1]$ and define the function $F(t) = \mathcal{F}[\vecc{x} + t(\vecc{x}'-\vecc{x})]$. Note that 
\begin{equation}
    \frac{dF(t)}{dt} = \vecc{J}[\vecc{x} + t(\vecc{x}'-\vecc{x})]  (\vecc{x}'-\vecc{x}).
\end{equation}

Therefore,
\begin{equation}
    \mathcal{F}[\vecc{x}'] - \mathcal{F}[\vecc{x}] = F(1)-F(0) = \int_0^1 \frac{dF(t)}{dt}  dt = \left( \int_0^1 \vecc{J}[\vecc{x} + t(\vecc{x}'-\vecc{x})] dt \right) (\vecc{x}'-\vecc{x}) ,
\end{equation}
where we have used the fundamental theorem of calculus. 

Now, 
\begin{align}
     \|\mathcal{F}[\vecc{x}] - \mathcal{F}[\vecc{x}']\|_2  &\leq \left\|\int_0^1 \vecc{J}[\vecc{x} + t(\vecc{x}'-\vecc{x})]   dt \right\|_2 \cdot \|(\vecc{x}'-\vecc{x})\|_2  \\
     &\leq \sup_{\vecc{x}, \vecc{x}' \in \mathcal{X}, t \in [0,1]} \| \vecc{J}[\vecc{x} + t(\vecc{x}'-\vecc{x})] \|_2  \cdot \|(\vecc{x}'-\vecc{x})\|_2  \\
     &\leq      \sup_{\bar{\vecc{x}} \in \conv(\mathcal{X})}  \|J[\bar{\vecc{x}}]\|_2  \cdot \|\vecc{x} - \vecc{x}'\|_2,
\end{align}
where we have used the fact that $\vecc{x} + t(\vecc{x}'-\vecc{x}) \in \conv(\mathcal{X})$ for all $t \in [0,1]$ to arrive at the last line. The proof is done.
\end{proof}

The classification margin depends on the score and the network's expansion and contraction
of distances around the training points. These can be quantified by studying the network's
input-output Jacobian matrix. The following proposition provides classification margin bounds in terms of the score and input-output Jacobian associated to the RNN classifier. 

\begin{prop}\label{prop_classmargin}
Assume that a RNN classifier $g^\delta(\vecc{x})$, defined in \eqref{defn_classifier}, classifies a training sample $\vecc{x}_i$ with the score $o(s_i) > 0$. Then we have the following lower bound for the classification margin:
\begin{equation}
    \gamma^d(s_i) \geq \frac{o(s_i)}{\sup_{\vecc{x} \in \conv(\mathcal{X})} \|J[\vecc{x}]\|_2 },
\end{equation}
where $\conv(\mathcal{X})$ denotes the convex hull of $\mathcal{X}$ and $J[\vecc{x}] = d\mathcal{F}[\vecc{x}]/d\vecc{x}$, with $\mathcal{F}[\vecc{x}] = p(V \bar{h}_M^\delta[\vecc{x}])$, is the input-output Jacobian associated to the RNN. 
\end{prop}
\begin{proof}

The proof is essentially identical to that of Theorem 4 in \cite{sokolic2017generalization}. We provide the full detail here for completeness. 

Denote $o(s_i) = o(\vecc{x}^{(i)},y^{(i)})$, where $\vecc{x}^{(i)}:= (x^{(i)}_0, \cdots, x^{(i)}_{M-1}) \in \mathcal{X} \subset \RR^{d_x M}$, and $v_{ij} = \sqrt{2} (e_i - e_j)$, where $e_i \in \RR^{d_y}$ denotes the Kronecker delta vector. 

The classification margin of the training sample $s_i$ is:
\begin{align}
    \gamma^d(s_i) &= \sup\{ a: \|\vecc{x}^{(i)} -  \vecc{x}\|_2 \leq a \implies g^\delta(\vecc{x}) = y^{(i)} \ \ \forall \vecc{x} \} \\
    &= \sup\{ a: \|\vecc{x}^{(i)} -  \vecc{x}\|_2 \leq a \implies o(\vecc{x},y^{(i)}) > 0 \ \ \forall \vecc{x} \}.
\end{align}

By Definition \ref{app_defn_score}, $o(\vecc{x},y^{(i)}) > 0$ if and only if $\min_{j \neq y^{(i)}} v^T_{y^{(i)} j} \mathcal{F}[\vecc{x}] > 0$. 

On the other hand, 
\begin{align}
    \min_{j \neq y^{(i)}} v^T_{y^{(i)} j} \mathcal{F}[\vecc{x}] &= \min_{j \neq y^{(i)}} (v^T_{y^{(i)} j} \mathcal{F}[\vecc{x}^{(i)}] +   v^T_{y^{(i)} j} (\mathcal{F}[\vecc{x}] - \mathcal{F}[\vecc{x}^{(i)}])) \\ 
    &\geq \min_{j \neq y^{(i)}} v^T_{y^{(i)} j} \mathcal{F}[\vecc{x}^{(i)}] + \min_{j \neq y^{(i)}}   v^T_{y^{(i)} j} (\mathcal{F}[\vecc{x}] - \mathcal{F}[\vecc{x}^{(i)}]) \\
    &= o(\vecc{x}^{(i)}, y^{(i)}) + \min_{j \neq y^{(i)}}   v^T_{y^{(i)} j} (\mathcal{F}[\vecc{x}] - \mathcal{F}[\vecc{x}^{(i)}]).
\end{align}

Therefore, $o(\vecc{x}^{(i)}, y^{(i)}) + \min_{j \neq y^{(i)}}   v^T_{y^{(i)} j} (\mathcal{F}[\vecc{x}] - \mathcal{F}[\vecc{x}^{(i)}]) > 0$ implies that $o(\vecc{x}, y^{(i)}) > 0$ and so
\begin{align}
    \gamma^d(s_i) &\geq \sup\left\{ a: \|\vecc{x}^{(i)} -  \vecc{x}\|_2 \leq a \implies o(\vecc{x}^{(i)}, y^{(i)}) + \min_{j \neq y^{(i)}}   v^T_{y^{(i)} j} (\mathcal{F}[\vecc{x}] - \mathcal{F}[\vecc{x}^{(i)}]) > 0  \ \ \forall \vecc{x} \right\} \\ 
    &= \sup\left\{ a: \|\vecc{x}^{(i)} -  \vecc{x}\|_2 \leq a \implies o(\vecc{x}^{(i)}, y^{(i)}) - \max_{j \neq y^{(i)}}   v^T_{y^{(i)} j} (\mathcal{F}[\vecc{x}^{(i)}]  - \mathcal{F}[\vecc{x}] ) > 0  \ \ \forall \vecc{x} \right\} \\ 
    &= \sup\left\{ a: \|\vecc{x}^{(i)} -  \vecc{x}\|_2 \leq a \implies o(\vecc{x}^{(i)}, y^{(i)}) > \max_{j \neq y^{(i)}}   v^T_{y^{(i)} j} (\mathcal{F}[\vecc{x}^{(i)}]  - \mathcal{F}[\vecc{x}])  \ \ \forall \vecc{x} \right\}.
\end{align}

Now, using the fact that $\|v_{y^{(i)} j}\|_2 = 1$ and Lemma \ref{lem_Jac}, we have:
\begin{equation}
    \max_{j \neq y^{(i)}}   v^T_{y^{(i)} j} (\mathcal{F}[\vecc{x}^{(i)}]  - \mathcal{F}[\vecc{x}]) \leq \sup_{\bar{\vecc{x}} \in \conv(\mathcal{X})} \|J[\bar{\vecc{x}}]\|_2 \cdot \|\vecc{x}^{(i)} - \vecc{x}\|_2. 
\end{equation}

Using this inequality gives:
\begin{align}
\gamma^d(s_i) &\geq \sup\left\{ a: \|\vecc{x}^{(i)} -  \vecc{x}\|_2 \leq a \implies o(\vecc{x}^{(i)}, y^{(i)}) > \sup_{\bar{\vecc{x}} \in \conv(\mathcal{X})} \|J[\bar{\vecc{x}}]\|_2 \cdot \|\vecc{x}^{(i)} - \vecc{x}\|_2  \ \ \forall \vecc{x} \right\} \\
&\geq \frac{ o(\vecc{x}^{(i)}, y^{(i)})}{ \sup_{\bar{\vecc{x}} \in \conv(\mathcal{X})} \|J[\bar{\vecc{x}}]\|_2 }.
\end{align}
The proof is done.
\end{proof}

We now have all the needed ingredients to prove Theorem \ref{app_thm_gen_discrete} and Theorem \ref{app_thm_gen_discrete_2}. 

\begin{proof}[Proof of Theorem \ref{app_thm_gen_discrete} and Theorem \ref{app_thm_gen_discrete_2}]

By Proposition \ref{prop_classmargin_rob}, Lemma \ref{ex_covering}, and our assumption on complexity of the sample space, the RNN classifier is $(d_y \cdot (2C_M/\gamma)^k, 0)$-robust, for some constant $C_M > 0$. Due to Theorem \ref{thm_xu_gen} (with $M:=L_g$ there),  it remains to prove the upper bound \eqref{ub_classmargin} for the classification margin of a training sample to complete the proof. Theorem \ref{app_thm_gen_discrete_2} then follows from Theorem \ref{thm_xu_gen} (with $M:=L_g$ there) and the inequality \eqref{gen_bound}  follows immediately from Theorem \ref{thm_xu_gen}. 

By Proposition \ref{prop_classmargin}, we have 
\begin{equation}
    \gamma^d(s_i) \geq \frac{o(s_i)}{\sup_{\hat{\vecc{x}} \in \conv(\mathcal{X})} \|J[\hat{\vecc{x}}]\|_2 },
\end{equation}
where $J[\hat{\vecc{x}}] := d\mathcal{F}[\hat{\vecc{ \vecc{x}}}]/d\hat{\vecc{ \vecc{x}}}$ is the input-output Jacobian associated to the RNN. Therefore, to complete the proof it suffices to show that 
\begin{equation}
  \|J[\hat{\vecc{x}}]\|_2  \leq  C \sum_{m=0}^{M-1} \delta_m  \|\hat{\Phi}_{M,m+1}[\hat{\vecc{x}}]\|_2,
\end{equation}
where  $C$ is the constant from the theorem and  $\hat{\Phi}_{m+1,k}$, $0 \leq k \leq m \leq M-1$ satisfies:
\begin{align}
    \hat{\Phi}_{k,k} &=  I, \\ 
    \hat{\Phi}_{m+1,k} &= \hat{J}_m \hat{\Phi}_{m,k}, \label{sm_ivp_dis}
\end{align}
where $\hat{J}_m = I + \delta_m f'(\hat{h}^{0}_m, \hat{x}_m)$ (with the $\hat{h}^{(0)}_m$ satisfying Eq. \eqref{hie1_disc}, recalling that we are replacing the superscript $\delta$ by hat when denoting the $\delta$-dependent approximating solutions for the sake of notation cleanliness) and the $\delta_m > 0$ are the step sizes.

Iterating \eqref{sm_ivp_dis} up to the $(m+1)$th step, for $m \geq k$, gives: 
\begin{equation} 
    \hat{\Phi}_{m+1,k} = \hat{J}_m \hat{J}_{m-1} \cdots \hat{J}_k =:  \prod_{l = k}^m  \hat{J}_l.
\end{equation}

Note that 
\begin{equation}
    \hat{\Phi}_{m+1,k} = \frac{\partial \hat{h}_{m+1}^{(0)}}{\partial \hat{h}_{m}^{(0)} } \frac{\partial \hat{h}_{m}^{(0)}}{\partial \hat{h}_{m-1}^{(0)} } \cdots \frac{\partial \hat{h}_{k+1}^{(0)}}{\partial \hat{h}_{k}^{(0)} } = \frac{d\hat{h}_{m+1}^{(0)}}{d\hat{h}_k^{(0)}}.  \label{chain}
\end{equation}

Now, applying chain rule:
\begin{align}
    J[\hat{\vecc{x}}] &= \frac{\partial p(V\hat{h}^{(0)}_M)}{\partial \hat{h}^{(0)}_M } \sum_{j=0}^{M-1} \frac{\partial \hat{h}_{M}^{(0)}}{\partial \hat{h}_{M-1}^{(0)} }
    \cdots 
    \frac{\partial \hat{h}_{j+2}^{(0)}}{\partial \hat{h}_{j+1}^{(0)} }
    \frac{\partial \hat{h}^{(0)}_{j+1}}{\partial \hat{x}_{j}},
\end{align}
where $p$ is the softmax function.

We compute:
\begin{equation}
    \frac{\partial p(V\hat{h}^{(0)}_M)}{\partial \hat{h}^{(0)}_M } =  VE,
\end{equation}
where $E^{ij} = p^i(e^{ij} - p^j)$.
From \eqref{chain}, we have 
\begin{equation}
    \frac{\partial \hat{h}_{M}^{(0)}}{\partial \hat{h}_{M-1}^{(0)} }
    \cdots 
    \frac{\partial \hat{h}_{j+2}^{(0)}}{\partial \hat{h}_{j+1}^{(0)} } = \hat{\Phi}_{M,j+1}.
\end{equation}

On the other hand, 
\begin{equation}
\frac{\partial \hat{h}^{(0)}_{j+1}}{\partial \hat{x}_{j}} = \delta_j \frac{\partial f(\hat{h}^{(0)}_j, \hat{x}_j)}{\partial \hat{x}_j},
\end{equation}
for $j=0,1,\dots,M-1$. Note that for Lipschitz RNNs, we have $\frac{\partial \hat{h}^{(0)}_{j+1}}{\partial \hat{x}_{j}} = \delta_j D_j U$, where $D_l^{ij} =  a'([W \hat{h}^{(0)}_l + U \hat{x}_l +  b]^i) e_{ij}$.

Using the results of the above computations gives:
\begin{align}
    J[\hat{\vecc{x}}] &= 
     VE
    \sum_{m=0}^{M-1} \delta_m 
    \hat{\Phi}_{M,m+1} \frac{\partial f(\hat{h}^{(0)}_m, \hat{x}_m)}{\partial \hat{x}_m}.
\end{align}
Therefore, 
\begin{align}
    \|J[\hat{\vecc{x}}]\|_2 &\leq 
     \|VE\|_2
    \sum_{m=0}^{M-1} \|\delta_m  
    \hat{\Phi}_{M,m+1}\|_2
     \left\| \frac{\partial f(\hat{h}^{(0)}_m, \hat{x}_m)}{\partial \hat{x}_m} \right\|_2  \\ 
     &\leq  \|V\|_2 \left(\max_{m=0,1,\dots,M-1}  \left\| \frac{\partial f(\hat{h}^{(0)}_m, \hat{x}_m)}{\partial \hat{x}_m}   \right\|_2 \right) 
    \sum_{m=0}^{M-1} \delta_m \|  
    \hat{\Phi}_{M,m+1}\|_2= C \sum_{m=0}^{M-1} \delta_m \|  
    \hat{\Phi}_{M,m+1}\|_2. \label{bound_ioJac}
\end{align}
For the Lipschitz RNN, we have $C:= \|V\| (max_{m=0,1,\dots,M-1}  \left\| D_m U \right\|_2)$.
The proof is done. 
\end{proof}

It follows immediately from Eq. \eqref{bound_ioJac} that we have the following sufficient condition for stability with respect to hidden states of deterministic RNN to guarantee stability with respect to  input sequence.

\begin{cor}
Fix a $M$ and assume that $C \sum_{m=0}^{M-1} \delta_m < 1$. Then,  $\|\hat{\Phi}_{M,m+1}\|_2 \leq 1 $ for $m=0,\dots,M-1$ implies that $\|J[\hat{\vecc{x}}]\|_2 < 1$.
\end{cor}

\section{Stability and Noise-Induced Stabilization for NRNNs: Proof of Theorem 3 in the Main Paper}
\label{app_sec:StabilityAppendix}

We begin by discussing stochastic stability for SDEs, which are the underlying continuous-time models for our NRNNs.

Although the additional complexities of SDEs over ODEs often necessitate more involved analyses, many of the same ideas typically carry across. This is also true for stability. A typical approach for proving stability of ODEs involves Lyapunov functions --- in Chapter 4 of \cite{mao2007stochastic}, such approaches are extended for SDEs. This gives way to three notions of stability: (1) stability in probability; (2) moment stability; and (3) almost sure stability. Their definitions are provided in Definitions 4.2.1, 4.3.1, 4.4.1 in \cite{mao2007stochastic}, and are repeated below for convenience.

To preface the definition, consider initializing (\ref{app_NLRNN}) at two different random variables $h_0$ and $h_0' := h_0 + \epsilon_0$, where $\epsilon_0 \in \RR^{d_h}$ is a constant non-random perturbation 
with $\|\epsilon_0\| \leq \delta$. The resulting hidden states, $h_t$ and $h_t'$, are set to satisfy (\ref{app_NLRNN}) with the same Brownian motion $B_t$, starting from their initial values $h_0$ and $h_0'$, respectively. 
The evolution of $\epsilon_t = h_t' - h_t$ satisfies 
\begin{equation} 
    \dd\epsilon_t = A\epsilon_t \dd t + \Delta a_t(\epsilon_t) \dd t  + \Delta \sigma_t(\epsilon_t) \dd B_t, \label{eq_perturb}
\end{equation}
where $\Delta a_t(\epsilon_t) = a(Wh'_t + Ux_t + b) - a(Wh_t + Ux_t + b)$ and $\Delta\sigma_t(\epsilon_t) = \sigma_t(h_t+\epsilon_t) - \sigma_t(h_t)$. 
Since $\Delta a_t(0) = 0$, $\Delta\sigma_t(0) = 0$ for all $t \in [0,T]$, $\epsilon_t = 0$ admits a trivial \emph{equilibrium} for \eqref{eq_perturb}. 

\begin{defn}[Stability for SDEs]
\label{def:Stability}
The trivial solution of the SDE \eqref{eq_perturb} is 
\begin{enumerate}[label=(\roman*)]
    \item \emph{stochastically stable} (or, stable in probability) if for every $\epsilon \in (0,1)$, $r>0$, there exists a $\delta = \delta(\epsilon,r) > 0$ such that $
    \mathbb{P}(\|\epsilon_t\| < r \text{  for all } t \geq 0) \geq 1-\epsilon$ whenever $\|\epsilon_0\| < \delta$.
    \item \emph{stochastically asymptotically stable} if it is stochastically stable and, moreover, for every $\epsilon \in (0,1)$, there exists a $\delta_0 = \delta_0(\epsilon) > 0$ such that $    \mathbb{P}(\lim_{t \to \infty} \epsilon_t = 0) \geq 1-\epsilon$ whenever $\|\epsilon_0\| < \delta_0$. 
    \item \emph{almost surely exponentially stable} if $\limsup_{t \to \infty} t^{-1} \log \|\epsilon_t\| < 0$ with probability one whenever $\|\epsilon_0\| < \delta_1$.
    \item \emph{$p$-th moment exponentially stable} if there exists $\lambda, C > 0$ such that $\mathbb{E}\|\epsilon_t\|^p \leq C\|\epsilon_0\|^p e^{-\lambda(t - t_0)}$ for all $t \geq t_0$.
\end{enumerate}
\end{defn}
The properties in Definition \ref{def:Stability} are said to hold globally if they also hold under no restrictions on $\epsilon_0$. Stability in probability neglects to quantify rates of convergence, and is implied by almost sure exponential stability. On the other hand, for our class of SDEs, $p$-th moment exponential stability would imply almost sure exponential stability (see Theorem 4.2 in \cite{mao2007stochastic}). 

One critical difference between Lyapunov stability theory for ODEs and SDEs lies in the \emph{stochastic stabilization phenomenon}. Let $L$ be the infinitesimal generator (for a given input signal $x_{t}$) of the diffusion process described by the SDE \eqref{eq_perturb}:
\begin{equation}
L=\frac{\partial}{\partial t}+\sum_{i}\left((A\epsilon)^{i}+\Delta a_{t}^{i}(\epsilon)\right)\frac{\partial}{\partial\epsilon^{i}}+\frac{1}{2}\sum_{i,j}\left[\Delta\sigma_{t}(\epsilon)\Delta\sigma_{t}(\epsilon)^{\top}\right]^{ij}\frac{\partial^{2}}{\partial\epsilon^{i}\partial\epsilon^{j}}.
\end{equation}
The generator for the corresponding ODE arises by taking $\Delta \sigma_t \equiv 0$. In classical Lyapunov theory for ODEs, the existence of a non-negative Lyapunov function $V$ satisfying $LV \leq 0$ in some neighbourhood of the equilibrium is both necessary and sufficient for stability (see Chapter 4 of \cite{khalil2002nonlinear}). For SDEs, it has been shown that this condition is sufficient, but no longer necessary \cite{mao2007stochastic,mao1994exponential}. This is by the nature of stochastic stabilization --- the addition of noise can can have the surprising effect of \emph{increased} stability over its deterministic counterpart. Of course, this is not universally the case as some forms of noise can be sufficiently extreme to induce instability; see Section 4.5 in \cite{mao2007stochastic}.

Identifying sufficient conditions which quantify the stochastic stabilization phenomenon are especially useful in our setting, and as it turns out (see also \cite{liu2019neural}), these are most easily obtained for almost sure exponential stability. Therefore, our stability analysis will focus on establishing \emph{almost sure exponential stability}. The objective is to analyze such stability of the solution $\epsilon_t = 0$, that is, to see how the final state $\epsilon_T$ (and hence the output $y'_T-y_T =V \epsilon_T$ of the RNN) changes for an arbitrarily small initial perturbation $\epsilon_0 \neq 0$. 

To this end, we  consider an extension of the Lyapunov exponent to SDEs at the level of sample path \cite{mao2007stochastic}.

\begin{defn}[Almost sure global exponential stability]
The sample (or pathwise) Lyapunov exponent of the trivial solution of \eqref{eq_perturb} is
$\Lambda = \limsup_{t \to \infty} t^{-1} \log \|\epsilon_t\|$. 
The trivial solution $\epsilon_t = 0$ is \emph{almost surely globally exponentially stable} if $\Lambda$ is almost surely negative for all $\epsilon_0 \in \RR^{d_h}$. 
\end{defn}

For the sample Lyapunov exponent $\Lambda(\omega)$, there is a constant $C>0$ and a random variable $0 \leq \tau(\omega) < \infty$ such that for all $t > \tau(\omega)$, $\|\epsilon_t\| = \|h_t'-h_t\| \leq  C e^{\Lambda t}$ almost surely. Therefore, almost sure exponential stability implies that almost all sample paths of (\ref{eq_perturb}) will tend to the equilibrium solution $\epsilon = 0$ exponentially fast. With this definition in tow, we state and prove our primary stability result, which is equivalent to  Theorem 3 in the main paper.

\begin{thm}[Bounds for sample Lyapunov exponent of the trivial solution]  \label{app_prop:Stability}
Assume that Assumption \ref{ass} holds. Suppose that $a$ is $L_a$-Lipschitz, $0 \leq a_\Delta^T(\epsilon, t) \epsilon \leq L_a \|\epsilon\|_2^2$ and $0 \leq \sigma_1 \|\epsilon\| \leq  \|\Delta \sigma_t(\epsilon)\|_F 
\leq \sigma_2 \|\epsilon\|$ for all nonzero $\epsilon \in \mathbb{R}^{d_h}$, $t \in [0,T]$.
Then, with probability one,
\begin{equation} \label{ineq_cor}
    -\sigma_2^2 + \frac{\sigma_1^2}{2} + \lambda_{\min}(A^{\sym}) \leq \Lambda \leq -\sigma_1^2+\frac{\sigma_2^2}{2} + L_a\sigma_{\max}(W) + \lambda_{\max}(A^{\sym}),
\end{equation} 
for any $\epsilon_0 \in \mathbb{R}^{d_h}$.
\end{thm}

To establish the bounds in Theorem \ref{app_prop:Stability}, we appeal to the following theorem, which arises from combining Theorems 4.3.3 and 4.3.5 in \cite{mao2007stochastic} in the case $p = 2$. Here, for a function $V$, we let $V_\epsilon = \partial V / \partial \epsilon$.
\begin{thm}[Stochastic Lyapunov theorem]
\label{stoch_lyp_thm}
If there exists a function $V\in\mathcal{C}^{2,1}(\mathbb{R}^{d_{h}}\times\mathbb{R}^{+};\mathbb{R}^{+})$ and $c_{1},C_{1}>0$, $c_{2},C_{2}\in\mathbb{R}$, $c_{3},C_{3}\geq 0$ such that for all $\epsilon\neq 0$ and $t\geq t_{0}$, 
\begin{enumerate}[label=(\roman*)]
    \item $c_{1}\|\epsilon\|^{2}\leq V(\epsilon,t)\leq C_{1}\|\epsilon\|^{2}$,
    \item $c_{2}V(\epsilon,t)\leq LV(\epsilon,t)\leq C_{2}V(\epsilon,t)$, and 
    \item $c_{3}V(\epsilon,t)^{2}\leq\left\lVert V_{\epsilon}(\epsilon,t)\Delta\sigma_{t}(\epsilon)\right\rVert _{F}^{2}\leq C_{3}V(\epsilon,t)^{2}$,
\end{enumerate}
then, with probability one, the Lyapunov exponent $\Lambda$ lies in the interval
\begin{equation} \label{lyp_exp_bounds}
\frac{2c_{2}-C_{3}}{4}\leq\Lambda\leq-\frac{c_{3}-2C_{2}}{4}.    
\end{equation}
\end{thm}

The proof of Theorem \ref{stoch_lyp_thm} involves the It\^o formula, an exponential martingale inequality and a Borel-Cantelli type argument. The functions $V$ above are called stochastic Lyapunov functions and the use of the theorem involves construction of these functions. We are now in a position to prove Theorem \ref{app_prop:Stability}, and will find that the choice $V(\epsilon,t) = \|\epsilon\|^2$ will suffice.

\begin{proof}[Proof of Theorem \ref{app_prop:Stability}]

It suffices to verify the conditions of Theorem 2 with $V(\epsilon,t)=V(\epsilon)=\|\epsilon\|^{2}$. 

Clearly (i) is satisfied. To show (iii), by the conditions on $\Delta\sigma_{t}$, we have that $4\sigma_{1}^{2}\|\epsilon\|^{4}\leq\|V_{\epsilon}(\epsilon)\Delta\sigma_{t}(\epsilon)\|_{F}^{2}\leq4\sigma_{2}^{2}\|\epsilon\|^{4}$. It remains only to show (ii). Observe that
\[
LV(\epsilon)=\epsilon^{\top}(A+A^{\top})\epsilon+2\Delta a_{t}(\epsilon)\epsilon+\mbox{tr}(\Delta\sigma_{t}(\epsilon)\Delta\sigma_{t}(\epsilon)^{\top}).
\]
Since $0\leq\Delta a_{t}(\epsilon)\epsilon$ and
\begin{align*}
\left|\Delta a_{t}(\epsilon)\epsilon\right|&	\leq\left\lVert a(Wh_{t}'+Ux_{t}+b)-a(Wh_{t}+Ux_{t}+b)\right\rVert \left\lVert \epsilon\right\rVert \\
&\leq L_{a}\left\lVert W\epsilon\right\rVert \left\lVert \epsilon\right\rVert \leq L_{a}\sigma_{\max}(W)\left\lVert \epsilon\right\rVert ^{2},
\end{align*}
it follows that
\[
LV(\epsilon)\leq(2\lambda_{\max}(A^{sym})+2L_{a}\sigma_{\max}(W)+\sigma_{2}^{2})\|\epsilon\|^{2},
\]
and
\[
LV(\epsilon)\ge(2\lambda_{\min}(A^{sym})+\sigma_{1}^{2})\|\epsilon\|^{2}.
\]
The bound \eqref{lyp_exp_bounds} now follows from Theorem \ref{stoch_lyp_thm} with $c_{1}=C_{1}=1$, $c_{2}=2\lambda_{\min}(A^{sym})+\sigma_{1}^{2}$, $C_{2}=2\lambda_{\max}(A^{sym})+2L_{a}\sigma_{\max}(W)+\sigma_{2}^{2}$, $c_{3}=4\sigma_{1}^{2}$, and $C_{3}=4\sigma_{2}^{2}$. 
\end{proof}
 
\begin{rmk} \label{rmk_A2}
 To see if the bounds in Theorem \ref{app_prop:Stability} are indeed sharp (at least for certain cases), consider the linear SDE $dH_t = AH_t dt + B H_t dW_t$, where $A \in \RR^{d_h \times d_h}$, $B = \sigma I$, $\sigma \in \RR$ and $W_t$ is a scalar Wiener process. Then, since $A$ and $B$ commute, they can be simultaneously diagonalized, and so the linear SDE can be reduced via transformation to a set of independent one-dimensional linear SDEs. In particular, one can show that $H_t = \exp((A-B^2/2)t+BW_t)) H_0$ and the Lyapunov exponents $\Lambda$ of this system are the real part of the eigenvalues of $A-B^2/2$. Note that $\lambda_{\min}(A_{sym}-B^2/2) \leq \Lambda \leq \lambda_{\max}(A_{sym}-B^2/2) $ a.s.. Since $B = \sigma I$, this inequality implies Eqn. \eqref{ineq_cor} with $L_a:=0$.
 The bounds are tight  in the scalar case ($d_h=1$ and $A$ is a scalar), with the inequality becoming an equality. 
\end{rmk}

\begin{rmk}
Even in the additive noise setting, however, the Lyapunov exponents of the CT-NRNN driven by additive noise are not generally the same as those of the corresponding deterministic CT-RNN. Oseledets multiplicative ergodic theorem implies they will be the same if the data generating process $x_t$ is ergodic \cite{arnold1986lyapunov}. Characterizing Lyapunov exponents for SDEs is a non-trivial affair in general --- we refer to, for instance, \cite{arnold1987large} for  details on this.
\end{rmk}

\section{Experimental Details}
\label{app_sect_experiment}

\subsection{Experimental Results Presented in the Main Paper}

Following \cite{erichson2020lipschitz}, we construct the hidden-to-hidden weight matrices $A$ and $W$ as
\begin{align}
A & = T(B, \beta_a, \gamma_a) := (1-\beta_a) \cdot (B+B^T) + \beta_a \cdot (B-B^T) - \gamma_a I,\\
W & = T(C ,\beta_w, \gamma_w) := (1-\beta_w) \cdot (C+C^T) + \beta_w \cdot (C-C^T) - \gamma_w I.
\end{align}
Here, $B$ and $C$ denote weight matrices that have the same dimensions as $A$ and $W$. 
The tuning parameters $\gamma_a$ and $\gamma_w$ can be used to increase dampening.
We initialize the weight matrices by sampling weights from the normal distribution $\mathcal{N}(0,\sigma^2_{init})$, where $\sigma_{init}^2$ is the variance.
Table~\ref{tab:tuning} summarizes the tuning parameters that we have used in our experiments. We train our models for $100$ epochs, with scheduled learning rate decays at epochs $\{90\}$. We use Adam with default parameters for minimizing the objective. 

\begin{table}[h]
	\caption{Tuning parameters used to train the NRNN.}
	\label{tab:tuning}
	\centering
	\scalebox{0.8}{
		\begin{tabular}{l c c c c c c c c c c c}
			\toprule
			Name           &  d\_h & lr  & decay & $\beta$ & $\gamma_a$ & $\gamma_w$ & $\epsilon$ & $\sigma_{init}^2$ & add. noise & mult. noise  \\
			\midrule 
			Ordered MNIST  & 128 & 0.001 & 0.1  & 0.75 &  0.001 & 0.001 & 0.01 & $0.1 / 128$  & 0.02& 0.02\\
			Ordered MNIST  & 128 & 0.001 & 0.1  & 0.75 &  0.001 & 0.001 & 0.01 & $0.1 / 128$  & 0.05 & 0.02\\			
			\midrule 
			Permuted MNIST & 128 & 0.001 & 0.1 & 0.75 & 0.001 & 0.001 & 0.01 & $0.1 / 128$  & 0.02 & 0.02\\
			Permuted MNIST & 128 & 0.001 & 0.1 & 0.75 & 0.001 & 0.001 & 0.01 & $0.1 / 128$  & 0.05 & 0.02\\			\midrule 
			ECG & 128 & 0.001 & 0.1 & 0.9 & 0.001 & 0.001 & 0.1 & $0.1 / 128$  & 0.06 & 0.03\\			
			
			\bottomrule
	\end{tabular}}
\end{table}

We performed a random search to obtain the tuning parameters. Since our model is closely related to the Lipschitz RNN, we started with the tuning parameters proposed in \cite{erichson2020lipschitz}. We evaluated different noise levels, both for multiplicative and additive noise, in the range $\left[ 0.01, 0.1 \right]$. We tuned the levels of noise-injection so that the models achieve state-of-the-art performance on clean input data. Further, we observed that the robustness of the model is not significantly improving when trained with increased levels of noise-injections. Overall, our experiments indicated that the model is relatively insensitive to the particular amount of additive and multiplicative noise level in the small noise regime. 

Further, we need to note that we only considered models that used a combination of additive and multiplicative noise-injections. One could also train models using either only additive or multiplicative noise-injections. We did not investigate in detail the trade-offs between the different strategies. The motivation for our experiments was to demonstrate that (i) models trained with noise-injections can achieve state-of-the-art performance on clean input data, and (ii) such models are also more resilient to input perturbations.

Figure \ref{fig_hessian} shows that NRNN exhibits a smoother Hessian landscape than that of the deterministic counterpart.
\begin{figure}[!t] 
		\begin{center}
    \includegraphics[width=0.7\textwidth]{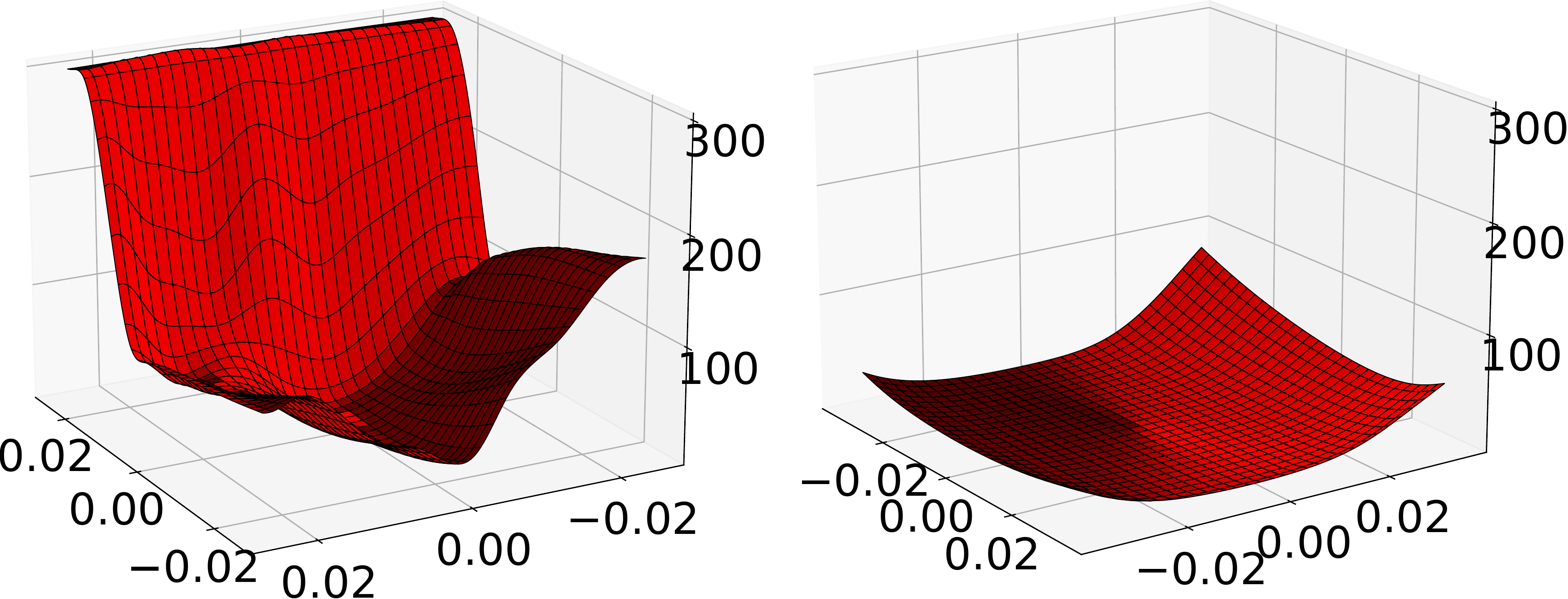}
  \end{center}
  \caption{Hessian loss landscapes for  deterministic (left) and noisy (right) model, computed using PyHessian.}
   \label{fig_hessian}
\end{figure}

For establishing a fair set of baselines, we used the following implementations and prescribed tuning parameters for the other models that we considered.

\begin{itemize}
    \item \textbf{Exponential RNN.} We used the following implementation: \url{https://github.com/Lezcano/expRNN}. We used the default parameters. We trained the model, with hidden dimension $d_h=128$, for $100$ epochs. 
    
    \item \textbf{CoRNN.} We used the following implementation, provided as part of the Supplementary Material: \url{https://openreview.net/forum?id=F3s69XzWOia}. We used the default parameters proposed by the authors for training the model with hidden dimension $d_h=128$. We trained the model for $100$ epochs with learning rate decay at epoch 90.

    \item \textbf{Lipschitz RNN.} We used the following implementation, provided as part of the Supplementary Material: \url{https://openreview.net/forum?id=-N7PBXqOUJZ}. We used the default parameters proposed by the authors for training the model with hidden dimension $d_h=128$. We trained the model for $100$ epochs with learning rate decay at epoch 90.

    \item \textbf{Antisymmetric RNN.} To our best knowledge, there is no public implementation by the authors for  this model. However, the Antisymmetric RNN can be seen as a special case of the Lipschitz RNN or the NRNN, without the stabilizing term $A$ and without noise-injection. We trained this model by using our implementation and the following tuning parameters: $\beta=1.0$, $\gamma=0.001$, lr = $0.002$, $\epsilon=0.01$. We trained the model for $100$ epochs with learning rate decay at epoch 90.

\end{itemize}

\subsection{Additional Results for Permuted Pixel-by-Pixel MNIST Classification}\label{sec:permutedMNIST}

Here we consider the permuted pixel-by-pixel MNIST classification task.
This task sequentially presents a scrambled sequence of the $784$ pixels to the model and uses the final hidden state to predict the class membership probability of the input image. 

Table~\ref{tab:pmnist-table} shows the average test accuracy (evaluated for models that are trained with 10 different seed values). Here we present results for white noise and salt and pepper (S\&P) perturbations.
Again, the NRNNs show an improved resilience to input perturbations. 
Figure~\ref{fig:pmnist} summarizes the performance of different models with respect to white noise and salt and pepper perturbations. 
\begin{table*}[!t]
	\caption{Robustness w.r.t. white noise ($\sigma$) and S\&P ($\alpha$) perturbations on the permuted MNIST task.}
	\label{tab:pmnist-table}
	\centering
	\scalebox{0.8}{
		\begin{tabular}{l c c c c | c c c c c c}
			\toprule
			Name                  &  clean & $\sigma=0.1$ & $\sigma=0.2$ & $\sigma=0.3$ &  $\alpha=0.03$ & $\alpha=0.05$ & $\alpha=0.1$\\
			\midrule

			Antisymmetric RNN~\cite{chang2019antisymmetricrnn}  & 92.8\% & 92.4\% & 89.5\% & 81.9\%   & 90.5\% & 87.9\% & 72.6\%  \\
				
			CoRNN~\cite{rusch2021coupled}  & 96.05\% & 65.1\% & 38.25\% & 29.1\%   & 84.8\% & 73.8\% & 52.6\%  \\
			
			Exponential RNN~\cite{lezcano2019cheap} & 93.3\% & 90.6\% & 78.4\% & 61.6\%   & 80.4\% & 70.6\% & 51.6\%  \\
			
			Lipschitz RNN~\cite{erichson2020lipschitz}   & \textbf{95.9}\% & 95.4\% & 93.5\% & 83.7\%    & 93.7\% & 90.2\% & 70.8\%  \\
			
			NRNN (mult./add. noise: 0.02/0.02)   & 94.9\% & \textbf{94.8}\% & \textbf{94.6}\% & 94.3\%   & \textbf{94.0}\% & {93.1}\% & 88.6\%  \\			
			
			NRNN (mult./add. noise: 0.02/0.05)   & 94.7\% & {94.6}\% & \textbf{94.6}\% & \textbf{94.4}\%  & \textbf{94.0}\% & \textbf{93.2}\% & \textbf{90.5}\%  \\
			
			\bottomrule
	\end{tabular}}
\end{table*}

\begin{figure*}[!t]
	\centering
	\begin{subfigure}[t]{0.49\textwidth}
		\centering
		\begin{overpic}[width=1\textwidth]{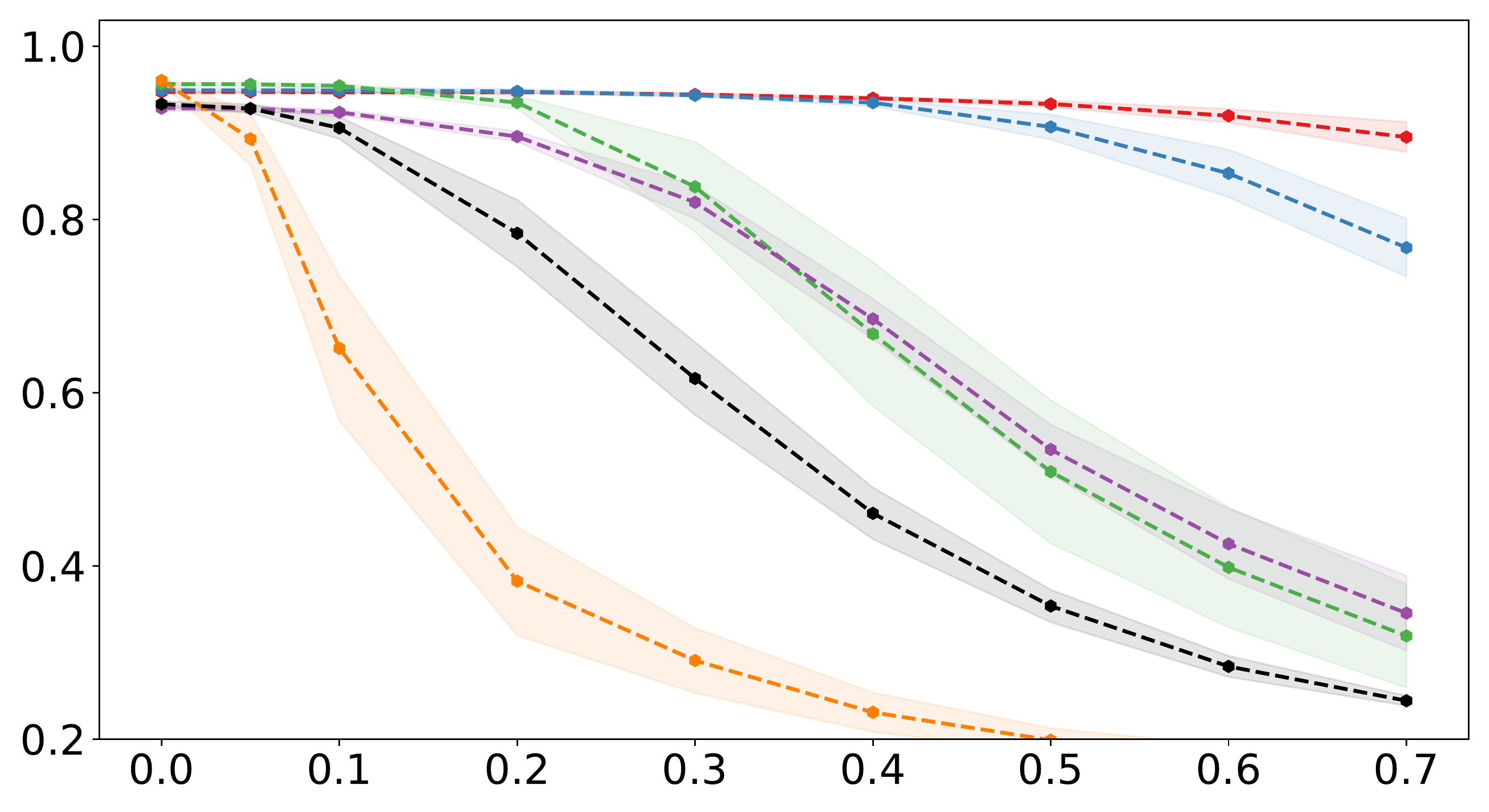}
			\put(-6,15){\rotatebox{90}{\footnotesize test accuracy}}			
			\put(42,-3){\footnotesize {amount of noise}}  	
		\end{overpic}\vspace{+0.2cm}		
		\caption{White noise perturbations.}
	\end{subfigure}
	~
	\begin{subfigure}[t]{0.49\textwidth}
		\centering
		\begin{overpic}[width=1\textwidth]{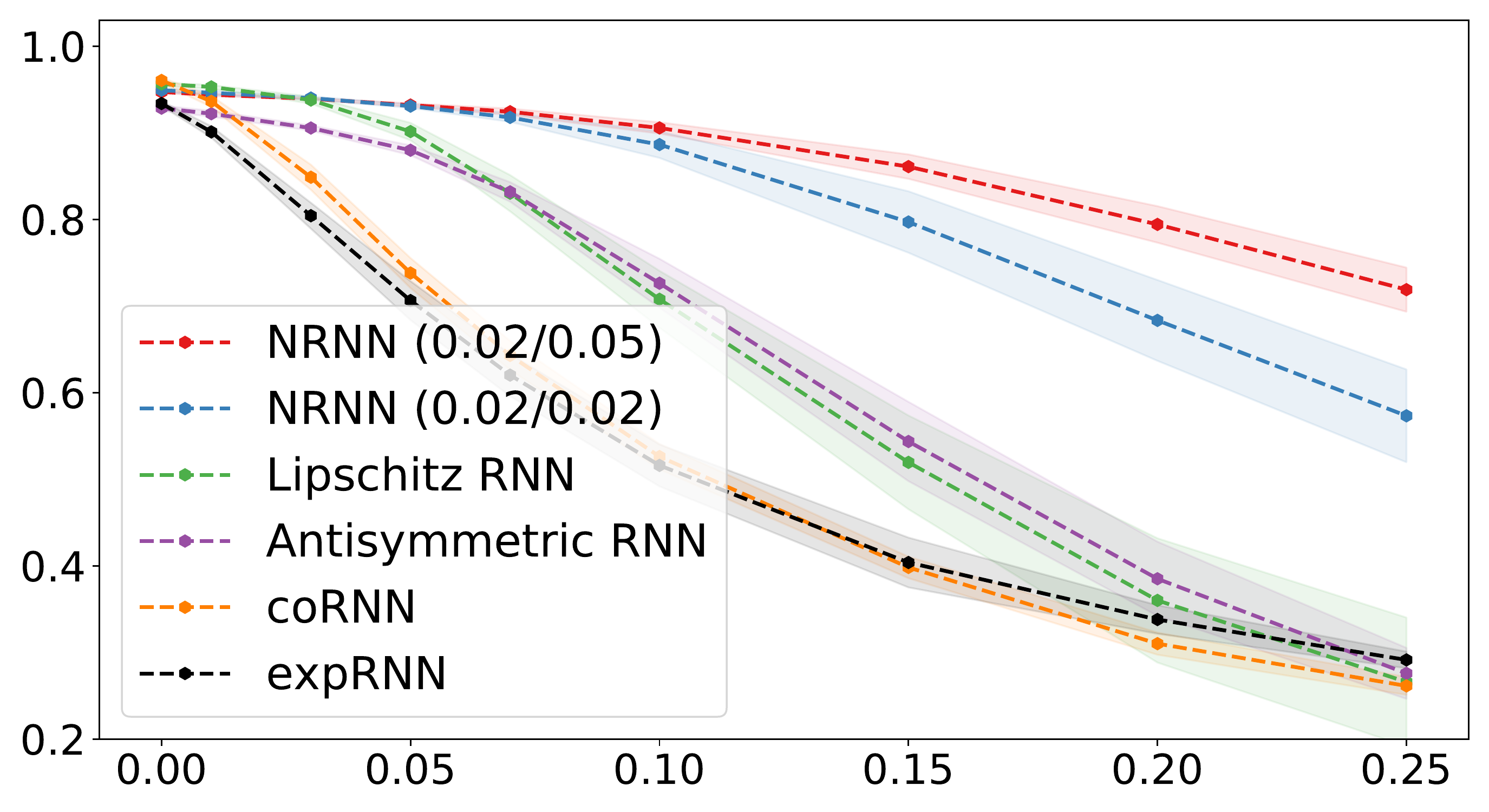} 
			\put(42,-3){\footnotesize {amount of noise}}  		
		\end{overpic}\vspace{+0.2cm}			
		\caption{Salt and pepper perturbations.}
	\end{subfigure}	
	\caption{Test accuracy for the permuted MNIST task as function of the strength of input perturbations.}
	\label{fig:pmnist}
\end{figure*}

\end{document}